\PassOptionsToPackage{dvipsnames}{xcolor}
\documentclass{article}
\usepackage{iclr2027_conference,times}

\usepackage{xcolor}
\usepackage{amsmath,amsfonts,amssymb,amsthm}
\usepackage{bm}
\usepackage{graphicx}
\usepackage{float}
\usepackage{booktabs}
\usepackage{etoc}
\usepackage{flafter}
\usepackage{hyperref}
\hypersetup{hidelinks}
\usepackage{xurl}

\newcommand{\quant}[1]{Q\left(#1\right)}

\newcommand{\softmax}{\mathrm{softmax}}
\newcommand{\KL}[2]{\mathrm{KL}\!\left(#1\,\middle\|\,#2\right)}

\definecolor{cleanred}{rgb}{0.62,0.05,0.10}
\definecolor{quantgreen}{rgb}{0.0,0.42,0.16}

\definecolor{termcontrast}{HTML}{15567F}
\definecolor{terminteraction}{HTML}{C05E00}
\definecolor{termalignment}{HTML}{1E6E1E}
\definecolor{termxi}{HTML}{6A4590}
\definecolor{removalpurple}{HTML}{7A5195}
\definecolor{reversalred}{HTML}{A33A4A}
\newcommand{\termlabel}[2]{\textcolor{#1}{\bm{#2}}}

\usepackage{ifthen}

\newif\ifdebug
\debugfalse

\newtheorem{theorem}{Theorem}
\newtheorem{proposition}{Proposition}

\newcommand{\W}[1]{{\Wv^{(#1)}}}

\newcommand{\T}[1]{{\Tv^{(#1)}}}

\newcommand{\F}[2]{{\Fv^{(#1)}}\left(#2\right)}
\newcommand{\G}[2]{{\Gv^{(#1)}}\left(#2\right)}
\newcommand{\f}[2]{{\fv^{(#1)}}\left(#2\right)}

\newcommand{\R}{\mathrm{R}}

\newcommand{\vect}[1]{\boldsymbol{\mathbf{#1}}}
\newcommand{\compacthat}[2]{\widehat{#1}\vphantom{#1}^{#2}}

\newcommand{\E}[1]{\mathbb{E}\left[#1\right]}

\newcommand{\N}{\mathcal{N}}

\newcommand{\fv}{\vect f}
\newcommand{\gv}{\vect g}
\newcommand{\hv}{\vect h}

\newcommand{\pv}{\vect p}

\newcommand{\uv}{\vect u}

\newcommand{\wv}{\vect w}

\newcommand{\zv}{\vect z}

\newcommand{\Fv}{\vect F}
\newcommand{\Gv}{\vect G}

\newcommand{\Iv}{\vect I}

\newcommand{\Tv}{\vect T}

\newcommand{\Wv}{\vect W}
\newcommand{\Xv}{\vect X}
\newcommand{\Yv}{\vect Y}

\newcommand{\Gc}{\mathcal G}

\newcommand{\Mc}{\mathcal M}

\newcommand{\Vc}{\mathcal V}

\newcommand{\Eb}{\mathbb E}

\newcommand{\Rb}{\mathbb R}

\newcommand{\norm}[1]{\lVert#1\rVert}

\newcommand{\iprod}[1]{\langle#1\rangle}

\makeatletter
\newtheorem*{rep@theorem}{\rep@title}
\newcommand{\newreptheorem}[2]{%
	\newenvironment{rep#1}[1]{%
		\def\rep@title{#2 \ref{##1}}%
		\begin{rep@theorem}}%
		{\end{rep@theorem}}}
\makeatother

\title{Why Does Post-Training Quantization Work?
}
\author{
Yuxiang Chen\textsuperscript{1,2},\quad
Michael Beyer\textsuperscript{3},\quad
Jun Zhu\textsuperscript{1},\quad
Jianfei Chen\textsuperscript{1,}\thanks{Corresponding author}\\
\textsuperscript{1}Dept. of Comp. Sci. and Tech., Institute for AI, BNRist Center, THBI Lab, \\~~Tsinghua-Bosch Joint ML Center, Tsinghua University \\
\textsuperscript{2}College of AI, Tsinghua University,~~
\textsuperscript{3}Bosch AI Research, Renningen, Germany
\vspace{3pt}\\
\texttt{chenyuxi22@mails.tsinghua.edu.cn}, ~~~
\texttt{michael.beyer2@de.bosch.com},\\
\texttt{\{dcszj,~jianfeic\}@tsinghua.edu.cn}
}

\iclrfinalcopy

\begin{document}
\etocdepthtag.toc{main}
\maketitle

\begin{abstract}
Post-training quantization compresses large language models (LLMs) by
storing their weights at reduced precision, and each quantized weight
introduces an error into the hidden states. Naively, these errors should
accumulate with depth and corrupt next-token prediction; randomly initialized
models accumulate these discrepancies rapidly, whereas quantized pretrained models
accumulate much less hidden-state error and largely maintain downstream task
performance, even though they were never trained with quantization noise. This
raises the question we address: why does post-training quantization work? Comparing full-precision
and quantized forward passes, we identify two mechanisms that
characterize pretrained quantization robustness. First, the error a layer newly introduces tends to oppose the error it
inherits from the layer's input. 
The two cancel partially such that the discrepancy between full-precision and quantized passes grows slowly. 
This counteracting residual interaction develops during pretraining.
Our quantitative analysis identifies it as a major factor slowing hidden-error growth.
Second, LM-head geometry preferentially preserves the scores and probabilities of high-ranked tokens, which typically represent the model's most confident predictions.
Together, these mechanisms explain why quantization error that passes through numerous layers can still produce only small output changes, and we verify the findings across models and quantization settings.
\end{abstract}

\section{Introduction}
\label{sec:introduction}

Weight-only post-training quantization (PTQ) compresses large language
models by storing their weights at reduced precision, and methods such as
GPTQ~\citep{frantar2023gptq} and AWQ~\citep{lin2024awq} are now widely used to
serve them. Quantizing a weight makes it slightly different from its
full-precision value, so at every layer, the quantized model's hidden states
differ from its full-precision counterpart, and these errors should naively
accumulate as they propagate through the layers. 
Yet in practice, even \emph{directly} rounding the weights to 4-bit barely hurts: casting
Qwen3-32B~\citep{yang2025qwen3} to NVFP4~\citep{nvidia2025nvfp4} without
calibration lowers benchmark accuracy by only $0.43$ percentage points on
average across six zero-shot benchmarks
(Sec.~\ref{sec:preferential-output-preservation}), even
though the model was never trained with quantization noise. This raises our
central question: \emph{why does post-training quantization work}?

The usual answer is that quantized weights remain close to their full-precision
values (cosine similarity to their NVFP4 reconstructions $\sim 0.996$), so
errors should propagate slowly. Yet \emph{pretrained} and \emph{randomly initialized}
weights have nearly identical reconstruction error
(App.~\ref{app:nvfp4-weight-closeness}), while the hidden-state discrepancy is
$5.5\times$ larger at random initialization 
(Fig.~\ref{fig:hidden-error-growth}).
Because the weight-level errors are closely matched, this gap in hidden-error growth should reflect properties acquired during pretraining. Prior work reduces or diagnoses quantization damage~\citep{NEURIPS2025_df2034a5,lee2026lfqlogitawarefinalblockquantization,lotfi2026quantizedreasoningmodelsthink},
but leaves a mechanistic gap: it does not explain how quantization errors
propagate through a pretrained model or why they produce only small output
changes.

We trace the discrepancy block by block and find two mechanisms that act in
series. Throughout, we compare the quantized and full-precision models on the
same inputs; the \emph{error} at each layer is the difference between their hidden
states, and we study how its size changes with depth. First, the new error a
block introduces tends to \emph{oppose} the error it inherits from earlier blocks, so
the two partially cancel and the size of the error grows slowly
(Sec.~\ref{sec:residual-counteraction}). Second, the final hidden-state error reaching the LM head is mainly a rotation whose effect is
attenuated by the high-dimensional LM head, especially for top-ranked tokens
(Sec.~\ref{sec:lmhead-geometry}). Together, they explain why a large internal
discrepancy translates into a small output change. Both mechanisms hold across
Qwen3 models of different scales, OLMo3~\citep{teamolmo2025olmo3},
Gemma3~\citep{gemmateam2025gemma3technicalreport}, and the OLMoE
mixture-of-experts model~\citep{muennighoff2025olmoe} (Sec.~\ref{sec:scope}).

We summarize the work's contributions as follows:
{\setlength{\leftmargini}{1.4em}
\begin{enumerate}
\setlength{\itemsep}{1pt}\setlength{\parskip}{0pt}\setlength{\parsep}{0pt}\setlength{\topsep}{2pt}
\item \textbf{Reframing quantization robustness as a mechanistic question.}
We ask why post-training quantization works, and use comparisons with randomly
initialized models to show that the usual explanation that weights stay close to their
full-precision values is not the full answer; slow hidden-error growth
largely comes from the pretrained model itself.
\item \textbf{A quantitative account of slow error growth.}
The error added by each layer tends to oppose the error it inherits, slowing
error growth. We derive an exact decomposition of this growth and confirm experimentally that this cancellation is significant for limiting hidden-error growth.
\item \textbf{A theoretical account of stable top-ranked token predictions.}
The scores of the tokens with the highest output probabilities change only slightly under
the surviving final hidden-state error; we derive how hidden-state error induces score
changes through the LM-head geometry, and connect the resulting score
perturbations to log-probability changes.
\end{enumerate}
}

\section{Setup: Comparing Full-Precision and Quantized Models}

\paragraph{Next-token prediction in LLMs}

For a tokenized sequence
$x_{1:n}=(x_1,\ldots,x_n)$ in vocabulary $\Vc$,
each token $x_i$ is embedded as $\hv_i^{(0)}\in\Rb^d$, where $d$ is the hidden
dimension, and then processed by $L$ Transformer
blocks~\citep{vaswani2017attention}. Block $\ell\in\{1,\ldots,L\}$ computes a
residual update:
\begin{align*}
\uv_i^{(\ell)}:=
\Fv^{(\ell)}\bigl(\hv_{1:i}^{(\ell-1)}\bigr),
\qquad
\hv_i^{(\ell)}=
\hv_i^{(\ell-1)}+\uv_i^{(\ell)}.
\end{align*}
Here $\Fv^{(\ell)}$ denotes the complete computation performed by block
$\ell$,
including the Attention and MLP contributions.
The final hidden state
is mapped to scores $\zv$ by the LM head to obtain the next-token distribution, with
$T=1$ unless stated otherwise:
\begin{align*}
\hv_{\rm LM}
=
\operatorname{Norm}\left(\hv_n^{(L)}\right)
\in \Rb^{d},
\qquad
\zv
=
\Wv_{\mathrm{LM}}\hv_{\rm LM}
\in \Rb^{|\Vc|},
\qquad
\pv
=
\operatorname{softmax}(\zv/T)
\in \Rb^{|\Vc|}.
\end{align*}
\paragraph{Quantifying the quantization impact}

A linear layer
maps an input $\Xv$ to $\Yv=\Xv\Wv^{\top}$ using a weight matrix
$\Wv$. Weight quantization stores $\quant{\Wv}$ as a low-precision approximation of $\Wv$. 
$\widehat{\Yv}=\Xv\quant{\Wv}^{\top}$ is thus an approximation of $\Yv$. 
Our study mainly uses the NVFP4 format~\citep{nvidia2025nvfp4} with round-to-nearest (RTN)
(App.~\ref{app:nvfp4-format}). 
We quantize standard linear projections in Attention and MLP to NVFP4 and keep
other operations, including the LM head, in full precision (BF16). 

Quantizing these linear layers 
gives the \emph{quantized model $\color{quantgreen}{\widehat{\Mc}}$}; we refer to
the unquantized model as the \emph{original model} $\color{cleanred}{\Mc}$.
We compare the two models on the \emph{same input sequence} $x_{1:n}$. At
each position $i\in\{1,\ldots,n-1\}$, both models receive the prefix
$x_{1:i}$ when predicting $x_{i+1}$. We analyze each position separately and
omit the position index $i$ below. Unless stated otherwise, the expectation $\Eb$ first
averages over positions within each input sequence and then averages equally across
sequences.

\begin{figure*}[t]
  \centering
  \includegraphics[width=0.82\textwidth]{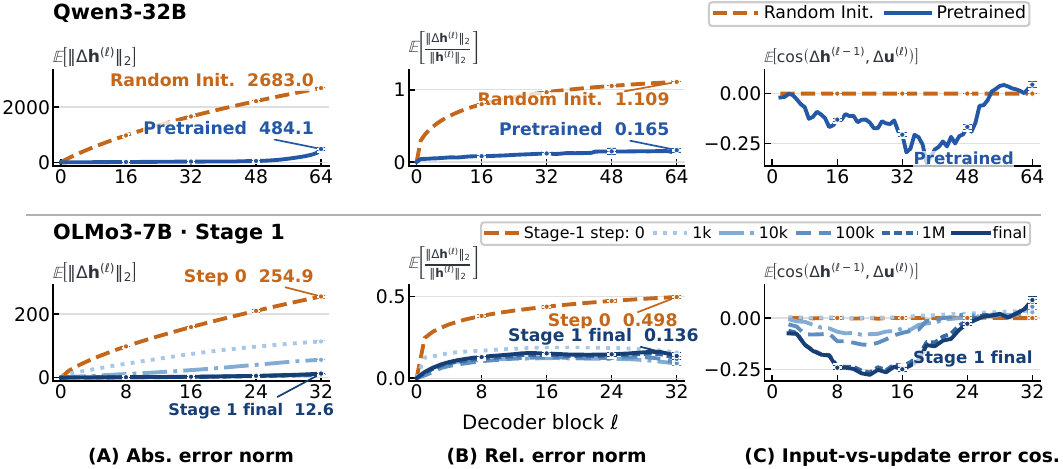}
  \caption{\textbf{Pretraining changes quantized model's hidden-error growth.}
  Rows compare random/pretrained Qwen3-32B and OLMo3-7B Stage-1 checkpoints
  on C4.
  \textbf{(A)} Mean absolute hidden-error norm; \textbf{(B)} mean relative
  hidden-error norm; \textbf{(C)} mean cosine between the block-input error and
  block-update error. Error bars show standard deviation.}
  \label{fig:hidden-error-growth}
\end{figure*}

At each decoding position $i$, the original model $\color{cleanred}{{\Mc}}$ and quantized model $\color{quantgreen}{\widehat{\Mc}}$ 
start with the same initial hidden state, $\compacthat{\hv}{(0)}=\hv^{(0)}$. 
For every block $\ell\in\{1,\ldots,L\}$, let $\uv^{(\ell)}$ and $\compacthat{\uv}{(\ell)}$ denote the respective block updates of $\color{cleanred}{{\Mc}}$ and $\color{quantgreen}{\widehat{\Mc}}$:
$\uv^{(\ell)}
={ \color{cleanred}{\Fv^{(\ell)}} }(
  \hv_{1:i}^{(\ell-1)}
)$
and
$\compacthat{\uv}{(\ell)}
={ \color{quantgreen}{\widehat{\Fv}^{(\ell)}} }(
  \compacthat{\hv}{(\ell-1)}_{1:i}
)$. The two residual updates and their hidden-state difference are
\begin{equation}
\underbrace{
\hv^{(\ell)}=\hv^{(\ell-1)}+\uv^{(\ell)},
}_{\smash[t]{
  \textbf{\color{cleanred}original model};\,
  \text{block }\ell:\,\hv^{(\ell-1)}\!\mapsto\!\hv^{(\ell)}}}
\quad
\underbrace{
\widehat{\hv}^{(\ell)}=\widehat{\hv}^{(\ell-1)}+\widehat{\uv}^{(\ell)},
}_{\smash[t]{
  \textbf{\color{quantgreen}quantized model};\,
  \text{block }\ell:\,\widehat{\hv}^{(\ell-1)}\!\mapsto\!\widehat{\hv}^{(\ell)}}}
\quad
\begin{cases}
\Delta\hv^{(\ell)}:=\widehat{\hv}^{(\ell)}-\hv^{(\ell)}, \\
\Delta\uv^{(\ell)}:=\widehat{\uv}^{(\ell)}-\uv^{(\ell)}
\end{cases}
\label{eq:setup-paired-block}
\end{equation}
We refer to the quantization-induced hidden-state difference
$\Delta\hv^{(\ell)}$ as the \emph{\textbf{hidden error}} at layer $\ell$.
The LM-head inputs are then
$\hv_{\rm LM}:=\operatorname{Norm}
\bigl(
    \hv^{(L)}
\bigr)$
and
$\widehat{\hv}_{\rm LM}:=\operatorname{Norm}
\bigl(
    \compacthat{\hv}{(L)}
\bigr)$:
\begin{equation*}
\underbrace{
\zv=\Wv_{\mathrm{LM}}\hv_{\rm LM},\quad
\pv=\softmax(\zv)
}_{\smash[t]{
  \text{\color{cleanred}original model } 
  \color{cleanred}{\Mc};\,
  \hv_{\rm LM}\!\mapsto\!\pv}}
\qquad
\underbrace{
\widehat{\zv}=\Wv_{\mathrm{LM}}\widehat{\hv}_{\rm LM},\quad
\widehat{\pv}=\softmax(\widehat{\zv})
}_{\smash[t]{
  \text{\color{quantgreen}quantized model }
  \color{quantgreen}{\widehat{\Mc}};\,
  \widehat{\hv}_{\rm LM}\!\mapsto\!\widehat{\pv}}}.
\end{equation*}
The LM-head input difference
$\Delta\hv_{\rm LM}:=\widehat{\hv}_{\rm LM}-\hv_{\rm LM}$ induces the score difference
$
\Delta\zv
:=\widehat{\zv}-\zv
=\Wv_{\mathrm{LM}}\Delta\hv_{\rm LM}.
$
We write $\operatorname{Top}_K(\zv)\subseteq\Vc$ for the vocabulary
tokens with the largest $K$ logits in 
$\color{cleanred}{\Mc}$, 
and measure output quality by cross-entropy (CE) and 
$D_{\mathrm{KL}}(\pv\|\widehat{\pv})$.

\section{The Growth of Hidden-Error Norm across Layers}
\label{sec:residual-counteraction}

The models start from the same hidden state 
but \emph{hidden error} $\Delta\hv^{(\ell)}=\widehat{\hv}^{(\ell)}-\hv^{(\ell)}$ changes across $\ell$. 
This section quantifies the growth by $\norm{\Delta\hv^{(\ell)}}_2$ and its relative size 
$
{\norm{\Delta\hv^{(\ell)}}_2}/
{\norm{\hv^{(\ell)}}_2}$.

\subsection{Hidden-Error Growth in Randomly Initialized and Pretrained Models}
\label{sec:residual-observation}
\begin{figure*}[t]
  \centering
  \includegraphics[width=0.7\textwidth]{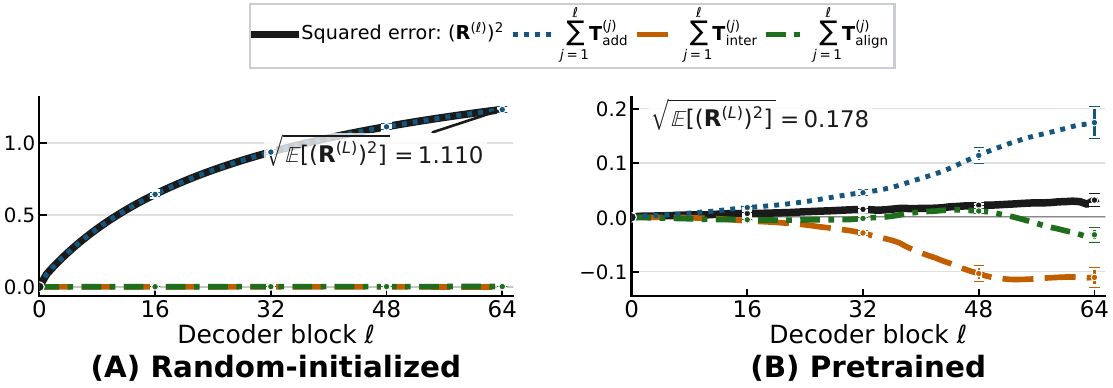}
  \caption{\textbf{Accumulated contributions of the three terms in
  Thm.~\ref{thm:relative-hidden-error-recurrence}.} Randomly initialized and
  pretrained Qwen3-32B on C4. Each colored curve accumulates one term over
  blocks $1$--$\ell$ and shows its averaged value across inputs.
  The black curve is their exact sum. Error bars show standard deviation across inputs.}
  \label{fig:sec3-counteraction-accounting}
\end{figure*}

NVFP4 weight-quantization reduces Qwen3-32B accuracy by only $0.43$ percentage points on average across six zero-shot benchmarks (Sec.~\ref{sec:preferential-output-preservation}). The final relative hidden-error norm is 
$
{\norm{\Delta\hv^{(L)}}_2}/
{\norm{\hv^{(L)}}_2}
\approx 0.15$ and cosine similarity is $\cos\angle\bigl(
\hv^{(L)},\compacthat{\hv}{(L)}
\bigr)\approx 0.98$ on average for pretrained models.
A direct intuition is that networks are robust to NVFP4 because each quantized weight matrix $Q(\Wv)$ remains close to its full-precision counterpart $\Wv$ (${\rm CosSim}(\Wv,Q(\Wv))\approx0.9955$ in App.~\ref{app:nvfp4-weight-closeness}).
This suggests that each quantized matrix multiplication introduces only a small error.
One might therefore expect the errors introduced across layers to remain limited, leaving a small final hidden error.

However, this account is incomplete.
Randomly initialized weights have \emph{nearly the same reconstruction cosine similarity} as the pretrained model
under NVFP4 quantization (App.~\ref{app:nvfp4-weight-closeness}).
If this weight-level similarity alone explained the small final hidden error,
the randomly initialized and pretrained models should therefore show similar hidden-error growth. 
However, the quantized randomly initialized Qwen3-32B reaches a $5.5\times$ larger final absolute hidden-error norm and a $6.7\times$ larger relative hidden-error norm than its pretrained checkpoint 
(Fig.~\ref{fig:hidden-error-growth}). 
The same pattern appears for OLMo3-7B: the initialized step-0 has $20.2\times$ larger final absolute error and $3.7\times$ larger relative error than the final checkpoint.

The growth trends across layers also differ. Over the 64 layers,
randomly initialized Qwen3-32B follows the consistently growing trends
$\Eb\norm{\Delta\hv^{(\ell)}}_2\propto\ell^{0.80}$ and
$\Eb\bigl[
    {\norm{\Delta\hv^{(\ell)}}_2}/
    {\norm{\hv^{(\ell)}}_2}
\bigr]
\propto\ell^{0.30}$, while the pretrained relative-error curve even stops increasing in the later layers
(details in App.~\ref{app:random-hidden-error-growth-rate}). Weight-level similarity alone is therefore insufficient to explain robustness across depth;
there should be a model-side mechanism that distinguishes randomly initialized models from pretrained models. 
The next subsection identifies counteraction between block-input error $\Delta \hv^{(\ell-1)}$ and block-update error $\Delta \uv^{(\ell)}$ as a major contributor to this slower growth.

\subsection{Counteraction between Block-Input and Block-Update Errors}
\label{sec:residual-absolute-counteraction}

Eq.~\eqref{eq:setup-paired-block} shows hidden error after block $\ell$ is
$\Delta\hv^{(\ell)}=\Delta\hv^{(\ell-1)}+\Delta\uv^{(\ell)}$, where
$\Delta\hv^{(\ell-1)}$ is the \textbf{\emph{input error}} inherited from earlier blocks and
$\Delta\uv^{(\ell)}$ is the \textbf{\emph{block-update error}}. Here is the recurrence:

\begin{proposition}[Squared hidden-error recurrence]
\label{prop:absolute-hidden-error-recurrence}
For every Transformer block $\ell\in\{1,\ldots,L\}$,
\begin{equation}
\norm{\Delta\hv^{(\ell)}}_2^2
-\norm{\Delta\hv^{(\ell-1)}}_2^2
=
{
\norm{\Delta\uv^{(\ell)}}_2^2
}
+
{
2\iprod{\Delta\hv^{(\ell-1)},\Delta\uv^{(\ell)}}
}
\label{eq:absolute-hidden-error-recurrence}
\end{equation}
\end{proposition}

The second term on the right determines
whether the block-update error $\Delta\uv^{(\ell)}$ increases or cancels the
block-input hidden error $\Delta\hv^{(\ell-1)}$. We call the latter case
\textbf{\emph{counteraction}}:
$\iprod{\Delta\hv^{(\ell-1)},\Delta\uv^{(\ell)}}<0$. Counteraction offsets
part of the newly introduced error, and slows the hidden-error norm growth.

For pretrained Qwen3-32B, over layers $1$--$48$, the interaction term $2\sum\Eb\iprod{\Delta\hv^{(\ell-1)},
\Delta\uv^{(\ell)}}$ cumulatively cancels $50.2\%$ of the block-update error term $\sum\Eb\norm{\Delta\uv^{(\ell)}}_2^2$ contribution, explaining the slow growth in Fig.~\ref{fig:hidden-error-growth}(A). 
At random initialization, $\bigl|
2\Eb\iprod{\Delta\hv^{(\ell-1)},
\Delta\uv^{(\ell)}}
\bigr|$ is at most $0.16\%$ of
$\Eb\norm{\Delta\uv^{(\ell)}}_2^2$ for each layer, consistent with the near-zero cosines in Fig.~\ref{fig:hidden-error-growth}(C). 
We next derive the recurrence for the \emph{relative hidden error}
${\norm{\Delta\hv^{(\ell)}}_2}/{\norm{\hv^{(\ell)}}_2}$ which normalizes for changes in hidden-state scale.

\begin{theorem}[Relative hidden-error recurrence]
\label{thm:relative-hidden-error-recurrence}
Define
$R^{(\ell)}:=\frac{\norm{\Delta\hv^{(\ell)}}_2}{\norm{\hv^{(\ell)}}_2}$. 
For every $\ell\in\{1,\ldots,L\}$ such that $\hv^{(\ell-1)},\hv^{(\ell)},\uv^{(\ell)}$ are nonzero,
{\small
\begin{align}
&\bigl(R^{(\ell)}\bigr)^2-\bigl(R^{(\ell-1)}\bigr)^2
\nonumber\\[-0.2em]
&=
\underbrace{
    \left(\frac{\norm{\uv^{(\ell)}}_2}{\norm{\hv^{(\ell)}}_2}\right)^2
    \Bigl[
    \left(\frac{\norm{\Delta\uv^{(\ell)}}_2}
            {\norm{\uv^{(\ell)}}_2}\right)^2
    -\bigl(R^{(\ell-1)}\bigr)^2
    \Bigr]
}_{\substack{
    \color{termcontrast}\text{\bfseries relative additional error 
    (\bfseries block-update vs. block-input): }\termlabel{termcontrast}{T_{\rm add}}}}
\!+\!
\underbrace{
    \frac{
        2\iprod{\Delta\hv^{(\ell-1)},\Delta\uv^{(\ell)}}
    }{
        \norm{\hv^{(\ell)}}_2^2
    }
}_{\substack{
    \color{terminteraction}\text{\bfseries error interaction: }
    \termlabel{terminteraction}{T_{\rm inter}}}}
\underbrace{
    -\frac{
        2\iprod{\hv^{(\ell-1)},\uv^{(\ell)}}
    }{
        \norm{\hv^{(\ell)}}_2^2
    }
    \bigl(R^{(\ell-1)}\bigr)^2
}_{\substack{
    \color{termalignment}\text{\bfseries residual-norm contribution: }
    \termlabel{termalignment}{T_{\rm align}}}}.
\label{eq:relative-hidden-error-recurrence}
\end{align}
}
\end{theorem}
\begin{figure*}[t]
  \centering
  \includegraphics[width=0.92\textwidth]{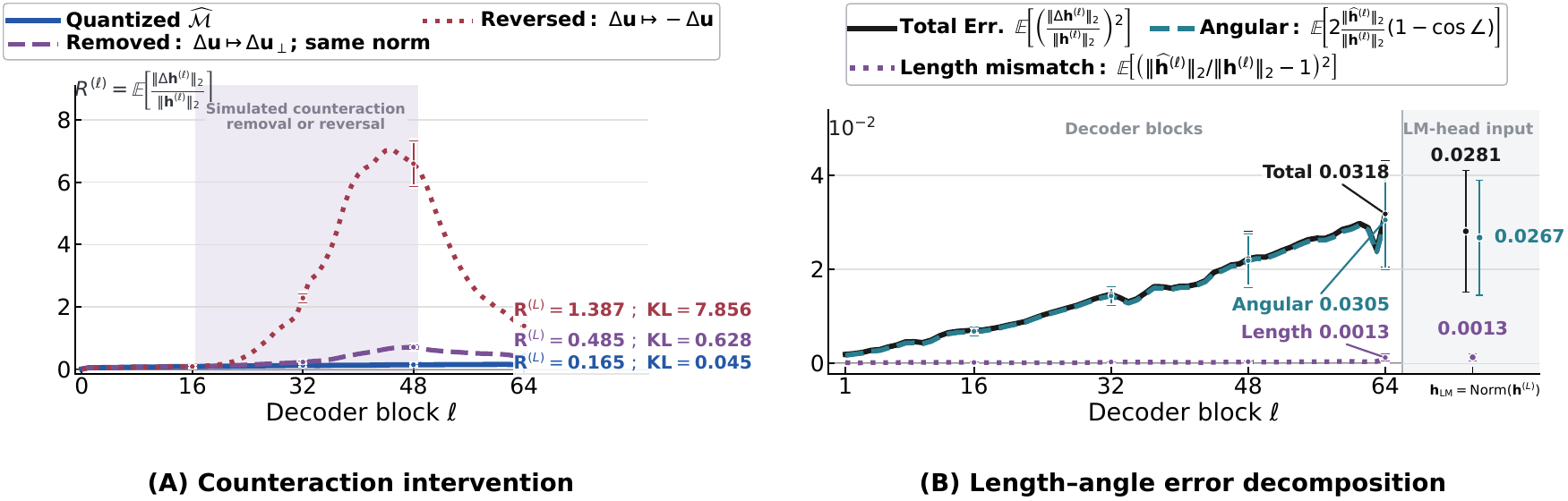}
  \caption{\textbf{Counteraction interventions and length--angle decomposition
  of the hidden error.} Qwen3-32B on C4.
  \textbf{(A)} Relative hidden error after counteraction removal or reversal. 
  \textbf{(B)} Terms of
  Prop.~\ref{prop:hidden-error-magnitude-direction}; the
  separated right region shows the LM-head input terms. Error bars
  show SD.}
  \label{fig:sec3-counteraction-effects}
\end{figure*}

A detailed derivation is given in
App.~\ref{app:hidden-error-recurrence-proofs}, and
Fig.~\ref{fig:sec3-counteraction-accounting} quantifies the three terms for
Qwen3-32B:
\begingroup
\setlength{\leftmargini}{0.8em}
\begin{itemize}
\setlength{\topsep}{0pt}
\setlength{\partopsep}{0pt}
\setlength{\itemsep}{0pt}
\setlength{\parsep}{0pt}
\setlength{\leftmargin}{0pt}
\setlength{\labelwidth}{0.75em}
\setlength{\labelsep}{0.35em}
\setlength{\itemindent}{0pt}
\setlength{\listparindent}{0pt}
\item $\termlabel{termcontrast}{T_{\rm add}}$ compares the relative error in the current block update with
the relative error already present at the block input, scaled by the
update-to-hidden norm ratio. 
\item $\termlabel{terminteraction}{T_{\rm inter}}$ measures the
signed interaction between the block-input error
$\Delta\hv^{(\ell-1)}$ and block-update error $\Delta\uv^{(\ell)}$.
This term encompasses the
\emph{counteraction} that results in the slower error growth: with $\termlabel{terminteraction}{T_{\rm inter}}<0$ for most blocks in pretrained Qwen3-32B, the term cumulatively cancels $63.4\%$ of
$\termlabel{termcontrast}{T_{\rm add}}$ over layers $1$--$L$, substantially slowing the growth of $\bigl(R^{(\ell)}\bigr)^2$.
\item $\termlabel{termalignment}{T_{\rm align}}$ captures how
the change of hidden-state norm affects
the relative error. 
If 
$\iprod{\hv^{(\ell-1)},\uv^{(\ell)}}>0$, the update increases
$\norm{\hv^{(\ell)}}_2$ and reduces the relative error. In
pretrained Qwen3-32B, this term cancels $18.4\%$ of
$\termlabel{termcontrast}{T_{\rm add}}$ cumulatively, and most
of the
cancellation occurs during later blocks.
\end{itemize}
\endgroup

Together, $\termlabel{terminteraction}{T_{\rm inter}}$ and $\termlabel{termalignment}{T_{\rm align}}$ cumulatively cancel $81.8\%$ of
$\termlabel{termcontrast}{T_{\rm add}}$ in pretrained Qwen3-32B over layers
$1$--$L$. 
At random initialization, the cumulative sum of these two terms is
only $0.09\%$ of $\termlabel{termcontrast}{T_{\rm add}}$ over layers $1$--$L$.
\textbf{These measurements identify \emph{counteraction} as a major factor in slowing down the hidden-error growth in the pretrained model}: across blocks, the block-update error $\Delta\uv^{(\ell)}$ tends to point \emph{against} the block-input hidden error $\Delta\hv^{(\ell-1)}$ to slow the error growth.

\subsection{Hidden-Error Growth without Counteraction}
\label{sec:residual-relative-decomposition}

After showing the significance of counteraction, we test what happens when counteraction is removed. 
We construct layer-wise
hidden-state trajectories in which each block-update error $\Delta\uv$ keeps
its norm, but its \emph{direction} is modified to be \emph{counteraction-free}. 
We compare the hidden-error trajectories and the final KL divergence across intervention policies.
Concretely,  
at every intervened block, we use one of the hidden-state recursions:
\begin{align*}
\textbf{Removal:}\quad
\widehat{\hv}_{\rm rem}^{(\ell)}
=
\widehat{\hv}_{\rm rem}^{(\ell-1)}
+\uv^{(\ell)}
\textcolor{removalpurple}{\boldsymbol{+}\Delta\uv_{\perp}^{(\ell)}};\qquad
\textbf{Reversal:}\quad
\widehat{\hv}_{\rm rev}^{(\ell)}
=
\widehat{\hv}_{\rm rev}^{(\ell-1)}
+\uv^{(\ell)}
\textcolor{reversalred}{\boldsymbol{-}\Delta\uv^{(\ell)}}.
\end{align*}
In both recursions, $\Delta\uv^{(\ell)}$ follows
Eq.~\eqref{eq:setup-paired-block} and is recomputed along the modified
trajectory; $\textcolor{removalpurple}{\Delta\uv_{\perp}^{(\ell)}}$ is orthogonal to the block-input error
$\Delta\hv^{(\ell-1)}$. Both interventions are applied in blocks 17--48, where
counteraction is strongest in Qwen3-32B; reversal is applied only when the interaction is negative.

At the final block, removal raises the relative hidden error by $2.94\times$,
while reversal raises it by $8.41\times$. Both interventions also sharply increase KL.
Moreover, after counteraction removal, 
$\termlabel{termcontrast}{T_{\rm add}}$ increases by 4.3× cumulatively (App.~\ref{app:counteraction-intervention}),
which indicates that counteraction not only contributes through the negative term $\termlabel{terminteraction}{T_{\rm inter}}$, but also alters the subsequent trajectory in a way that limits $\termlabel{termcontrast}{T_{\rm add}}$.
The interventions provide causal evidence that \textbf{counteraction is a major factor limiting hidden-error growth}. Sec.~\ref{sec:lmhead-geometry} examines how the remaining hidden-state error affects the outputs.

\section{The Effect of Final Hidden Error on Output Quality}
\label{sec:lmhead-geometry}

The quantized Transformer blocks leave a nonzero error at the final decoder
state $\hv^{(L)}$: the mean relative error norm
$\norm{\Delta\hv^{(L)}}_2/\norm{\hv^{(L)}}_2$ across the three datasets is
${\sim}0.245$ (App.~\ref{app:cross-model-relative-recurrence}). To see what
kind of error survives to the LM head input, we decompose it into its \emph{length} and
\emph{angle} components.
\begin{figure}[t]
  \centering
  \includegraphics[width=\textwidth]{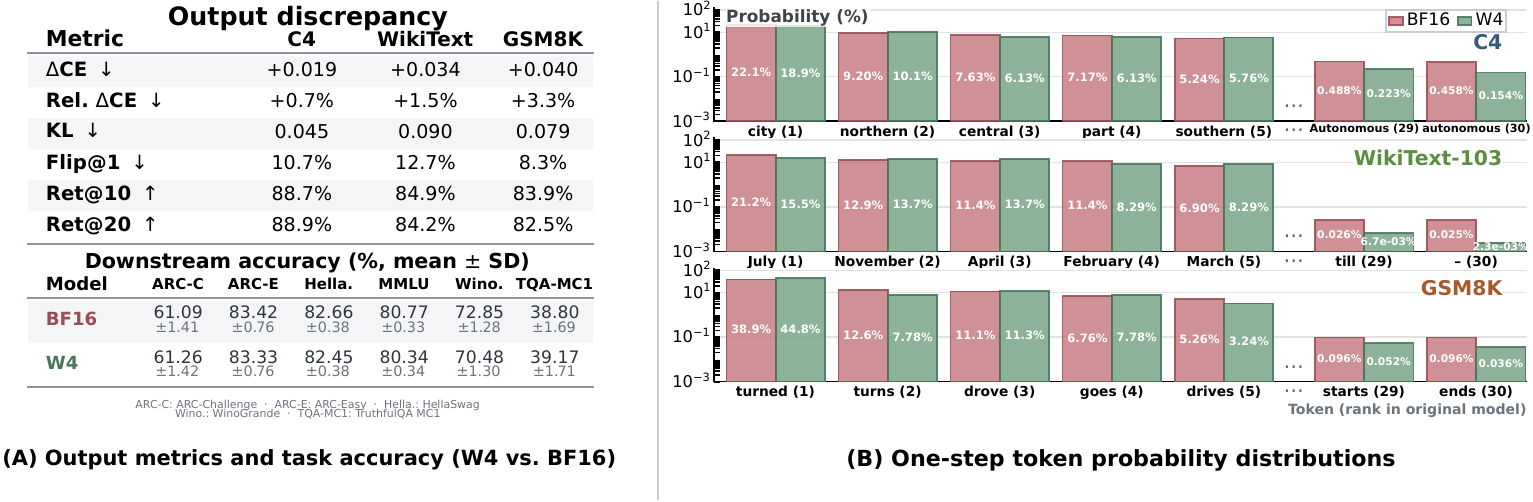}
  \caption{\textbf{Output metrics and paired next-token probabilities after W4
  quantization.} Qwen3-32B.
  \textbf{(A)} Output-quality metrics and zero-shot benchmark accuracy.
  \textbf{(B)} Paired BF16/W4 probabilities at one representative next-token
  position per dataset, showing BF16 ranks 1--5 and 29--30.}
  \label{fig:sec4-output-geometry}
\end{figure}

\begin{proposition}[Length--angle decomposition of relative hidden error]
\label{prop:hidden-error-magnitude-direction}
Define
$\Delta\hv^{(\ell)}:=\compacthat{\hv}{(\ell)}-\hv^{(\ell)}$.
For every block $\ell$ such that $\hv^{(\ell)}\ne\bm{0}$ and
$\compacthat{\hv}{(\ell)}\ne\bm{0}$,
\begin{equation}
\Bigl(
\frac{\norm{\Delta\hv^{(\ell)}}_2}{\norm{\hv^{(\ell)}}_2}
\Bigr)^2
=
\Bigl(
\frac{\norm{\widehat{\hv}^{(\ell)}}_2}
    {\norm{\hv^{(\ell)}}_2}-1
\Bigr)^2
+2\frac{\norm{\widehat{\hv}^{(\ell)}}_2}
       {\norm{\hv^{(\ell)}}_2}
\left[
1-\cos\angle\bigl(\hv^{(\ell)},\widehat{\hv}^{(\ell)}\bigr)
\right].
\label{eq:hidden-error-magnitude-direction}
\end{equation}
In particular, if
$\norm{\compacthat{\hv}{(\ell)}}_2=\norm{\hv^{(\ell)}}_2$, then
$
\left(
\frac{\norm{\Delta\hv^{(\ell)}}_2}{\norm{\hv^{(\ell)}}_2}
\right)^2
=2[
1-\cos\angle(\hv^{(\ell)},  \compacthat{\hv}{(\ell)})
].
$
\end{proposition}

We refer to the two terms on the right-hand side of Eq.~\eqref{eq:hidden-error-magnitude-direction} as the \emph{length}-mismatch
contribution and the \emph{angular} contribution, respectively. Fig.~\ref{fig:sec3-counteraction-effects}(B) shows that the angular contribution accounts
for 88.9--98.7\% of the squared relative hidden error across depth, while
changes in hidden-state norm remain small. After final normalization, this
corresponds to an average $12.47^\circ$ rotation of the LM-head input
(Fig.~\ref{fig:sec4-topk-protection}(A)). We next measure its effect on output
quality (Sec.~\ref{sec:preferential-output-preservation}) and explain theoretically why
top-ranked token probabilities are more stable in the outputs
(Sec.~\ref{sec:lmhead-top-rank-geometry}).

\subsection{Output Distributions and Downstream Accuracy after Quantization}
\label{sec:preferential-output-preservation}

Hereafter, \textcolor{cleanred}{\emph{BF16}} denotes the original model
$\color{cleanred}{\Mc}$ evaluated in bfloat16, and
\textcolor{quantgreen}{\emph{W4}} denotes
$\color{quantgreen}{\widehat{\Mc}}$ with only its Transformer block linear
weights quantized to NVFP4. We show that the output quality is mostly retained after quantization, especially for top-ranked tokens.

First, we compare the output quality of BF16 and W4 models by loss and divergence:
$\Delta\mathrm{CE}:=\mathrm{CE}_{\mathrm{W4}}-\mathrm{CE}_{\mathrm{BF16}}$, its
relative value $
\frac{
\Delta\mathrm{CE}}
{
\mathrm{CE}_{\mathrm{BF16}}
}$, and the
forward KL divergence $D_{\mathrm{KL}}(\pv\|\widehat{\pv})$.
Fig.~\ref{fig:sec4-output-geometry}(A)
reports all three metrics on each dataset (details in
App.~\ref{app:main-figure-protocols}), and the accuracy drops by only
$0.43$ percentage points on average across six benchmarks. Across other models, including
Qwen3-30B-A3B, Qwen3-8B, and OLMo3-32B, quantization also largely preserves output quality (App.~\ref{app:sec5-cross-model-output}).

Beyond the metrics, the BF16 model's highest-scoring tokens largely remain top-ranked in the W4 model.
Let $\pi_r\in\Vc$ denote the rank-$r$ token of BF16 model's score $\zv$. Define
\begin{equation*}
\text{Flip@1}
:=\Pr\!\left[
\operatorname{Top}_1(\widehat{\zv})
\ne
\operatorname{Top}_1(\zv)
\right],
\qquad
\text{Ret@}K
:=\Eb\!\Bigl[
\frac{
\left|
\operatorname{Top}_K(\zv)
\cap
\operatorname{Top}_K(\widehat{\zv})
\right|
}{K}
\Bigr].
\end{equation*}
Across the three datasets, Flip@1 is $8.3\%$--$12.7\%$, while
Ret@10 and Ret@20 are both about $85\%$ on average
(Fig.~\ref{fig:sec4-output-geometry}(A)).
In the examples shown in Fig.~\ref{fig:sec4-output-geometry}(B), the average
relative probability change over ranks 29--30 is $2.5$--$4.4\times$ that over
ranks 1--5. Sec.~\ref{sec:lmhead-top-rank-geometry} tests this pattern across
all evaluated positions and explains why top-ranked token scores are less
sensitive to quantization.

\begin{figure}[t]
  \centering
  \includegraphics[width=\textwidth]{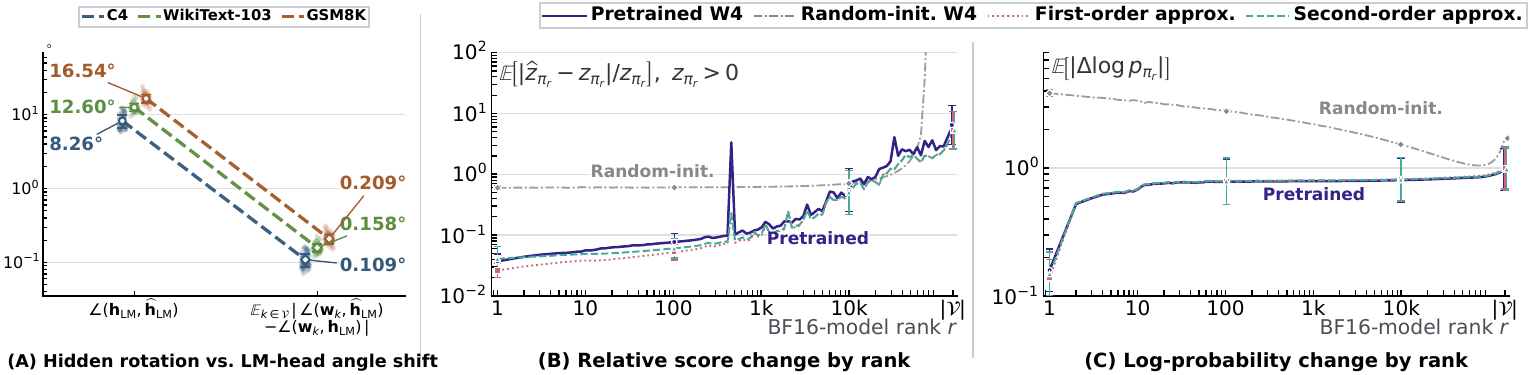}
  \caption{\textbf{Top-ranked token scores and probabilities are more stable
  after quantization in the pretrained model.} Qwen3-32B on C4,
  WikiText-103, and GSM8K text. \textbf{(A)} Across datasets, the LM-head input
  rotates by $8.26^\circ$ to $16.54^\circ$, whereas the vocabulary-mean
  projection-angle change is only $0.109^\circ$ to $0.209^\circ$.
  \textbf{(B)} Relative score error increases toward lower-ranked tokens and is
  much smaller in the pretrained model than at random initialization. 
  The theoretical approximation is under $\rho=1$ of Thm.~\ref{thm:lmhead-tangent-attenuation}. 
  \textbf{(C)} Log-probability error follows the same rank trend, and the approximations follow 
  Thm.~\ref{thm:log-probability-approximations}. Error bars
  show one log-SD; details are given in
  App.~\ref{app:main-figure-protocols}.
  }
  \label{fig:sec4-topk-protection}
\end{figure}

\subsection{Preferential Preservation of Top-Ranked Scores and Probabilities}
\label{sec:lmhead-top-rank-geometry}

In this subsection, we find that top-ranked token scores have smaller
relative errors after quantization and theoretically explain why this trend arises through the geometry of the LM head. This preferential
preservation matters because top-ranked tokens govern the high-probability
region of the output.
Here $\widehat{\hv}_{\rm LM}$ and $\widehat{\zv}$ are the LM-head input and
score vector of the W4 model, and $\widehat{\pv}$ is the corresponding
probability vector;
$\hv_{\rm LM}$, $\zv$, and $\pv$ denote their counterparts in the unquantized model.

We first examine how the score $\zv=(z_1,z_2,\dots,z_{|\Vc|})^{\top}$ changes after quantization. 
The score of token $k\in\Vc$ is the projection of the LM-head input onto
the corresponding LM-head weight vector:
$z_k=\iprod{\wv_k,\hv_{\rm LM}}$.
Because the BF16 and W4 models share the same LM head, a token score is changed only by the input norm
$\norm{\hv_{\rm LM}}_2$ and \textbf{\emph{projection angle}}
$\angle(\wv_k,\hv_{\rm LM})$; App.~\ref{app:lmhead-geometry-proofs} gives the
exact decomposition.
The average relative
input-norm change
$\Eb\bigl|\norm{\widehat{\hv}_{\rm LM}}_2/\norm{\hv_{\rm LM}}_2-1\bigr|$
is only $3.7\%$. We therefore focus on quantifying how changes in the projection angles affect the token scores.
Fig.~\ref{fig:sec4-topk-protection} gives two observations:
\begingroup
\setlength{\leftmargini}{0.8em}
\begin{itemize}
\setlength{\topsep}{2pt}
\setlength{\partopsep}{0pt}
\setlength{\itemsep}{2pt}
\setlength{\parsep}{0pt}
\setlength{\itemindent}{0pt}
\setlength{\listparindent}{0pt}
\item \textbf{Observation 1: projection angles change much less than the
LM-head input.} In Fig.~\ref{fig:sec4-topk-protection}(A), the input rotates
by $\angle(\hv_{\rm LM},\widehat{\hv}_{\rm LM})\approx 12.5^\circ$, whereas
the vocabulary-mean projection-angle change is
$\lvert
\angle(\wv_k,\hv_{\rm LM})-\angle(\wv_k,\widehat{\hv}_{\rm LM})
\rvert\approx 0.159^\circ$, a $78.6\times$ attenuation.
\item \textbf{Observation 2: top-ranked tokens have smaller score and log-probability changes for the quantized pretrained model.} For BF16 rank-$r$ token $\pi_r$,
the relative score change grows
with rank in Fig.~\ref{fig:sec4-topk-protection}(B). The log-probability
change $\Eb[|\Delta\log p_{\pi_r}|]$ follows the same rank trend, while both
quantities are much larger at random initialization
(Fig.~\ref{fig:sec4-topk-protection}(B--C)).
\end{itemize}
\endgroup

To explain these, Thm.~\ref{thm:lmhead-tangent-attenuation} below derives how the hidden-state rotation
$\alpha=\angle(\hv_{\rm LM},\widehat{\hv}_{\rm LM})$
changes both projection angles and relative scores.
Intuitively, in high dimension, only a small component of this rotation affects the projection onto any fixed LM-head weight vector $\wv_k$, 
so the angle between
$\wv_k$ and the hidden state changes much less than the hidden state itself.
Moreover, a smaller angle between $\wv_k$ and
$\hv_{\rm LM}$ makes the relative score less sensitive to the same rotation, and higher-ranked tokens tend to have such smaller angles empirically
(App.~\ref{app:sec5-cross-model-output}), producing the rank-dependent trend in
Fig.~\ref{fig:sec4-topk-protection}(B).
Thm.~\ref{thm:log-probability-approximations} then links these score changes
to the log-probability changes in
Fig.~\ref{fig:sec4-topk-protection}(C).

\begin{theorem}[Projection-angle and score changes under an
isotropic rotation of the LM-head input]
\label{thm:lmhead-tangent-attenuation}
Let $d\ge3$ and let LM-head input $\hv_{\rm LM}$, $\widehat{\hv}_{\rm LM}$, and LM-head weight vector 
$\wv_k$ be nonzero. Define
\begin{equation*}
\alpha:=
\angle(\hv_{\rm LM},\widehat{\hv}_{\rm LM}),
\qquad
\theta_k:=
\angle(\wv_k,\hv_{\rm LM}),
\qquad
\widehat\theta_k:=
\angle(\wv_k,\widehat{\hv}_{\rm LM}),
\qquad
\rho:=\tfrac{
        \norm{\widehat{\hv}_{\rm LM}}_2
}{
        \norm{\hv_{\rm LM}}_2}
\end{equation*}
Conditional on the measured $\alpha$, assume that the direction of the
rotation from $\hv_{\rm LM}$ to $\widehat{\hv}_{\rm LM}$ is uniformly
distributed over all directions orthogonal to $\hv_{\rm LM}$. For
$z_k\ne0$ and $\alpha,\theta_k\in (0,\pi)$, with all expansions taken as $\alpha\to0^+$ for fixed $d$ and
$\theta_k$, let $\Delta z_k:=\widehat z_k-z_k$. Then
\begin{align}
\Eb\left|\widehat{\theta}_k-\theta_k\right|
&=\alpha\mu_d+O(\alpha^2),
\qquad
\mu_d:=\frac{\Gamma\!\left(\tfrac{d-1}{2}\right)}
            {\sqrt{\pi}\,\Gamma\!\left(\tfrac{d}{2}\right)}
=\sqrt{\frac{2}{\pi(d-1)}}
\left(1+O(d^{-1})\right),
\label{eq:lmhead-tangent-angle-rate}
\\
\frac{\Delta z_k}{z_k}
&=(\rho\cos\alpha-1)
+\rho\tan\theta_k\sin\alpha\cdot q_k,
\qquad
q_k^2\sim\operatorname{Beta}\!\left(\frac12,\frac{d-2}{2}\right),
\label{eq:lmhead-relative-score-distortion}
\\
\Eb\left|\frac{\Delta z_k}{z_k}\right|
&=\alpha\lvert\tan\theta_k\rvert\mu_d
+O(\alpha^2),
\qquad\text{if }\rho=1,
\label{eq:lmhead-relative-score-rate}
\end{align}
Here $q_k \text{ is symmetric about zero}$. The corresponding second-order estimate is given in
Eq.~\eqref{eq:lmhead-relative-score-second-order}.
\end{theorem}

A detailed derivation is given in App.~\ref{app:lmhead-tangent-rotation-proof}.  
The first-order and second-order analytic references from
Eqs.~\eqref{eq:lmhead-relative-score-rate} and \eqref{eq:lmhead-relative-score-second-order} 
set $\rho=1$ to remove the modest norm difference and
isolate the effect of changing the LM-head input direction, while the empirical W4 score-error curves in
Fig.~\ref{fig:sec4-topk-protection}(B) retain the measured
input-norm changes. 
We next relate the score changes to the log-probability changes:

\begin{theorem}[Log-probability change estimates]
\label{thm:log-probability-approximations}
At a fixed next-token position, let
$\Delta\log p_k:=\log\widehat p_k-\log p_k$. The exact identity and expansions for uniformly small centered score changes are
\begin{equation}
\begin{aligned}
\Delta\log p_k
&=\Delta z_k-\log\Eb_{j\sim\pv}e^{\Delta z_j}
=\Delta z_k-\Eb_{j\sim\pv}\Delta z_j-\KL{\pv}{\widehat{\pv}}
&&\text{\rm(exact)},
\\[-2pt]
\Delta\log p_k
&=\Delta z_k-\Eb_{j\sim\pv}\Delta z_j
+O\!\left(\operatorname{Var}_{j\sim\pv}(\Delta z_j)\right)
&&\text{\rm(1-order)},
\\[-2pt]
\Delta\log p_k
&=\Delta z_k-\Eb_{j\sim\pv}\Delta z_j
-\tfrac12\operatorname{Var}_{j\sim\pv}(\Delta z_j)
+O(
\Eb_{j\sim\pv}
\left|\Delta z_j-\Eb_{i\sim\pv}\Delta z_i\right|^3
)
&&\text{\rm(2-order)}.
\end{aligned}
\label{eq:log-probability-approximations}
\end{equation}
\end{theorem}

A detailed derivation is given in
App.~\ref{app:log-probability-approximation-proof}. Thms.~\ref{thm:lmhead-tangent-attenuation}
and~\ref{thm:log-probability-approximations} account for the two observations
above:
\begingroup
\setlength{\leftmargini}{0.8em}
\begin{itemize}
\setlength{\topsep}{2pt}
\setlength{\partopsep}{0pt}
\setlength{\itemsep}{2pt}
\setlength{\parsep}{0pt}
\setlength{\itemindent}{0pt}
\setlength{\listparindent}{0pt}

\item
\textbf{Why do projection angles change so little?}
Eq.~\eqref{eq:lmhead-tangent-angle-rate} in Thm.~\ref{thm:lmhead-tangent-attenuation} gives
$\Eb|\widehat{\theta}_k-\theta_k|
=\alpha\mu_d+O(\alpha^2)$.
Thus, in high dimension, the rotation of the hidden state is strongly
attenuated when measured as the angle change to any fixed LM-head row vector.
For Qwen3-32B, the predicted attenuation
$\mu_d^{-1}\approx89.7\times$ is close to the measured $78.6\times$ in
Fig.~\ref{fig:sec4-topk-protection}(A).

\item
\textbf{Why are higher-ranked tokens more stable?}
Eqs.~\eqref{eq:lmhead-relative-score-distortion}
and~\eqref{eq:lmhead-relative-score-rate} in Thm.~\ref{thm:lmhead-tangent-attenuation} show that, for positive scores,
relative score sensitivity is governed by
$\lvert\tan\theta_{\pi_r}\rvert$.
A smaller projection angle therefore gives a smaller relative score change
under the same hidden-state rotation.
Empirically, higher-ranked tokens tend to have smaller projection angles
(App.~\ref{app:sec5-cross-model-output}), matching the rank dependence in
Fig.~\ref{fig:sec4-topk-protection}(B).
Thm.~\ref{thm:log-probability-approximations} then connects these score
changes to log-probability changes, and both approximations reproduce the
increasing rank trend in Fig.~\ref{fig:sec4-topk-protection}(C).

\end{itemize}
\endgroup

These results also explain why a non-negligible rotation of
$\hv_{\rm LM}$ produces only modest output changes in
Fig.~\ref{fig:sec4-output-geometry}(A). 
KL remains small because it weights tokens by their original probabilities, 
and the highest-probability tokens are best preserved after quantization. 
CE remains small because the probability assigned to the ground-truth token changes little on average.

\section{Discussion}
\label{sec:scope}

\subsection{Consistency across Models, Quantization Settings, and PTQ Algorithms}
\label{sec:scope-model-quantization}

The counteraction phenomenon introduced in Sec.~\ref{sec:residual-counteraction}
persists across pretrained Qwen, OLMo, and Gemma models, including dense
and mixture-of-experts (MoE) architectures
(App.~\ref{app:cross-model-relative-recurrence}).
The layer-wise curve of
$\Eb\cos\angle(\Delta\hv^{(\ell-1)},\Delta\uv^{(\ell)})$ changes little and remains negative in most blocks when quantizing weights, activations, or both (Fig.~\ref{fig:sec5-quantization-scope}(A)), and across multiple PTQ
algorithms (Fig.~\ref{fig:sec5-quantization-scope}(B)), 
showing that calibration-based PTQ changes the counteraction geometry little (App.~\ref{app:sec5-quantization-transfer}).
Random-initialization
comparisons and intervention experiments on additional models further support the role of counteraction in slowing hidden-error growth
(Apps.~\ref{app:random-hidden-error-growth-rate}
and~\ref{app:counteraction-intervention}).

\begin{figure}[t]
  \centering
  \includegraphics[width=0.85\textwidth]{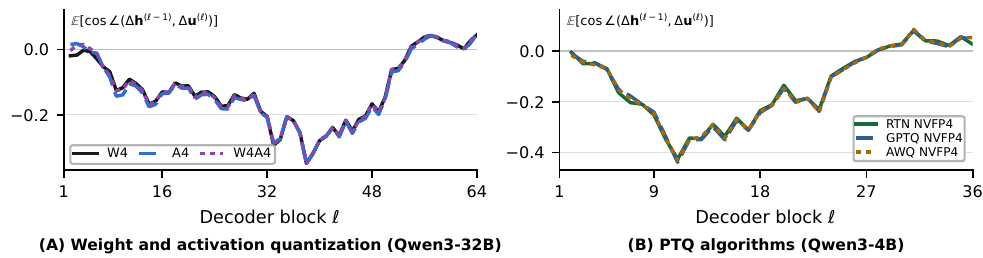}
  \caption{\textbf{Blockwise counteraction across quantization components
  and PTQ methods.}
  \textbf{(A)} Qwen3-32B on C4 under W4 (weight-only), A4 (activation-only), and W4A4 (both quantized).
  \textbf{(B)} Qwen3-4B under RTN, GPTQ, and AWQ, all quantizing the same weights with NVFP4 format.}
  \label{fig:sec5-quantization-scope}
\end{figure}

Preferential preservation of top-ranked scores and probabilities also transfers
across models. Higher-ranked tokens tend to have smaller projection angles to
the LM-head input and smaller relative score errors
(App.~\ref{app:sec5-cross-model-output}), and this rank dependence persists
across softmax temperatures (App.~\ref{app:sec5-temperature}).

\subsection{Analysis of the Source of Counteraction}
\label{sec:counteraction-source}

What produces this negative error interaction that slows hidden-error growth
(Sec.~\ref{sec:residual-counteraction})? Prior work found that residual updates
across Transformer layers can partially cancel earlier
updates. When an earlier contribution is rescaled, the later block adjusts its opposing update, showing that its response depends on the input signal~\citep{patrawala2025llmlayerscorrecteachother}. Does a block's input hidden error similarly induce a change in its update that points against and partially cancels that error?

This self-correcting tendency is much stronger for \emph{errors}
than for the native hidden-state/update pairs. 
Specifically,
$\cos\angle(\hv^{(\ell-1)},\uv^{(\ell)})$ and
$\cos\angle(
\compacthat{\hv}{(\ell-1)},
\compacthat{\uv}{(\ell)})$
are positive in $65.6\%$ and $64.1\%$ of the blocks, whereas interactions between the errors 
$\cos\angle(\Delta\hv^{(\ell-1)},\Delta\uv^{(\ell)})$ are \textbf{negative} in $82.5\%$ of the blocks
(App.~\ref{app:block-response-attribution}). Thus, the strong negative relation
is not a general tendency for a block update to oppose its input hidden state.
It emerges more consistently between the accumulated input \emph{hidden error} and the new \emph{block-update error}.

To identify its source, we separate
$\Delta\uv^{(\ell)}=
{\color{quantgreen}\compacthat{\Fv}{(\ell)}}(\compacthat{\hv}{(\ell-1)})-
{\color{cleanred}\Fv^{(\ell)}}(\hv^{(\ell-1)})$
into two parts. The direct \emph{weight effect} is
${\color{quantgreen}\compacthat{\Fv}{(\ell)}}(\hv^{(\ell-1)})-
{\color{cleanred}\Fv^{(\ell)}}(\hv^{(\ell-1)})$, which changes the weights
at a fixed input. The \emph{input-error response} is
${\color{quantgreen}\compacthat{\Fv}{(\ell)}}(\compacthat{\hv}{(\ell-1)})-
{\color{quantgreen}\compacthat{\Fv}{(\ell)}}(\hv^{(\ell-1)})$, which keeps the weight fixed and changes only its input. 
Although the two parts have comparable magnitudes, the
input-error response contributes $ 99.9\%$ of the negative cumulative
interaction
$\sum_{\ell=1}^{L}\termlabel{terminteraction}{T_{\rm inter}^{(\ell)}}$, 
compared with less than $0.1\%$ from the direct weight effect (App.~\ref{app:block-response-attribution}).
Counteraction therefore comes \textbf{mainly from the block's response to its input error}, which tends to oppose accumulated hidden error and slow its propagation.

\subsection{Rare Failures in Hidden-Error Propagation}
\label{sec:scope-data-position}

Counteraction significantly slows the hidden-error growth, which provides a self-correction effect that enhances the model's robustness to the quantization noise.
However, at rare decoding positions, the perturbation enters a regime where
counteraction is no longer sufficient to keep the quantized hidden-state
trajectory close to the original model, producing a large hidden-norm mismatch.
These rare failures produce extreme values that can dominate the averages and make the typical error dynamics harder to characterize.
We therefore exclude these positions from the recurrence-based aggregate statistics (only ${\sim}0.28\%$ of positions are excluded) using the norm-based rule in
App.~\ref{app:main-figure-protocols} and analyze them separately in
App.~\ref{app:relative-recurrence-position-selection}. These cases show that large
hidden-state norm distortion can cause output failures, but does not always do so
(Fig.~\ref{fig:app-sec3-filtered-fail-cases}).

\subsection{Limitations}
\label{sec:scope-limitations}

Our analysis has three limitations. First, we analyze the quantization mechanisms for individual
next-token predictions, but not how these mechanisms extend to multi-token generation. 
Second, our mechanism analysis relies on aggregate statistics 
and  does not precisely characterize every hidden-error trajectory or output change,
especially at rare positions with extreme errors.
Third, the LM-head theory uses a uniform-direction assumption for the approximation, although actual hidden-state rotation directions need not be uniform. Two questions remain open: whether training stochasticity produces
counteraction, as suggested by implicit-bias analyses of
SGD~\citep{pmlr-v119-smith20a}, and whether these findings can improve
post-training quantization.
\section{Conclusion}

This work studies why pretrained models maintain stable next-token predictions after
weight quantization. Comparing the original and quantized forward passes
reveals two mechanisms: (1) Even though the weights of randomly initialized and
pretrained models have nearly identical reconstruction cosine similarities under
NVFP4 quantization, pretrained models accumulate much less hidden error. Our exact recurrence isolates a mechanism specific to pretrained
models: the block-update error tends to oppose the block-input hidden error.
Interventions that remove this counteraction sharply increase the final hidden
error, providing causal evidence that counteraction is a major factor
limiting hidden-error growth. (2) The quantization-induced change in the final
hidden state is primarily a rotation rather than a change in norm. Our theory
shows that, in high dimensions, only a small component of this rotation affects
the projection onto any fixed LM-head weight vector and that LM-head geometry
makes the resulting score and log-probability changes smaller for higher-ranked
tokens. We verify both mechanisms across models and quantization settings.
\subsubsection*{Acknowledgments}
The authors sincerely thank Pengle Zhang and Zichen Liang for insightful discussions.

\bibliography{references}
\bibliographystyle{iclr2027_conference}

\newpage
\appendix
\etocdepthtag.toc{appendix}
{\centering\section*{Appendix Contents}}
\etocsettagdepth{main}{none}
\etocsettagdepth{appendix}{subsection}
\etocsettocstyle{}{}
{\tableofcontents}
\bigskip

\section{Related Work}
\label{app:related-work}

We organize related work around three stages of quantization error: how it is
introduced, how it propagates through Transformer blocks, and how it affects
model outputs.

\subsection{PTQ Optimization and Low-Precision Formats}

Research on weight-only LLM Post-Training Quantization (PTQ) largely frames PTQ as an optimization problem, exemplified
by the second-order weight reconstruction of GPTQ~\citep{frantar2023gptq} and the
activation-aware scaling of AWQ~\citep{lin2024awq}. A broad family of successors
extends the calibration objective to learned clipping and equivalent
transformations, block-level and cross-block reconstruction, asymmetric calibration
against full-precision outputs, alternating error correction, and per-weight loss
sensitivity~\citep{shao2024omniquant,li2021brecq,ding2025cbq,li2025gptaq,
zhang2026qronos,hu2025identifyingsensitiveweights}. These methods improve the
quantized model by reducing the errors introduced or retained during calibration.

On the format side, a broad microscaling benchmark finds that scale construction and
compatibility between formats and algorithms are critical at 4-bit
precision~\citep{zhang-etal-2026-benchmarking}. NVFP4 has become common in recent
large-model inference and training recipes, pairing fine-grained scaling for more
accurate 4-bit representation with methods for stable end-to-end low-precision
training~\citep{meng-etal-2026-arcquant,nvidia2025pretrainingnvfp4,
chen2026tetrajetv2,panferov2026quartetii}. We use NVFP4 round-to-nearest (RTN)
as our main setting (App.~\ref{app:nvfp4-format}). This lets us study the
pretrained model's response without adding format-specific optimization.

\subsection{Transformer Block Responses to Quantization}

Prior work has identified several factors that affect quantization error.
Input-dependent failures have been linked to residual
magnitudes, late-layer activations, and MLP gates
~\citep{chang-etal-2025-inputs}, 
while robustness across training trajectories depends strongly on learning-rate dynamics and other training hyperparameters~\citep{catalantatjer2026trainingdynamicsimpactposttraining}.
Broader accounts connect parameter-space robustness to training noise and flat minima~\citep{hochreiter1997flat,foret2021sam}. 
Related theory studies gradient variance during quantized training~\citep{chen2020statistical} 
and forward stability in quantized convolutional and graph residual networks~\citep{benyair2022quantized}. 
This literature helps explain when and why quantization error becomes large in the forward and backward passes, but not how errors from different Transformer blocks interact.

Some PTQ methods address error propagation during optimization.
Quantization Error Propagation
(QEP)~\citep{NEURIPS2025_df2034a5} carries upstream quantization error into
layer calibration and adjusts its reconstruction target to compensate for accumulated error.
Model-Preserving Adaptive
Rounding~\citep{tseng2026modelpreservingadaptiverounding} optimizes
rounding against approximate end-to-end output error because local
activation error can be a poor proxy for the final distribution. 
These methods modify the quantization objective to reduce propagated error. 
Rather than proposing another objective, we study the propagation process itself: 
how errors introduced by successive Transformer blocks interact and accumulate through depth. 
Clarifying this mechanism may inspire future PTQ methods.

A closer mechanistic parallel comes from full-precision Transformers.
\citet{patrawala2025llmlayerscorrecteachother} show that residual contributions from different layers can oppose one another.
Under quantization, we find that the block-update error tends to point against the hidden error already present at the block input (Sec.~\ref{sec:residual-counteraction}). 
This negative interaction between errors is much stronger than 
that between the hidden state and block update. 
Our decomposition further shows that this counteraction comes from how the quantized block responds to its input error
(Sec.~\ref{sec:counteraction-source};
App.~\ref{app:block-response-attribution}).

\subsection{Output Robustness under Quantization}

Output-aware PTQ starts from the observation that hidden-state reconstruction
does not necessarily preserve token predictions. These methods incorporate
output cross-entropy distortion or end-loss gradients into
calibration~\citep{edalati2025oac,kim2025guidedquant}.
Logit-aware Final-block Quantization
(LFQ)~\citep{lee2026lfqlogitawarefinalblockquantization} shows that lower final-block
hidden-state mean-squared error can worsen token predictions, and calibrates
the final decoder block through the full-precision LM head. These methods use
output sensitivity to improve quantization, but do not explain how the output probabilities respond to hidden error accumulated through the Transformer blocks.

Sequence-level studies examine the downstream effects of quantization on generation. 4-bit weight quantization
can retain reasoning performance in many
settings~\citep{liu2025quantizationhurtsreasoningempirical}, while more aggressive PTQ lengthens chains of thought and induces failures after correct intermediate
answers~\citep{lotfi2026quantizedreasoningmodelsthink}. These studies measure behavior in the more demanding end-to-end setting, where the original and quantized models follow their own generated histories. 
To understand the mechanism more clearly, 
we compare the two models on the same prefix at each next-token prediction and trace how quantization changes the hidden state, token scores, and output probabilities (Sec.~\ref{sec:lmhead-geometry}).

At each prediction, the LM head maps the final hidden state to token scores by projecting it onto rows that serve as output word
embeddings~\citep{press2017output}. Prior work studies the structure of the LM
head and the output distributions it can represent
\citep{yang2018softmax,cancedda2024spectral,cho2025token,finlayson2026signature}.
Related work also connects hidden states to outputs through intermediate-state
decoding and general perturbation
bounds~\citep{logitlens,belrose2023tunedlens,tsuzuku2018lipschitz}.
These studies characterize the structure of the LM head and general output
sensitivity, but do not explain why quantization affects high-ranked token
scores less than the rest of the vocabulary. 
We explain this trend
theoretically in Sec.~\ref{sec:lmhead-top-rank-geometry}.

\section{Notation}
\label{app:notation}

\begin{center}
\small
\setlength{\tabcolsep}{6pt}
\begin{tabular}{@{}p{0.30\textwidth}p{0.64\textwidth}@{}}
\toprule
Symbol & Meaning \\
\midrule
\multicolumn{2}{@{}l@{}}{\emph{Dimensions and indices}}\\
$d$;~ $L$;~ $\ell$ & hidden dimension; number of decoder blocks; block index \\
$\Vc$ & token vocabulary with $|\Vc|$ tokens \\
$\pi_r\in\Vc$ & token at rank $r$ of the BF16 logits $\zv$ at one position \\
$\operatorname{Top}_K(\zv)\subseteq\Vc$ & the $K$ tokens with the largest entries of $\zv$;
  $\operatorname{Top}_K(\zv)=\{\pi_1,\ldots,\pi_K\}$ \\
\midrule
\multicolumn{2}{@{}l@{}}{\emph{Models and passes}}\\
$\color{cleanred}{\Mc}$;~ $\color{quantgreen}{\widehat{\Mc}}$ &
  original (BF16) model; quantized (NVFP4) model 
  \\
$\quant{\Wv}$ & NVFP4 quantize--dequantize reconstruction of a weight matrix $\Wv$ \\
$\color{cleanred}{\Fv^{(\ell)}}$;~ $\color{quantgreen}{\compacthat{\Fv}{(\ell)}}$ &
  BF16/W4 block maps (Attention and MLP of block $\ell$) \\
\midrule
\multicolumn{2}{@{}l@{}}{\emph{Residual stream states and errors}}\\
$\hv^{(\ell)}$, $\compacthat{\hv}{(\ell)}$;~ $\uv^{(\ell)}$, $\compacthat{\uv}{(\ell)}$ &
  BF16/W4 hidden states and block updates;
  $\hv^{(\ell)}=\hv^{(\ell-1)}+\uv^{(\ell)}$ \\
$\Delta\hv^{(\ell)}$;~ $\Delta\uv^{(\ell)}$ &
  hidden error $\compacthat{\hv}{(\ell)}-\hv^{(\ell)}$ (the hidden error entering
  block $\ell{+}1$); block-update error
  $\compacthat{\uv}{(\ell)}-\uv^{(\ell)}$ \\
$R^{(\ell)}$ & relative hidden-error norm
  $\norm{\Delta\hv^{(\ell)}}_2/\norm{\hv^{(\ell)}}_2$ \\
$\termlabel{termcontrast}{T_{\rm add}}$,
$\termlabel{terminteraction}{T_{\rm inter}}$,
$\termlabel{termalignment}{T_{\rm align}}$ &
  the three terms of the recurrence in
  Thm.~\ref{thm:relative-hidden-error-recurrence}: relative added-error contribution,
  error interaction (named \emph{counteraction} when negative), and
  the residual-norm contribution from the clean input--update dot product \\
\midrule
\multicolumn{2}{@{}l@{}}{\emph{Output layer}}\\
$\hv_{\rm LM}$, $\widehat{\hv}_{\rm LM}$ & BF16/W4 LM-head inputs,
  $\hv_{\rm LM}=\operatorname{Norm}(\hv^{(L)})$ \\
$\Wv_{\mathrm{LM}}\in\Rb^{|\Vc|\times d}$;~ $\wv_k^\top$ &
  shared BF16 LM head; its row for token $k$ \\
$\zv$, $\widehat{\zv}$;~ $\pv$, $\widehat{\pv}$ &
  BF16/W4 logits $\zv=\Wv_{\mathrm{LM}}\hv_{\rm LM}$ and next-token
  distributions%
  \\
$\rho$ & LM-head input-norm ratio
  $\norm{\widehat{\hv}_{\rm LM}}_2/\norm{\hv_{\rm LM}}_2$ \\
$\alpha$ & hidden rotation
  $\angle(\hv_{\rm LM},\widehat{\hv}_{\rm LM})$ \\
$\theta_k$;~ $\widehat{\theta}_k$ & projection angles
  $\angle(\wv_k,\hv_{\rm LM})$ and
  $\angle(\wv_k,\widehat{\hv}_{\rm LM})$ \\
$g\in\Vc$;~ $\Delta\mathrm{CE}$;~ $\KL{\pv}{\widehat{\pv}}$ &
  ground-truth next token; CE increase
  $\mathrm{CE}_{\mathrm{W4}}-\mathrm{CE}_{\mathrm{BF16}}$; forward KL \\
Flip@1;~ Ret@$K$ & greedy flip rate
  $\Pr[\operatorname{Top}_1(\widehat{\zv})\neq\operatorname{Top}_1(\zv)]$ and
  mean top-$K$ retention
  $\Eb[\tfrac1K|\operatorname{Top}_K(\zv)\cap\operatorname{Top}_K(\widehat{\zv})|]$ \\
\bottomrule
\end{tabular}
\end{center}

\section{Proofs and Theoretical Analysis}
\label{app:proofs}

\subsection{Hidden-Error Recurrence Proofs}
\label{app:hidden-error-recurrence-proofs}

\begin{proof}[Proof of Prop.~\ref{prop:absolute-hidden-error-recurrence}]
Subtracting the original residual update from the quantized residual update in
Eq.~\eqref{eq:setup-paired-block} gives
$\Delta\hv^{(\ell)}=\Delta\hv^{(\ell-1)}+\Delta\uv^{(\ell)}$. Therefore,
\begin{align*}
\norm{\Delta\hv^{(\ell)}}_2^2
~=~\norm{\Delta\hv^{(\ell-1)}+\Delta\uv^{(\ell)}}_2^2
~=~\norm{\Delta\hv^{(\ell-1)}}_2^2
 +\norm{\Delta\uv^{(\ell)}}_2^2
 +2\iprod{\Delta\hv^{(\ell-1)},\Delta\uv^{(\ell)}}.
\end{align*}
Subtracting $\norm{\Delta\hv^{(\ell-1)}}_2^2$ from both sides proves
Eq.~\eqref{eq:absolute-hidden-error-recurrence}.
\end{proof}

\begin{proof}[Proof of Thm.~\ref{thm:relative-hidden-error-recurrence}]
By Prop.~\ref{prop:absolute-hidden-error-recurrence},
\begin{align*}
\bigl(R^{(\ell)}\bigr)^2
&=\frac{\norm{\Delta\hv^{(\ell-1)}}_2^2
        +\norm{\Delta\uv^{(\ell)}}_2^2
        +2\iprod{\Delta\hv^{(\ell-1)},\Delta\uv^{(\ell)}}}
       {\norm{\hv^{(\ell)}}_2^2}.
\end{align*}
The definition of $R^{(\ell-1)}$ gives
\begin{equation*}
\norm{\Delta\hv^{(\ell-1)}}_2^2
=\norm{\hv^{(\ell-1)}}_2^2\bigl(R^{(\ell-1)}\bigr)^2.
\end{equation*}
Substituting this identity and subtracting
$\bigl(R^{(\ell-1)}\bigr)^2$ yields
\begin{align*}
\bigl(R^{(\ell)}\bigr)^2-\bigl(R^{(\ell-1)}\bigr)^2
&=\frac{\norm{\Delta\uv^{(\ell)}}_2^2
        +2\iprod{\Delta\hv^{(\ell-1)},\Delta\uv^{(\ell)}}}
       {\norm{\hv^{(\ell)}}_2^2}
+
\left(
\frac{\norm{\hv^{(\ell-1)}}_2^2}{\norm{\hv^{(\ell)}}_2^2}-1
\right)
\bigl(R^{(\ell-1)}\bigr)^2.
\end{align*}

The original residual update
$\hv^{(\ell)}=\hv^{(\ell-1)}+\uv^{(\ell)}$ implies
\begin{equation*}
\norm{\hv^{(\ell)}}_2^2
=\norm{\hv^{(\ell-1)}}_2^2+\norm{\uv^{(\ell)}}_2^2
+2\iprod{\hv^{(\ell-1)},\uv^{(\ell)}}.
\end{equation*}
Therefore,
\begin{align*}
\frac{\norm{\hv^{(\ell-1)}}_2^2}{\norm{\hv^{(\ell)}}_2^2}-1
&=-\frac{\norm{\uv^{(\ell)}}_2^2
          +2\iprod{\hv^{(\ell-1)},\uv^{(\ell)}}}
         {\norm{\hv^{(\ell)}}_2^2}.
\end{align*}
Substitution gives
\begin{align*}
\bigl(R^{(\ell)}\bigr)^2-\bigl(R^{(\ell-1)}\bigr)^2
&=\frac{\norm{\Delta\uv^{(\ell)}}_2^2
        -\norm{\uv^{(\ell)}}_2^2
         \bigl(R^{(\ell-1)}\bigr)^2}
       {\norm{\hv^{(\ell)}}_2^2}
\nonumber\\
&\quad+
\frac{2\iprod{\Delta\hv^{(\ell-1)},\Delta\uv^{(\ell)}}}
     {\norm{\hv^{(\ell)}}_2^2}
\;-
\frac{2\iprod{\hv^{(\ell-1)},\uv^{(\ell)}}}
     {\norm{\hv^{(\ell)}}_2^2}
\bigl(R^{(\ell-1)}\bigr)^2.
\end{align*}
Finally, the first fraction factors as
\begin{align*}
\frac{\norm{\Delta\uv^{(\ell)}}_2^2
      -\norm{\uv^{(\ell)}}_2^2
       \bigl(R^{(\ell-1)}\bigr)^2}
     {\norm{\hv^{(\ell)}}_2^2}
&=\left(\frac{\norm{\uv^{(\ell)}}_2}
                  {\norm{\hv^{(\ell)}}_2}\right)^2
  \left[
  \left(\frac{\norm{\Delta\uv^{(\ell)}}_2}
             {\norm{\uv^{(\ell)}}_2}\right)^2
  -\bigl(R^{(\ell-1)}\bigr)^2
  \right].
\end{align*}
The three lines are $T_{\rm add}$, $T_{\rm inter}$, and $T_{\rm align}$,
respectively, which proves
Eq.~\eqref{eq:relative-hidden-error-recurrence}.
\end{proof}

\subsection{Length--Angle Decomposition}
\label{app:residual-direction-proofs}

\begin{proof}[Proof of Prop.~\ref{prop:hidden-error-magnitude-direction}]
By the definition of cosine,
$\iprod{\hv^{(\ell)},\compacthat{\hv}{(\ell)}}
=\norm{\hv^{(\ell)}}_2\norm{\compacthat{\hv}{(\ell)}}_2
\cos\angle\bigl(\hv^{(\ell)},\compacthat{\hv}{(\ell)}\bigr)$. Therefore,
\begin{align*}
\norm{\Delta\hv^{(\ell)}}_2^2
&=\norm{\widehat{\hv}^{(\ell)}}_2^2+\norm{\hv^{(\ell)}}_2^2
 -2\iprod{\hv^{(\ell)},\widehat{\hv}^{(\ell)}}\\
&=\left(\norm{\widehat{\hv}^{(\ell)}}_2
        -\norm{\hv^{(\ell)}}_2\right)^2
 +2\norm{\hv^{(\ell)}}_2\norm{\widehat{\hv}^{(\ell)}}_2
 \left[
 1-\cos\angle\bigl(\hv^{(\ell)},\widehat{\hv}^{(\ell)}\bigr)
 \right].
\end{align*}
Dividing by $\norm{\hv^{(\ell)}}_2^2$ proves
Eq.~\eqref{eq:hidden-error-magnitude-direction}. Substituting
$\norm{\compacthat{\hv}{(\ell)}}_2=\norm{\hv^{(\ell)}}_2$ gives the
equal-length identity stated in the proposition.
\end{proof}

\subsection{LM-Head Score Decomposition}
\label{app:lmhead-geometry-proofs}

Sec.~\ref{sec:lmhead-top-rank-geometry} studies how LM-head input rotation
changes token scores. The following exact decomposition separates score changes
caused by the input norm from those caused by the projection angle and clarifies
what the $\rho=1$ analytic references omit.

With $\rho$, $\theta_k$, and $\widehat{\theta}_k$ as in
Thm.~\ref{thm:lmhead-tangent-attenuation}, each vocabulary score change
splits exactly into a length part and an angle part:
\begin{equation}
\widehat z_k-z_k
=(\rho-1)z_k
+\rho\norm{\wv_k}_2\norm{\hv_{\rm LM}}_2
  (\cos\widehat\theta_k-\cos\theta_k).
\label{eq:lmhead-logit-decomposition}
\end{equation}

\begin{proof}[Proof of Eq.~\eqref{eq:lmhead-logit-decomposition}]
For each nonzero vocabulary row $\wv_k$,
\[
z_k=\norm{\wv_k}_2\norm{\hv_{\rm LM}}_2\cos\theta_k,
\qquad
\widehat z_k
=\norm{\wv_k}_2\norm{\widehat{\hv}_{\rm LM}}_2
 \cos\widehat{\theta}_k.
\]
Subtracting the two equalities and adding and subtracting
$\norm{\wv_k}_2\norm{\widehat{\hv}_{\rm LM}}_2\cos\theta_k$
gives Eq.~\eqref{eq:lmhead-logit-decomposition}. If
$\widehat{\hv}_{\rm LM}=c\hv_{\rm LM}$ with $c>0$, then
$\widehat{\zv}=c\zv$, so every pairwise score ordering is unchanged.
\end{proof}

\subsection{High-Dimensional Rotation Analysis}
\label{app:lmhead-tangent-rotation-proof}

\begin{proof}[Proof of Thm.~\ref{thm:lmhead-tangent-attenuation}]
We first characterize the geometry of the LM-head input rotation relative to
a fixed LM-head weight vector.

\noindent\textbf{Geometry of the hidden-state rotation.}
Let
\begin{equation*}
\bar{\hv}:=\frac{\hv_{\rm LM}}{\norm{\hv_{\rm LM}}_2},
\qquad
\bar{\wv}_k:=\frac{\wv_k}{\norm{\wv_k}_2},
\qquad
P_\perp:=\Iv-\bar{\hv}\bar{\hv}^{\mathsf T}.
\end{equation*}
Because $\theta_k\in(0,\pi)$, the component of $\bar{\wv}_k$ orthogonal to
$\bar{\hv}$ is nonzero. Define its unit direction by
\begin{equation*}
\boldsymbol{\xi}_k:=
\frac{P_\perp\wv_k}{\norm{P_\perp\wv_k}_2}.
\end{equation*}
Then $\bar{\wv}_k$ decomposes into components parallel and orthogonal to
$\bar{\hv}$:
\begin{equation*}
\bar{\wv}_k
=
\cos\theta_k\,\bar{\hv}
+
\sin\theta_k\,\boldsymbol{\xi}_k.
\end{equation*}

Similarly, because
$\alpha=\angle(\hv_{\rm LM},\widehat{\hv}_{\rm LM})\in(0,\pi)$,
the normalized quantized LM-head input can be written as
\begin{equation*}
\frac{\widehat{\hv}_{\rm LM}}
     {\norm{\widehat{\hv}_{\rm LM}}_2}
=
\cos\alpha\,\bar{\hv}
+
\sin\alpha\,\boldsymbol{\eta},
\end{equation*}
where
\begin{equation*}
\boldsymbol{\eta}:=
\frac{P_\perp\widehat{\hv}_{\rm LM}}
     {\norm{P_\perp\widehat{\hv}_{\rm LM}}_2}
\end{equation*}
is a unit vector orthogonal to $\bar{\hv}$.
Define the directional coordinate
\begin{equation*}
q_k:=\iprod{\boldsymbol{\eta},\boldsymbol{\xi}_k}.
\end{equation*}
Thus, $q_k$ measures how much the hidden-state rotation points toward the
component of $\wv_k$ orthogonal to the original hidden state.
Taking the inner product between the decompositions of
$\bar{\wv}_k$ and
$\widehat{\hv}_{\rm LM}/\norm{\widehat{\hv}_{\rm LM}}_2$
gives the exact identity
\begin{equation}
\cos\widehat\theta_k
=
\cos\theta_k\cos\alpha
+
\sin\theta_k\sin\alpha\,q_k.
\label{eq:lmhead-tangent-cosine}
\end{equation}

\textbf{Isotropic rotation direction and the distribution of $q_k$.}
Conditional on the measured rotation angle $\alpha$, the theorem assumes that
the rotation direction $\boldsymbol{\eta}$ is uniformly distributed over all
unit directions orthogonal to $\hv_{\rm LM}$. Equivalently,
$\boldsymbol{\eta}$ is uniform on the unit sphere
$\mathbb{S}^{d-2}$ in the $(d-1)$-dimensional subspace
$\hv_{\rm LM}^{\perp}$.

A convenient way to represent such a uniform direction is to normalize an
isotropic Gaussian vector. Choose an orthonormal basis of
$\hv_{\rm LM}^{\perp}$ whose first basis vector is
$\boldsymbol{\xi}_k$, and let
\begin{equation*}
\gv=(G_1,\ldots,G_{d-1})^{\mathsf T},
\qquad
G_j\stackrel{\mathrm{i.i.d.}}{\sim}\mathcal{N}(0,1).
\end{equation*}
Because the standard Gaussian distribution is rotationally invariant,
$\gv/\norm{\gv}_2$ is uniformly distributed on
$\mathbb{S}^{d-2}$. Hence we may represent
$\boldsymbol{\eta}$ in this basis as
\begin{equation*}
\boldsymbol{\eta}
\stackrel{d}{=}
\frac{\gv}{\norm{\gv}_2}.
\end{equation*}
where $\stackrel{d}{=}$ denotes equality in distribution.
This does not mean that a uniform spherical direction is Gaussian:
the Gaussian vector $\gv$ has a random norm, while
$\gv/\norm{\gv}_2$ has unit norm and only its direction is retained.

Since the first basis vector is $\boldsymbol{\xi}_k$, the directional
coordinate
\begin{equation*}
q_k:=\iprod{\boldsymbol{\eta},\boldsymbol{\xi}_k}
\end{equation*}
is distributed as the first coordinate of this normalized Gaussian vector:
\begin{equation*}
q_k
\stackrel{d}{=}
\frac{G_1}
{\sqrt{G_1^2+\cdots+G_{d-1}^2}}.
\end{equation*}
Therefore,
\begin{equation*}
q_k^2
\stackrel{d}{=}
\frac{G_1^2}
{G_1^2+\sum_{j=2}^{d-1}G_j^2}.
\end{equation*}
Now
\begin{equation*}
G_1^2\sim\chi_1^2,
\qquad
\sum_{j=2}^{d-1}G_j^2\sim\chi_{d-2}^2,
\end{equation*}
and these two random variables are independent.

For independent $X\sim\chi_m^2$ and $Y\sim\chi_n^2$,
$X/(X+Y)\sim\operatorname{Beta}(m/2,n/2)$.
Applying this with $m=1$ and $n=d-2$ gives
\begin{equation*}
q_k^2
\sim
\operatorname{Beta}\!\left(
\frac12,\frac{d-2}{2}
\right).
\end{equation*}
Moreover, $q_k$ is symmetric about zero, so $\Eb[q_k]=0$.
Equivalently, $q_k$ has density
\begin{equation*}
f_d(q)
=
\frac{\Gamma((d-1)/2)}
{\sqrt{\pi}\,\Gamma((d-2)/2)}
(1-q^2)^{(d-4)/2},
\qquad -1<q<1.
\end{equation*}
Its mean absolute value is
\begin{align*}
\Eb|q_k|
&=
2
\frac{\Gamma((d-1)/2)}
{\sqrt{\pi}\,\Gamma((d-2)/2)}
\int_0^1 q(1-q^2)^{(d-4)/2}\,dq
\\
&=
\frac{\Gamma((d-1)/2)}
{\sqrt{\pi}\,\Gamma(d/2)}
=: \mu_d.
\end{align*}
The standard gamma-ratio expansion gives
\begin{equation*}
\mu_d
=
\sqrt{\frac{2}{\pi(d-1)}}
\left(1+O(d^{-1})\right).
\end{equation*}
Thus, $q_k$ is the coordinate of a random unit direction along one fixed
direction in a $(d-1)$-dimensional space, and its typical magnitude is
$O(d^{-1/2})$. This is the source of the high-dimensional attenuation in
Eq.~\eqref{eq:lmhead-tangent-angle-rate}.

\textbf{Projection-angle change.}
For fixed $q_k$, define
\begin{equation*}
g_k(\alpha):=
\cos\theta_k\cos\alpha
+
\sin\theta_k\sin\alpha\,q_k,
\qquad
\widehat\theta_k=\arccos g_k(\alpha).
\end{equation*}
At $\alpha=0$,
\begin{equation*}
g_k(0)=\cos\theta_k,
\qquad
g_k'(0)=\sin\theta_k\,q_k.
\end{equation*}
Therefore,
\begin{align*}
\left.
\frac{d\widehat\theta_k}{d\alpha}
\right|_{\alpha=0}
=
-\frac{g_k'(0)}
       {\sqrt{1-g_k(0)^2}}
=
-\frac{\sin\theta_k\,q_k}
       {\sqrt{1-\cos^2\theta_k}}
=
-q_k,
\end{align*}
where the last equality uses $\theta_k\in(0,\pi)$ and hence
$\sin\theta_k>0$.

For fixed $\theta_k\in(0,\pi)$, the denominator above remains bounded away
from zero for sufficiently small $\alpha$, uniformly over
$q_k\in[-1,1]$. Thus the Taylor remainder can be taken uniformly in $q_k$,
and
\begin{equation*}
\widehat\theta_k-\theta_k
=
-\alpha q_k+O(\alpha^2).
\end{equation*}
Using
\begin{equation*}
\left|
\,|\widehat\theta_k-\theta_k|
-
\alpha|q_k|
\right|
=
O(\alpha^2)
\end{equation*}
uniformly in $q_k$, taking expectation gives
\begin{equation*}
\Eb\left|
\widehat\theta_k-\theta_k
\right|
=
\alpha\mu_d+O(\alpha^2).
\end{equation*}
Together with the expression for $\mu_d$ above, this proves
Eq.~\eqref{eq:lmhead-tangent-angle-rate}.

\textbf{Relative score change.}
The original and quantized scores are
\begin{equation*}
z_k
=
\norm{\wv_k}_2
\norm{\hv_{\rm LM}}_2
\cos\theta_k,
\qquad
\widehat z_k
=
\norm{\wv_k}_2
\norm{\widehat{\hv}_{\rm LM}}_2
\cos\widehat\theta_k.
\end{equation*}
Using
\begin{equation*}
\rho
:=
\frac{\norm{\widehat{\hv}_{\rm LM}}_2}
     {\norm{\hv_{\rm LM}}_2},
\end{equation*}
and substituting Eq.~\eqref{eq:lmhead-tangent-cosine},
\begin{align*}
\frac{\widehat z_k}{z_k}
&=
\rho
\frac{\cos\widehat\theta_k}
     {\cos\theta_k}
\\
&=
\rho\cos\alpha
+
\rho\tan\theta_k\sin\alpha\,q_k.
\end{align*}
Subtracting one gives the exact decomposition
\begin{equation*}
\frac{\widehat z_k-z_k}{z_k}
=
(\rho\cos\alpha-1)
+
\rho\tan\theta_k\sin\alpha\,q_k,
\end{equation*}
which proves
Eq.~\eqref{eq:lmhead-relative-score-distortion}.

When $\rho=1$,
\begin{equation*}
\frac{\widehat z_k-z_k}{z_k}
=
(\cos\alpha-1)
+
\tan\theta_k\sin\alpha\,q_k.
\end{equation*}
Since
\begin{equation*}
\cos\alpha-1=O(\alpha^2),
\qquad
\sin\alpha=\alpha+O(\alpha^3),
\end{equation*}
we obtain, uniformly over $q_k\in[-1,1]$,
\begin{equation*}
\frac{\widehat z_k-z_k}{z_k}
=
\alpha\tan\theta_k\,q_k+O(\alpha^2).
\end{equation*}
Taking the expected absolute value therefore gives
\begin{equation*}
\Eb\left|
\frac{\widehat z_k-z_k}{z_k}
\right|
=
\alpha
\lvert\tan\theta_k\rvert
\mu_d
+
O(\alpha^2),
\end{equation*}
which proves
Eq.~\eqref{eq:lmhead-relative-score-rate}.

\textbf{Finite-angle expectation.}
The first-order result above keeps only the leading term in $\alpha$.
For $\rho=1$, we can also compute the expected absolute relative score
change exactly at a finite rotation angle.

Let $q$ denote a random variable with density $f_d$ above, and for
$A,B\ge0$ define
\begin{equation}
\Psi_d(A,B)
:=
\Eb_q\!\left[\lvert -A+Bq \rvert\right].
\label{eq:tangent-shifted-absolute-moment-definition}
\end{equation}
Thus, $\Psi_d(A,B)$ is simply the expected absolute value of a random
directional term $Bq$ shifted by the deterministic quantity $A$.

Let
\begin{equation*}
I_x(a,b)
:=
\frac{B_x(a,b)}{B(a,b)},
\qquad
B_x(a,b)
:=
\int_0^x
t^{a-1}(1-t)^{b-1}\,dt,
\end{equation*}
denote the regularized incomplete beta function, with
$B(a,b):=B_1(a,b)$.

We now derive a closed form for $\Psi_d$.
If $B=0$ or $A\ge B$, then for all $q\in[-1,1]$,
\begin{equation*}
-A+Bq\le -A+B\le0.
\end{equation*}
Hence
\begin{equation*}
\Psi_d(A,B)
=
\Eb_q[A-Bq]
=
A,
\end{equation*}
because $\Eb_q[q]=0$.

Now suppose $0\le A<B$, and define
\begin{equation*}
t:=\frac{A}{B}\in[0,1).
\end{equation*}
The quantity $-A+Bq$ changes sign at $q=t$. Splitting the expectation at
this point and using the symmetry $f_d(q)=f_d(-q)$ gives
\begin{align*}
\Psi_d(A,B)
&=
\Eb_q|-A+Bq|
\\
&=
A\Pr(|q|\le t)
+
2B\int_t^1 qf_d(q)\,dq.
\end{align*}

Because
\begin{equation*}
q^2
\sim
\operatorname{Beta}\!\left(
\frac12,\frac{d-2}{2}
\right),
\end{equation*}
the first term is
\begin{equation*}
\Pr(|q|\le t)
=
\Pr(q^2\le t^2)
=
I_{t^2}\!\left(
\frac12,\frac{d-2}{2}
\right).
\end{equation*}
For the second term,
\begin{align*}
2\int_t^1 qf_d(q)\,dq
&=
2
\frac{\Gamma((d-1)/2)}
     {\sqrt{\pi}\,\Gamma((d-2)/2)}
\int_t^1
q(1-q^2)^{(d-4)/2}\,dq
=
\mu_d
(1-t^2)^{(d-2)/2}.
\end{align*}
Substituting $t=A/B$ gives
\begin{equation}
\Psi_d(A,B)
=
\begin{cases}
A,
&
B=0\ \text{or}\ A\ge B,
\\[1ex]
A
I_{(A/B)^2}\!\left(
\frac12,\frac{d-2}{2}
\right)
+
B\mu_d
\left[
1-(A/B)^2
\right]^{(d-2)/2},
&
0\le A<B.
\end{cases}
\label{eq:tangent-shifted-absolute-moment}
\end{equation}

Returning to the relative score error with $\rho=1$,
\begin{equation*}
\frac{\widehat z_k-z_k}{z_k}
=
-(1-\cos\alpha)
+
\tan\theta_k\sin\alpha\,q_k.
\end{equation*}
Since $q_k$ is symmetric about zero, the sign of the coefficient multiplying
$q_k$ does not affect the expected absolute value. Therefore, setting
\begin{equation*}
A:=1-\cos\alpha,
\qquad
B:=
\left|
\tan\theta_k\sin\alpha
\right|
\end{equation*}
gives the exact finite-angle expression
\begin{equation*}
\Eb\left|
\frac{\widehat z_k-z_k}{z_k}
\right|
=
\Psi_d\!\left(
1-\cos\alpha,
\left|
\tan\theta_k\sin\alpha
\right|
\right).
\end{equation*}

Finally,
\begin{equation*}
1-\cos\alpha
=
\frac{\alpha^2}{2}
+
O(\alpha^4),
\qquad
\sin\alpha
=
\alpha
+
O(\alpha^3).
\end{equation*}
Moreover, since $|q|\le1$,
\begin{equation*}
\left|
\Psi_d(A,B)-\Psi_d(A',B')
\right|
\le
|A-A'|+|B-B'|.
\end{equation*}
Thus replacing the exact arguments by their second-order approximations
changes $\Psi_d$ by at most $O(\alpha^3)$, yielding
\begin{equation}
\Eb\left|
\frac{\widehat z_k-z_k}{z_k}
\right|
=
\Psi_d\!\left(
\frac{\alpha^2}{2},
\alpha\lvert\tan\theta_k\rvert
\right)
+
O(\alpha^3).
\label{eq:lmhead-relative-score-second-order}
\end{equation}
\end{proof}

\newpage
\subsection{Log-Probability Approximation Proof}
\label{app:log-probability-approximation-proof}

\begin{proof}[Proof of Thm.~\ref{thm:log-probability-approximations}]
Let $Z:=\sum_j e^{z_j}$ and
$\widehat Z:=\sum_j e^{z_j+\Delta z_j}$. Since $p_j=e^{z_j}/Z$,
\begin{equation*}
\widehat Z
=Z\sum_j p_j e^{\Delta z_j}
=Z\Eb_{j\sim\pv}e^{\Delta z_j}.
\end{equation*}
It follows that
\begin{equation*}
\widehat p_k
=\frac{p_k e^{\Delta z_k}}
       {\Eb_{j\sim\pv}e^{\Delta z_j}},
\qquad
\Delta\log p_k
=\Delta z_k-\log\Eb_{j\sim\pv}e^{\Delta z_j}.
\end{equation*}
Also,
\begin{align*}
\KL{\pv}{\widehat{\pv}}
&=\Eb_{k\sim\pv}\log\frac{p_k}{\widehat p_k}
=\log\Eb_{j\sim\pv}e^{\Delta z_j}
-\Eb_{j\sim\pv}\Delta z_j.
\end{align*}
Combining these two identities proves the exact line of
Eq.~\eqref{eq:log-probability-approximations}.

Set
\begin{equation*}
m:=\Eb_{j\sim\pv}\Delta z_j,
\qquad
x_j:=\Delta z_j-m,
\qquad
\Eb_{j\sim\pv}x_j=0.
\end{equation*}
Then
\begin{equation*}
\log\Eb_{j\sim\pv}e^{\Delta z_j}
=m+\log\Eb_{j\sim\pv}e^{x_j}.
\end{equation*}
For uniformly small $x_j$, Taylor expansion gives
\begin{align*}
\Eb_{j\sim\pv}e^{x_j}
&=1+\frac12\Eb_{j\sim\pv}x_j^2
+O\!\left(\Eb_{j\sim\pv}|x_j|^3\right),\\
\log\Eb_{j\sim\pv}e^{x_j}
&=\frac12\Eb_{j\sim\pv}x_j^2
+O\!\left(\Eb_{j\sim\pv}|x_j|^3\right)\\
&=\frac12\operatorname{Var}_{j\sim\pv}(\Delta z_j)
+O\!\left(\Eb_{j\sim\pv}|x_j|^3\right).
\end{align*}
Substituting into the exact identity yields
\begin{equation*}
\Delta\log p_k
=\Delta z_k-m
-\frac12\operatorname{Var}_{j\sim\pv}(\Delta z_j)
+O\!\left(\Eb_{j\sim\pv}|x_j|^3\right),
\end{equation*}
which is the second-order estimate. Omitting the quadratic term gives
\begin{equation*}
\Delta\log p_k
=\Delta z_k-m
+O\!\left(\operatorname{Var}_{j\sim\pv}(\Delta z_j)\right),
\end{equation*}
which is the first-order estimate.
\end{proof}

\subsection{Softmax and Token-Ranking Stability}
\label{app:output-probability-decision-proof}

This appendix gives the exact conditions under which a token's ranking is
preserved or changed, complementing the estimates by rank in
Sec.~\ref{sec:lmhead-top-rank-geometry}.

\begin{theorem}[Probability-change identities and token-preservation conditions]
\label{thm:output-probability-decision}
Let $\Delta\zv:=\widehat{\zv}-\zv$,
$\pv:=\softmax(\zv)$, and
$\widehat{\pv}:=\softmax(\widehat{\zv})$, with $|\Vc|\ge2$. Then
\begin{equation}
\KL{\pv}{\widehat{\pv}}
=\log\mathbb{E}_{k\sim\pv}e^{\Delta z_k}
 -\mathbb{E}_{k\sim\pv}\Delta z_k.
\label{eq:output-kl-logit-shift}
\end{equation}
For a ground-truth next token $g$,
\begin{equation}
-\log\widehat p_g+\log p_g
=\log\mathbb{E}_{k\sim\pv}e^{\Delta z_k}-\Delta z_g.
\label{eq:output-cross-entropy-change}
\end{equation}
Adding the same scalar to every $\Delta z_k$ leaves these quantities unchanged.
If $a$ is the unique top-1 token under $\zv$, it remains top-1 exactly when
\begin{equation}
z_a-z_j>\Delta z_j-\Delta z_a
\qquad\text{for every }j\ne a.
\label{eq:output-top1-margin-condition}
\end{equation}
The analogous condition for a unique top-$K$ set $S_K$ is
\begin{equation}
z_a-z_j>\Delta z_j-\Delta z_a
\qquad\text{for every }a\in S_K,\ j\notin S_K.
\label{eq:output-topk-margin-condition}
\end{equation}
\end{theorem}

\begin{proof}[Proof of Thm.~\ref{thm:output-probability-decision}]
Let
$Z:=\sum_j e^{z_j}$ and
$\widehat Z:=\sum_j e^{z_j+\Delta z_j}$.
Then
\[
\widehat Z
=Z\,\mathbb{E}_{j\sim\pv}e^{\Delta z_j},
\qquad
\widehat p_k
=\frac{p_k e^{\Delta z_k}}
       {\mathbb{E}_{j\sim\pv}e^{\Delta z_j}}.
\]
Therefore,
\[
\log\frac{p_k}{\widehat p_k}
=\log\mathbb{E}_{j\sim\pv}e^{\Delta z_j}-\Delta z_k.
\]
Taking expectation over $k\sim\pv$ proves
Eq.~\eqref{eq:output-kl-logit-shift}; setting $k=g$ proves
Eq.~\eqref{eq:output-cross-entropy-change}. Replacing every
$\Delta z_k$ by $\Delta z_k+c$ adds $c$ to both terms on the right-hand
sides of Eqs.~\eqref{eq:output-kl-logit-shift} and
\eqref{eq:output-cross-entropy-change}, so the additions cancel.

The original top-1 token $a$ remains the unique perturbed top-1 token exactly
when
\[
z_a+\Delta z_a>z_j+\Delta z_j
\qquad\text{for every }j\ne a.
\]
Rearranging gives Eq.~\eqref{eq:output-top1-margin-condition}. Likewise,
$S_K$ remains exactly the top-$K$ set if and only if every member stays above
every nonmember:
\[
z_a+\Delta z_a>z_j+\Delta z_j
\qquad\text{for every }a\in S_K,\ j\notin S_K.
\]
Rearranging gives Eq.~\eqref{eq:output-topk-margin-condition}.
\end{proof}

For Qwen3-32B, Flip@1 is $10.7\%$ on C4, $12.7\%$ on WikiText-103,
and $8.3\%$ on GSM8K text under NVFP4
(Fig.~\ref{fig:sec4-output-geometry}(A)). Under the theorem's assumption that
the original top-1 token $a$ is unique, each flip implies
$\Delta z_j-\Delta z_a\ge z_a-z_j$ for at least one competitor $j\ne a$,
where $z_a-z_j>0$ is its original score margin to the top-1 token. Thus,
Flip@1 measures how often quantization changes relative token
scores enough to overturn the original top-1 prediction, rather than how often
any score changes.

\newpage
\section{Experimental Details}

\label{app:nvfp4-format}
\subsection{NVFP4 Quantization Format}

NVFP4 represents each quantized value by a 4-bit E2M1 floating-point code and
two levels of shared scales~\citep{nvidia2025nvfp4}. The signed E2M1 value set
is
\begin{equation*}
    \mathcal{Q}_{\mathrm{E2M1}}
    = \{0, \pm 0.5, \pm 1, \pm 1.5, \pm 2, \pm 3, \pm 4, \pm 6\}.
    \end{equation*}
Thus, the largest E2M1 magnitude is $q_{\max}=6$. Each weight matrix
$\Wv\in\mathbb{R}^{d_{\mathrm{out}}\times d_{\mathrm{in}}}$ is partitioned
row by row into consecutive groups of $16$ entries along its input dimension.
If necessary, the final group is zero-padded for quantization and the padding
is removed afterward. We denote these groups by $\{\mathcal{G}_b\}_b$, and each group $\Gc_b$ contains $16$ entries with shape $1\times 16$.

The quantization of one weight matrix proceeds as follows. First, we compute a tensor-level scale
\begin{equation}
    s_{\Wv}
    = \frac{\max_{i,j}|W_{ij}|}{q_{\max} f_{\max}}
    = \frac{\max_{i,j}|W_{ij}|}{6\times 448},
    \label{eq:appendix-nvfp4-tensor-scale}
\end{equation}
where $f_{\max}=448$ is the largest finite magnitude of the E4M3 scale format.
For each 16-entry group $\mathcal{G}_b$, we then compute the ideal local scale
from the tensor-normalized weights and store it in E4M3:
\begin{align*}
    \widetilde{s}_b
    = \frac{1}{q_{\max}}
       \max_{(i,j)\in\mathcal{G}_b}
       \left|\frac{W_{ij}}{s_{\Wv}}\right|, \quad
    s_b
    = \operatorname{cast}_{\mathrm{E4M3}}\!\left(\widetilde{s}_b\right).
    \end{align*}
The factor $6\times448$ in Eq.~\eqref{eq:appendix-nvfp4-tensor-scale}
ensures that the required local scales lie within the E4M3 range. Finally,
each weight is divided by both scales, rounded to its nearest E2M1 value, and
dequantized:
\begin{align}
    q_{ij}
    = \operatorname{RTN}_{\mathcal{Q}_{\mathrm{E2M1}}}
       \!\left(\frac{W_{ij}}{s_{\Wv}s_b}\right),
    \quad
    \quant{\Wv}_{ij}
    = s_{\Wv}s_b q_{ij}.
       \qquad (i,j)\in\mathcal{G}_b.
    \label{eq:appendix-nvfp4-qdq}
\end{align}
Here, round-to-nearest (RTN) selects the closest E2M1 value, with values outside
its range clipped to $[-6,6]$. An all-zero tensor or group is mapped to zero,
avoiding division by a zero scale.

We apply RTN NVFP4 to the linear weights in every Transformer block, while embeddings, normalization layers, activations, and the LM head
remain in their original precision. For MoE layers, this includes every expert's gate, up, and down projections but not the router. We simulate quantization by casting the reconstructed weights
$\quant{\Wv}$ in Eq.~\eqref{eq:appendix-nvfp4-qdq} back to the computation dtype
before each linear operation. The experiments therefore measure the numerical
effect of quantization, not the runtime or memory cost of a packed 4-bit kernel.

\subsection{NVFP4 Weight Reconstruction Error}
\label{app:nvfp4-weight-closeness}

We compare the NVFP4 reconstruction errors of pretrained and randomly
initialized models directly. The Frobenius cosine-similarity between a weight matrix and its reconstruction is defined as
$\frac{\iprod{\Wv,\quant{\Wv}}_F}
{\norm{\Wv}_F\norm{\quant{\Wv}}_F}.
$ Across all models, pretrained and randomly initialized weights both have reconstruction cosines of about $0.9955$ and relative errors of about
$9.5\%$. These \emph{\textbf{nearly identical matrix-level errors}} contrast with their
different hidden-error growth in Sec.~\ref{sec:residual-observation}. Thus,
weight reconstruction alone does not explain the slower growth after pretraining.
\begin{table*}[h]
\centering
\caption{\textbf{Per-matrix reconstruction error under direct-cast NVFP4.}
Values are means $\pm$ SD. Relative error is
$\norm{\quant{\Wv}-\Wv}_{F}/\norm{\Wv}_{F}$. Random-init columns use one independently initialized instance per model; all reported SDs are computed across weight matrices.}
\label{tab:app-nvfp4-weight-closeness}
\scriptsize
\setlength{\tabcolsep}{5pt}
\begin{tabular}{lcccc}
\toprule
& \multicolumn{2}{c}{Pretrained} & \multicolumn{2}{c}{Random initialization} \\
\cmidrule(lr){2-3}\cmidrule(lr){4-5}
Model & $\cos(\Wv,\quant{\Wv})$ & Relative error
& $\cos(\Wv,\quant{\Wv})$ & Relative error \\
\midrule
Qwen3-4B    & $0.99550 \pm 0.00001$ & $9.49\% \pm 0.02\%$ & $0.99548 \pm 0.00001$ & $9.51\% \pm 0.01\%$ \\
Qwen3-8B    & $0.99551 \pm 0.00003$ & $9.47\% \pm 0.03\%$ & $0.99548 \pm 0.00001$ & $9.51\% \pm 0.01\%$ \\
Qwen3-14B   & $0.99551 \pm 0.00004$ & $9.47\% \pm 0.04\%$ & $0.99548 \pm 0.00001$ & $9.51\% \pm 0.01\%$ \\
Qwen3-32B   & $0.99551 \pm 0.00003$ & $9.48\% \pm 0.04\%$ & $0.99547 \pm 0.00001$ & $9.51\% \pm 0.01\%$ \\
Pythia-1.4B & $0.99549 \pm 0.00002$ & $9.50\% \pm 0.02\%$ & $0.99548 \pm 0.00001$ & $9.51\% \pm 0.01\%$ \\
OLMo3-7B    & $0.99549 \pm 0.00002$ & $9.50\% \pm 0.02\%$ & $0.99548 \pm 0.00001$ & $9.51\% \pm 0.01\%$ \\
\bottomrule
\end{tabular}
\end{table*}

\newpage
\subsection{Reproduction Details}
\label{app:main-figure-protocols}

\paragraph{Datasets.}
The token-level experiments use C4~\citep{raffel2020t5},
WikiText-103~\citep{merity2017wikitext}, and GSM8K
text~\citep{cobbe2021gsm8k}. Unless noted otherwise, each dataset setting
contains 64 randomly sampled sequences with 512 next-token prediction
positions per sequence.
The cross-model recurrence accounting in
Figs.~\ref{fig:app-sec3-counteraction-generalization-cumulative}
and~\ref{fig:app-sec3-counteraction-generalization-per-block} instead uses 128
sequences per model and dataset. The PTQ comparison uses eight C4 evaluation
inputs, while the temperature replay uses eight inputs from each of C4 and
GSM8K text.
For GSM8K text, we join each official test-set question with its reference
solution and use teacher forcing: the model predicts each reference token from
the preceding reference text. It does not generate a complete solution, so
this setting does not measure generated-solution accuracy.

\paragraph{Benchmarks.}
We use the full test splits of ARC-Challenge and
ARC-Easy~\citep{clark2018arc} and MMLU~\citep{hendrycks2021mmlu}, and the full
validation splits of HellaSwag~\citep{zellers2019hellaswag},
WinoGrande~\citep{winograd2021}, and
TruthfulQA~\citep{lin2022truthfulqa}.
All six benchmarks are evaluated zero-shot without a chat template. ARC and
HellaSwag use character-normalized candidate likelihoods, MMLU scores answer
letters, WinoGrande scores the shared suffix conditioned on each candidate,
and TruthfulQA reports MC1 accuracy. These results appear in
Fig.~\ref{fig:sec4-output-geometry}.

\paragraph{Models.}
We evaluate Qwen3-4B, Qwen3-8B, Qwen3-14B, Qwen3-30B-A3B, and
Qwen3-32B~\citep{yang2025qwen3}; OLMo3-7B and
OLMo3-32B~\citep{teamolmo2025olmo3};
OLMoE-1B-7B~\citep{muennighoff2025olmoe};
Gemma3-4B~\citep{gemmateam2025gemma3technicalreport}; and Pythia-1.4B and
Pythia-2.8B~\citep{biderman2023pythia}. Qwen3-32B is the primary model for the
hidden-state and output analyses. Fig.~\ref{fig:hidden-error-growth} compares
its pretrained checkpoint with ten independent random initializations and
also compares six released OLMo3-7B Stage-1 checkpoints. The random reference
in Fig.~\ref{fig:sec4-topk-protection} uses three fixed, seeded random
initializations of Qwen3-32B.

\paragraph{Post-training weight quantization.}
Unless stated otherwise, we use the weight-only NVFP4 RTN conversion defined
in App.~\ref{app:nvfp4-format}. Attention and MLP projection weights are
quantized, while embeddings, normalization layers, activations, and the LM
head remain in their original precision. The counteraction intervention
protocol is given in App.~\ref{app:counteraction-intervention}. For the
calibrated PTQ comparison, GPTQ and AWQ use the same 64 C4 calibration inputs,
disjoint from the eight evaluation inputs; RTN uses no calibration data.

\paragraph{Activation and activation-weight joint quantization.}
For A4, the weights remain in their original precision, while the input to
each decoder Attention and MLP projection is quantized immediately before
matrix multiplication. We use the same E2M1 values and E4M3 block scales as
for W4. Each activation tensor is partitioned into contiguous groups of 16
values along its last dimension.
The tensor-wide scale is recomputed for every
forward pass, followed by one E4M3 scale per group, making activation
quantization dynamic and calibration-free.
Residual states, normalization
layers, embeddings, the KV cache, and the LM head remain in their original
precision.
W4A4 combines this activation quantization with the W4 conversion defined in App.~\ref{app:nvfp4-format}.

\paragraph{Evaluation.}
Unless specified below, statistics are first averaged over positions within
each input and then equally across inputs. Error bars are one sample standard
deviation across complete inputs. For random models, each combination of an
initialization and input is one sample.

For the relative-error recurrence analysis in
Fig.~\ref{fig:sec3-counteraction-accounting}, we form each term at every
position before aggregation. Position zero is the only prediction conditioned
on a one-token prefix, so we exclude this fixed-window boundary and report the
recurrence over within-sequence positions. Rare large differences between the
BF16 and W4 hidden-state norms produce extreme terms that can dominate the
aggregate, so we retain a complete position trajectory only when its maximum
relative hidden-norm difference across depth satisfies
\begin{equation*}
g_i^{\max}
:=\max_{\ell}
\frac{\left|\norm{\widehat{\hv}_i^{(\ell)}}_2
-\norm{\hv_i^{(\ell)}}_2\right|}
{\norm{\hv_i^{(\ell)}}_2}.
\end{equation*}
We retain position $i$ if $g_i^{\max}\leq0.5$ and use the same retained
positions at every block. This removes $0.08\%$ of Qwen3-32B C4 positions and
retains $99.20$--$99.97\%$ across the 15 cross-model and cross-dataset
settings. Each plotted standard deviation is computed separately for one
recurrence term and therefore need not satisfy the additive recurrence.
App.~\ref{app:cross-model-relative-recurrence} reports the
cross-model results, and
App.~\ref{app:relative-recurrence-position-selection} reports the unfiltered
and threshold-sensitivity checks.

Fig.~\ref{fig:sec3-counteraction-effects}(B) forms the terms in
Prop.~\ref{prop:hidden-error-magnitude-direction} at each prediction position
before averaging. The LM-head-input region applies the same decomposition
after final normalization and uses the same retained positions.

Fig.~\ref{fig:sec4-output-geometry}(A) evaluates all 512 positions in each
input. Its benchmark error bars come from 10,000 bootstrap estimates of
accuracy.

Fig.~\ref{fig:sec4-topk-protection}(A) uses the same complete inputs and
reports the mean angle between the BF16 and W4 LM-head inputs together with
the vocabulary-mean absolute change in their projection angles to the same
LM-head weight vectors.
For Fig.~\ref{fig:sec4-topk-protection}(B) and
Fig.~\ref{fig:sec4-topk-protection}(C), $\pi_r$ is
the token ranked $r$ by the BF16 scores.
Fig.~\ref{fig:sec4-topk-protection}(B) retains ranks 1--10 individually and
groups the rest of the vocabulary into contiguous log-spaced
intervals. It includes a position--rank pair only when its BF16 score and
FP32-recomputed projection $\iprod{\wv_{\pi_r},\hv_{\rm LM}}$ are both positive.
The observed absolute relative score changes are averaged first within each
complete input and then equally across inputs. The pretrained band shows one
log-SD across the input-level means. For the random reference, each
initialization--input pair is one sample, and the curve ends when a rank bin
contains no positive BF16 scores.

The dotted first-order and dashed second-order curves evaluate
Eqs.~\eqref{eq:lmhead-relative-score-rate}
and~\eqref{eq:lmhead-relative-score-second-order} at each retained pair's
observed hidden rotation $\alpha$ and projection angle $\theta_{\pi_r}$ under
$\rho=1$, then apply the same averaging. These analytic references use the
conditional expectation over the rotation direction $\boldsymbol\eta$
(App.~\ref{app:lmhead-tangent-rotation-proof}); neither is fitted to the W4
measurements.

Fig.~\ref{fig:sec4-topk-protection}(C) retains all vocabulary rows. Its three
curves average the absolute values of the exact,
first-order, and second-order expressions in
Thm.~\ref{thm:log-probability-approximations}, first within each complete input
and then equally across the three datasets. The random reference uses the same
aggregation over initialization--input pairs. Bars show one log-SD across the
corresponding input-level means.

\newpage
\section{Additional Experimental Analysis}
\label{app:scope-extensions}

\makeatletter
\setlength{\@fptop}{0pt}
\setlength{\@dblfptop}{0pt}
\makeatother

\subsection{Hidden-Error Growth at Initialization and after Pretraining}
\label{app:random-hidden-error-growth-rate}

Sec.~\ref{sec:residual-observation} describes the growth of
absolute and relative hidden error at random initialization and contrasts it
with pretrained models. This subsection derives the random-initialization fit
from the recurrence, checks it across architectures, and quantifies the slower
growth after pretraining.

\paragraph{Hidden error in randomly initialized models} We derive the randomly initialized Qwen3-32B growth rate from
Fig.~\ref{fig:hidden-error-growth}. For random initialization, expectations
first average positions within each sequence and then average equally over
initialization--sequence pairs; the pretrained curve uses the same sequence-level
averaging. Taking expectations in
Prop.~\ref{prop:absolute-hidden-error-recurrence} gives
\begin{align}
&\mathbb{E}\norm{\Delta\hv^{(\ell)}}_2^2
-\mathbb{E}\norm{\Delta\hv^{(\ell-1)}}_2^2
={}
\mathbb{E}\norm{\Delta\uv^{(\ell)}}_2^2
+2\mathbb{E}\iprod{\Delta\hv^{(\ell-1)},\Delta\uv^{(\ell)}}.
\label{eq:appendix-expected-absolute-recurrence}
\end{align}
Across the 64 measured blocks, the ratio of the left-hand side of
Eq.~\eqref{eq:appendix-expected-absolute-recurrence} to
$\mathbb{E}\norm{\Delta\uv^{(\ell)}}_2^2$ lies in $[0.9989,1.0003]$.
Thus, over this measured depth range,
\begin{equation}
\mathbb{E}\norm{\Delta\hv^{(L)}}_2^2
\approx
\sum_{\ell=1}^{L}\mathbb{E}\norm{\Delta\uv^{(\ell)}}_2^2.
\label{eq:appendix-random-orthogonal-growth}
\end{equation}
The mean squared block-update error
$\mathbb{E}\norm{\Delta\uv^{(\ell)}}_2^2$ does not follow a single power over
blocks $1$--$64$. A three-parameter saturation curve instead gives
\begin{equation*}
\mathbb{E}\norm{\Delta\uv^{(\ell)}}_2^2
\approx
A\frac{\ell^\beta}{\tau^\beta+\ell^\beta},
\qquad
A=1.7154\times10^5,\quad \beta=1.143,\quad \tau=13.31,
\end{equation*}
with log-scale $R^2=0.99931$ and original-scale $R^2=0.99956$. Substituting this fit into
Eq.~\eqref{eq:appendix-random-orthogonal-growth}:
\begin{equation*}
\sqrt{\mathbb{E}\norm{\Delta\hv^{(L)}}_2^2}
\approx
\left[
\sum_{\ell=1}^{L}
A\frac{\ell^\beta}{\tau^\beta+\ell^\beta}
\right]^{1/2}.
\end{equation*}
This predicted curve matches the measured root-mean-square (RMS) hidden-error
magnitude with
$R^2=0.999981$ and differs by $0.014\%$ at block 64. Separately, a power fit
over the measured depth range gives exponent $0.801$ with log--log $R^2=0.997$
and a standard deviation of $0.0014$ across trajectory fits for each
initialization and sequence; the
mean $\mathbb{E}\norm{\Delta\hv^{(\ell)}}_2$ reported in the main text has the same fitted
exponent. Fig.~\ref{fig:app-random-init-power-fits}(C) directly fits the mean
relative error, yielding exponent $0.302$ with log--log $R^2=0.984$. These fits
support the main-text $\ell^{0.80}$ and
$\ell^{0.30}$ descriptions over the measured 64 blocks, without asserting
asymptotic growth laws.

\begin{figure}[H]
  \centering
  \includegraphics[width=\textwidth]{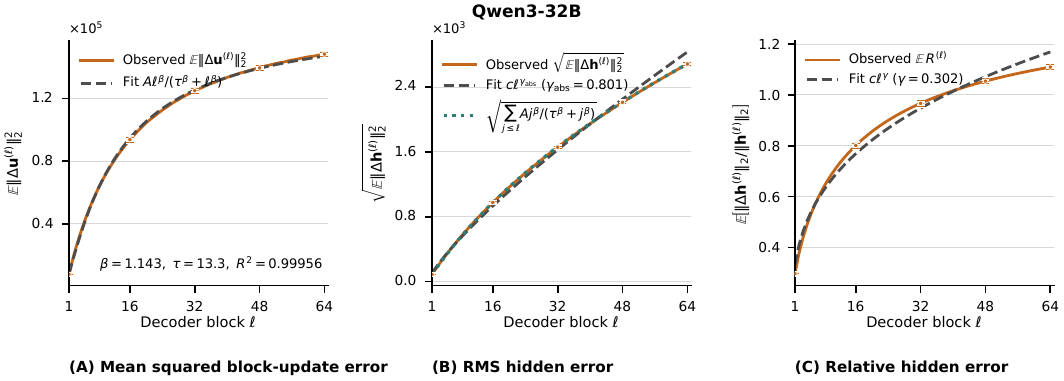}
  \caption{\textbf{Fits to block-update and hidden-error growth in randomly
  initialized Qwen3-32B.}
  \textbf{(A)} A three-parameter curve fits
  $\mathbb{E}\norm{\Delta\uv^{(\ell)}}_2^2$ over the measured blocks.
  \textbf{(B)} Summing these fitted values closely matches the measured RMS
  hidden error, whose growth is also fitted by a power
  law. \textbf{(C)} A separate power law fits the mean relative hidden error.
  Error bars show standard deviation.}
  \label{fig:app-random-init-power-fits}
\end{figure}

The same fit form for $\mathbb{E}\norm{\Delta\uv^{(\ell)}}_2^2$ also describes
randomly initialized Qwen3-4B, Qwen3-8B, and OLMo3-7B Stage-1 step 0
(Fig.~\ref{fig:app-random-init-fit-generalization});
across the four models the hidden-error exponents fitted over the measured depth
range are $0.70$--$0.83$,
the corresponding mean-relative-error exponents are $0.20$--$0.33$, and the
recurrence-predicted final RMS errors match the measured values to within
$0.02\%$. Adding an unconstrained intercept to the power fit yields negative intercepts,
violating the boundary $\Delta h^{(0)}=0$; we therefore report
these as descriptive fits over the measured intervals, not growth laws.

\begin{figure}[H]
  \centering
  \includegraphics[width=\textwidth]{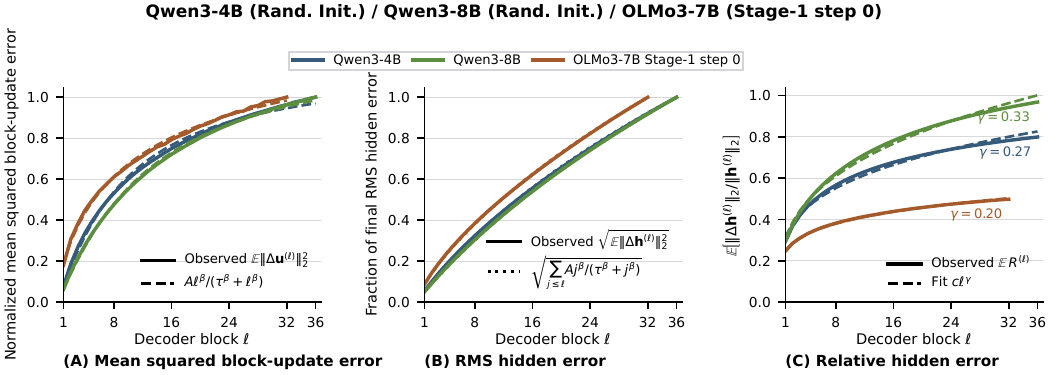}
  \caption{\textbf{Fits to block-update and hidden-error growth across randomly
  initialized models.} Results in \textbf{(A)} and \textbf{(B)} are
  normalized by their final values for comparison across Qwen3-4B, Qwen3-8B,
  and OLMo3-7B Stage-1 step 0.
  \textbf{(A)} Fitted curves follow the measured
  $\mathbb{E}\norm{\Delta\uv^{(\ell)}}_2^2$.
  \textbf{(B)} Summing the fitted values closely matches the measured RMS hidden
  error. \textbf{(C)} Power laws fit the unnormalized mean relative hidden
  error. All fits are restricted to the measured blocks.}
  \label{fig:app-random-init-fit-generalization}
\end{figure}

\paragraph{Hidden-error growth in pretrained models} We do not assign the pretrained curve a global power exponent because its
growth is not uniform across depth. We quantify its rate over intervals. Define the average
per-block relative-error increase over blocks $a+1,\ldots,b$ as
\begin{equation*}
\overline{s}_{a:b}^{\mathrm{rel}}
:=
\frac{
\mathbb{E}\!\left[
\norm{\Delta\hv^{(b)}}_2/\norm{\hv^{(b)}}_2
\right]
-
\mathbb{E}\!\left[
\norm{\Delta\hv^{(a)}}_2/\norm{\hv^{(a)}}_2
\right]
}{b-a}.
\end{equation*}
Over blocks 33--64, $\overline{s}_{33:64}$ is $1.46\times10^{-3}$ for pretrained
Qwen3-32B and $4.45\times10^{-3}$ at random initialization. Thus, the measured
second-half relative-error growth is $3.1\times$ slower after pretraining. At
block 64, the mean hidden-error magnitude and mean relative hidden error are
$5.5\times$ and $6.7\times$ smaller, respectively; these final-block ratios
describe the final errors rather than their growth rates.

Fig.~\ref{fig:app-random-pretrained-models} compares randomly initialized
and pretrained models across three architecture families. At the final block,
random initialization increases relative hidden error by $7.9\times$ in
Qwen3-8B, $5.9\times$ in OLMo3-32B, and $4.0\times$ in Gemma3-4B.

\begin{figure}[H]
  \centering
  \includegraphics[width=0.8\textwidth]{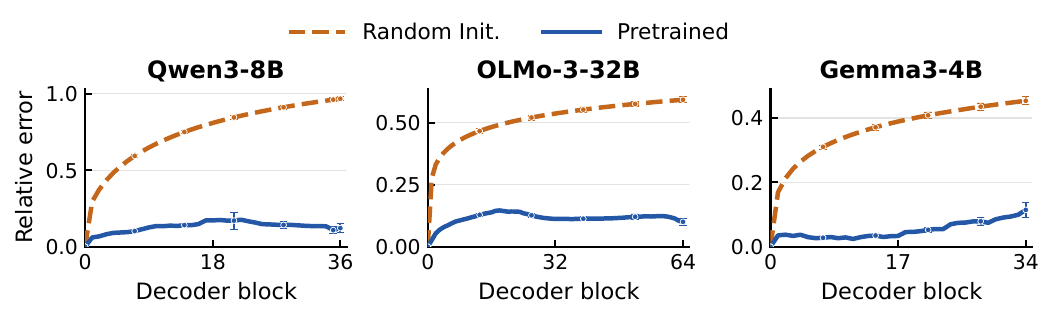}
  \caption{\textbf{Randomly initialized and pretrained models across three architecture families.}
  Relative hidden errors are shown for Qwen3-8B, OLMo3-32B, and
  Gemma3-4B under the same NVFP4 conversion. Error bars show
  standard deviation. Qwen3-8B and OLMo3-32B use BF16 for
  non-quantized computation; Gemma3 uses FP32 for the original pass and all
  non-quantized operations.}
  \label{fig:app-random-pretrained-models}
\end{figure}

\subsection{Hidden-Error Recurrence across Models}
\label{app:cross-model-relative-recurrence}

The exact recurrence in Thm.~\ref{thm:relative-hidden-error-recurrence}
separates each block's change in squared relative hidden error into the three
terms defined in the main text: relative additional error
$\termlabel{termcontrast}{T_{\rm add}}$, error interaction
$\termlabel{terminteraction}{T_{\rm inter}}$, and residual-norm contribution
$\termlabel{termalignment}{T_{\rm align}}$. We test whether the same signed
balance holds beyond the Qwen3-32B result in the main text.

We evaluate five pretrained dense and mixture-of-experts models on C4,
WikiText-103, and GSM8K text. Model precision, quantization, aggregation, and
position selection follow App.~\ref{app:main-figure-protocols}. In every
setting, the negative interaction term
$\termlabel{terminteraction}{T_{\rm inter}}$ offsets part of
$\termlabel{termcontrast}{T_{\rm add}}$. The residual-norm term
$\termlabel{termalignment}{T_{\rm align}}$ gives a further reduction, although
the size of each contribution varies by model and dataset. The signs
$\termlabel{termcontrast}{T_{\rm add}}>0$ and
$\termlabel{terminteraction}{T_{\rm inter}}<0$ remain unchanged across the
tested filter thresholds.
For the primary Qwen3-32B model, the mean final relative hidden error across
the three datasets is $0.245$, as reported in Sec.~\ref{sec:lmhead-geometry}.

Figs.~\ref{fig:app-sec3-counteraction-generalization-cumulative}
and~\ref{fig:app-sec3-counteraction-generalization-per-block} show the
cumulative and per-block results.

\begin{figure}[H]
  \centering
  \includegraphics[width=\textwidth]{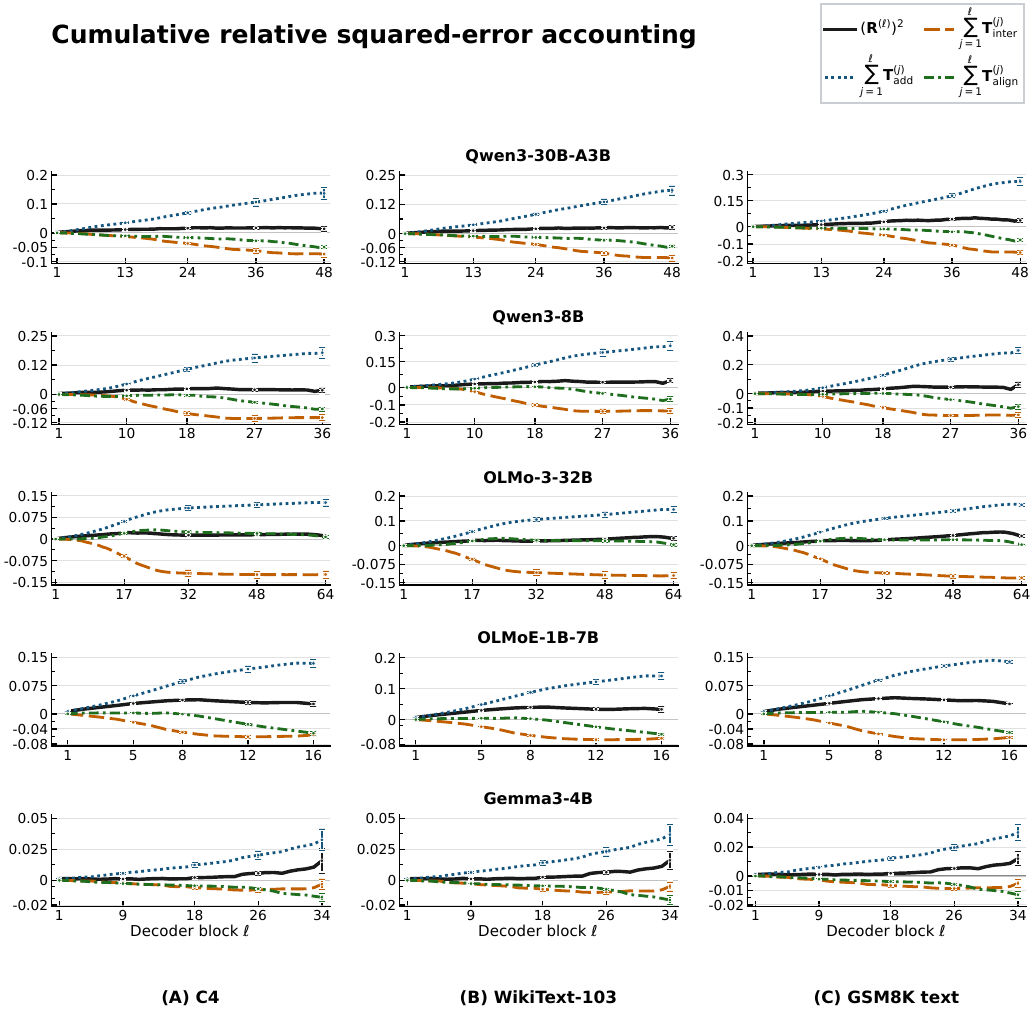}
  \caption{\textbf{Cumulative relative squared-error accounting.}
  Rows are models and columns are datasets. The black curve is the exact
  squared relative hidden error; the colored curves are the cumulative
  three-term contributions. Error bars show one sample standard deviation
  across complete input sequences. Position selection is described in
  App.~\ref{app:relative-recurrence-position-selection}.}
  \label{fig:app-sec3-counteraction-generalization-cumulative}
\end{figure}

\begin{figure}[H]
  \centering
  \includegraphics[width=\textwidth]{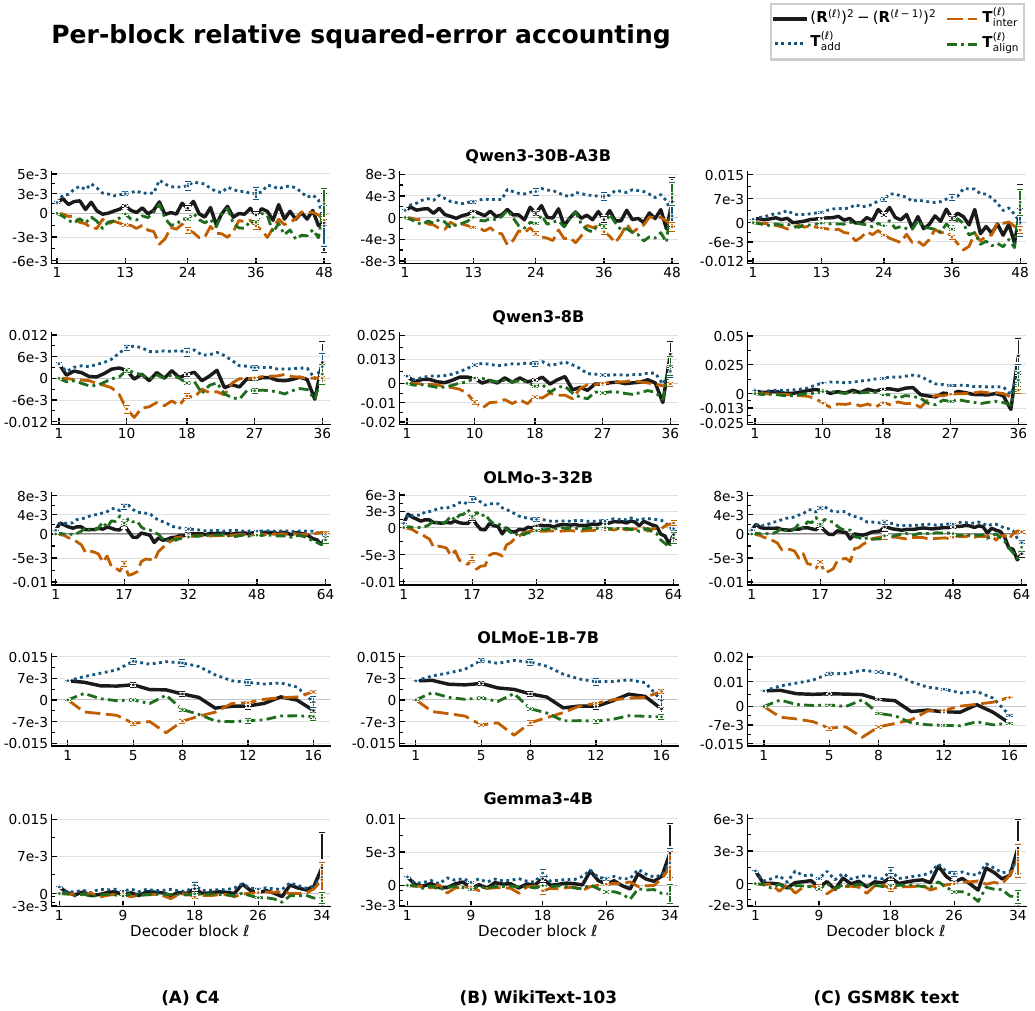}
  \caption{\textbf{Per-block relative squared-error contribution.}
  Rows are models and columns are datasets. The black curve is the exact
  blockwise change; the colored curves are the three terms in
  Thm.~\ref{thm:relative-hidden-error-recurrence}. Error bars show one sample
  standard deviation across complete input sequences.}
  \label{fig:app-sec3-counteraction-generalization-per-block}
\end{figure}

\newpage
\subsection{Hidden-Error Analysis: Abnormal Trajectory Filtering}
\label{app:relative-recurrence-position-selection}

Counteraction slows hidden-error growth on average, but it does not keep every
quantized next-token trajectory stable. At a small fraction of positions, the
hidden error becomes unusually large and the original and W4 hidden-state
norms separate sharply across depth. These trajectories can dominate averages
of the signed recurrence terms, although the recurrence formula is exact at
every position. We therefore separate these cases with a trajectory filter,
test its threshold sensitivity, and examine the excluded cases.
App.~\ref{app:main-figure-protocols} gives the full selection and aggregation
protocol.

Each decoding position $i$ produces a BF16/W4 hidden-state trajectory
$\{\hv_i^{(\ell)},\compacthat{\hv}{(\ell)}_i\}_{\ell=0}^{L}$. We compute the
largest relative difference between the BF16 and W4 hidden-state norms along
that trajectory:
\begin{equation*}
g_i^{\max}
:=\max_{\ell\in\{0,\ldots,L\}}
\frac{\left|\norm{\widehat{\hv}_i^{(\ell)}}_2
-\norm{\hv_i^{(\ell)}}_2\right|}
{\norm{\hv_i^{(\ell)}}_2}.
\end{equation*}
By default, we retain position $i$ when $g_i^{\max}\leq0.5$, which retains
$99.72\%$ of positions across the evaluated settings. The same positions are
used at every layer, so all points on a layerwise curve refer to one fixed
population. We form the recurrence terms of
Thm.~\ref{thm:relative-hidden-error-recurrence} at each retained position and
then average over positions within each complete sequence.
Error bars in the reported figures show one sample standard
deviation across complete input sequences.

We repeat the analysis with thresholds $1.0$, $0.5$, $0.25$, and $0.1$, and
without filtering. Tab.~\ref{tab:app-hidden-error-filter-sensitivity}
reports $\mathbb{E}[R^{(L)}]$, first averaging retained positions
within each sequence and then giving equal weight to each sequence and
model-dataset setting, where
$R^{(L)}:=\norm{\compacthat{\hv}{(L)}-\hv^{(L)}}_2/
\norm{\hv^{(L)}}_2$ is the final relative hidden error defined in
Sec.~\ref{sec:residual-absolute-counteraction}.

\begingroup
\begin{table}[H]
  \centering
  \caption{\textbf{Final relative hidden error under different
  hidden-norm-gap thresholds.} We exclude position $i$ with $g^{\max}_i > \text{threshold}$. Values are macro-averaged over the evaluated model--dataset settings. Both sums in the last column run over decoder blocks $1$--$L$.}
  \label{tab:app-hidden-error-filter-sensitivity}
  \small
  \begin{tabular}{lccc}
    \toprule
    Norm-gap threshold & Positions retained
      & $\mathbb{E}[R^{(L)}]$
      & $\mathbb{E}\!\left[-\sum T_{\rm inter}^{(\ell)}\,/\,
         \sum T_{\rm add}^{(\ell)}\right]$ \\
    \midrule
    Unfiltered & $100.00\%$ & $0.148$ & $30.4\%$ \\
    $1.0$ & $99.92\%$ & $0.147$ & $57.7\%$ \\
    $0.5$ (default) & $99.72\%$ & $0.147$ & $50.4\%$ \\
    $0.25$ & $98.62\%$ & $0.144$ & $50.2\%$ \\
    $0.1$ & $89.34\%$ & $0.136$ & $51.8\%$ \\
    \bottomrule
  \end{tabular}
\end{table}
\endgroup

The default threshold is a suitable cutoff: it changes the mean final relative
hidden error only from $0.148$ to $0.147$ while retaining $99.72\%$ of
positions. It removes rare scale failures while preserving the main
statistics. When all positions
are included, the fraction of
$\termlabel{termcontrast}{T_{\rm add}}$ canceled by
$\termlabel{terminteraction}{T_{\rm inter}}$ falls from $50.4\%$ to $30.4\%$.
Counteraction is less effective when the W4 norm trajectory departs sharply
from the original trajectory. {{These rare errors appear to exceed the range that the model can regulate through counteraction.}}

For the LM-head norm statistics used in Sec.~\ref{sec:lmhead-geometry}, the default filter gives
\[
\mathbb{E}\left|
\frac{\norm{\widehat{\hv}_{\rm LM}}_2}{\norm{\hv_{\rm LM}}_2}-1
\right|=0.0372,
\qquad
\mathbb{E}\left(
\frac{\norm{\widehat{\hv}_{\rm LM}}_2}{\norm{\hv_{\rm LM}}_2}-1
\right)^2=0.00320,
\]
averaged equally over C4, WikiText-103, and GSM8K text. Without filtering, the
values are $0.0373$ and $0.00324$, only $0.31\%$ and $1.27\%$ higher. Thus,
the LM-head-input norm remains statistically stable even without the filter, supporting the
approximation
$\norm{\widehat{\hv}_{\rm LM}}_2/\norm{\hv_{\rm LM}}_2\approx1$
used in Sec.~\ref{sec:lmhead-top-rank-geometry}.

Fig.~\ref{fig:app-sec3-filtered-fail-cases} shows the excluded regime directly.
For each of Qwen3-32B, Qwen3-8B, and OLMo3-7B, we select the position whose
$g_i^{\max}$ is closest to the median among that model's excluded positions.
This deterministic rule does not inspect the recurrence values or output
probabilities. The BF16 and W4 hidden-state norms separate sharply in these examples, and the
recurrence terms of Thm.~\ref{thm:relative-hidden-error-recurrence} become large enough to dominate an unfiltered
mean despite the small number of such positions.

\begin{figure}[H]
  \centering
  \includegraphics[width=0.94\textwidth]{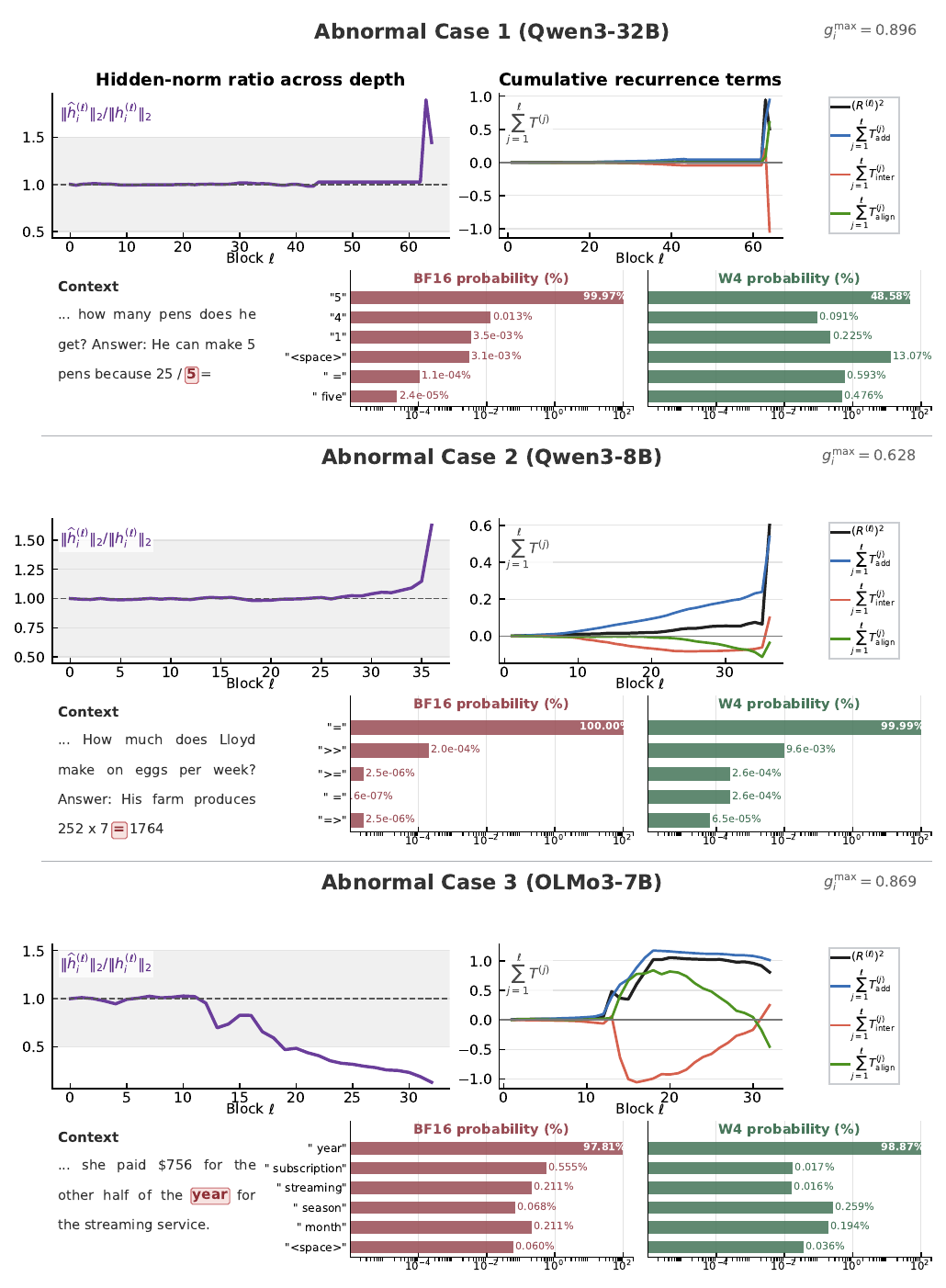}
  \caption{\textbf{Rare hidden-state trajectory failures occur when BF16 and W4 hidden
  norms diverge.} All three cases are from GSM8K text. Rows show Qwen3-32B,
  Qwen3-8B, and OLMo3-7B positions selected by proximity to each model's median
  excluded $g_i^{\max}$; the value beside each row title is this selection
  statistic, not the plotted ordinate. Selection does not use recurrence
  values or output changes. The left column shows
  $\norm{\compacthat{\hv_i}{(\ell)}}_2/\norm{\hv_i^{(\ell)}}_2$, with the retained
  range $[0.5,1.5]$ shaded in gray. The right column shows the cumulative
  recurrence terms. Both columns use decoder block $\ell$ on the horizontal
  axis. The lower strip shows the context and BF16/W4 next-token probabilities.}
  \label{fig:app-sec3-filtered-fail-cases}
\end{figure}

\clearpage
\subsection{Hidden-Error Growth: Late-Layer Differences across Data}
\label{app:sec5-data-dependence}

Even with the counteraction mechanism, hidden-error growth can also vary across input domains for pretrained models. Qwen3-32B accumulates relative
hidden error faster on GSM8K text than on C4 in the later blocks. Here, we
identify the responsible recurrence term and test it by applying the same
update rescaling to both streams.

The relative additional-error term in
Thm.~\ref{thm:relative-hidden-error-recurrence} is
\begin{equation*}
\termlabel{termcontrast}{T_{\rm add}^{(\ell)}}
=
\left(\frac{\norm{\uv^{(\ell)}}_2}{\norm{\hv^{(\ell)}}_2}\right)^2
\left[
\left(\frac{\norm{\Delta\uv^{(\ell)}}_2}
{\norm{\uv^{(\ell)}}_2}\right)^2
-\bigl(R^{(\ell-1)}\bigr)^2
\right].
\end{equation*}
Fig.~\ref{fig:app-sec5-data-dependence}(B) shows that the relative hidden-error
trajectories remain close in the earlier blocks but separate sharply later.
Over blocks $36$--$64$, the coefficient
$\bigl(\norm{\uv^{(\ell)}}_2/\norm{\hv^{(\ell)}}_2\bigr)^2$ is larger on
GSM8K text (Fig.~\ref{fig:app-sec5-data-dependence}(A)). The block-update norms
are comparable across the two datasets, but the hidden-state norm is smaller
on average for GSM8K text. Its updates are therefore stronger relative to the
hidden state. When the bracketed term is positive, the larger coefficient
increases $\termlabel{termcontrast}{T_{\rm add}^{(\ell)}}$ and accelerates
relative hidden-error growth.

Fig.~\ref{fig:app-sec5-data-dependence}(A) and
Fig.~\ref{fig:app-sec5-data-dependence}(B) show that the coefficient and hidden
error change together but do not isolate the coefficient. We keep the GSM8K
inputs and internal block computations fixed, then rescale only the residual
updates to match the C4 coefficient at each block. We intervene over blocks
$36$--$64$ and denote the modified trajectories by the subscript $c$:
\begin{equation*}
\hv_c^{(\ell)}
=\hv_c^{(\ell-1)}+\lambda^{(\ell)}\uv_c^{(\ell)},
\qquad
\widehat{\hv}_c^{(\ell)}
=\widehat{\hv}_c^{(\ell-1)}+\lambda^{(\ell)}\widehat{\uv}_c^{(\ell)}.
\end{equation*}
On the BF16 pass, we choose $\lambda^{(\ell)}$ so that the
position-averaged value
$\Eb\bigl(\lambda^{(\ell)}\norm{\uv_c^{(\ell)}}_2/
\norm{\hv_c^{(\ell)}}_2\bigr)^2$ on GSM8K dataset exactly matches the averaged value on C4 at the same block, then
reuse $\lambda^{(\ell)}$ on the W4 pass. Only the residual addition is
rescaled; the internal attention and MLP computations remain unchanged. This
oracle intervention is used only to test the coefficient, so we do not evaluate
output quality.

After matching the coefficients, the final
$\mathbb{E}[(R^{(L)})^2]$ decreases from $0.121$ to $0.062$, corresponding to a
$48.8\%$ reduction (Fig.~\ref{fig:app-sec5-data-dependence}(C)). This
oracle-style intervention shows that the coefficient contributes substantially
to the faster late-layer hidden-error growth on GSM8K text. More generally,
data change the relative strength of block updates across depth, producing
different hidden-state norm trajectories. The ratio
$\norm{\uv^{(\ell)}}_2/\norm{\hv^{(\ell)}}_2$ measures the signal generated by
block $\ell$ relative to its hidden state. When the bracketed term is positive,
a larger ratio gives newly generated block-update error a larger contribution
to the residual stream. Later blocks inherit this larger hidden error, so
subsequent propagation starts from a larger value. Thus,
\textbf{data-dependent block-update strength} can determine \textbf{where
hidden-error growth accelerates}. On GSM8K, the larger ratio after block 36
accounts for much of the rapid growth in
Fig.~\ref{fig:app-sec5-data-dependence}(B).

\begin{figure}[H]
  \centering
  \includegraphics[width=\textwidth]{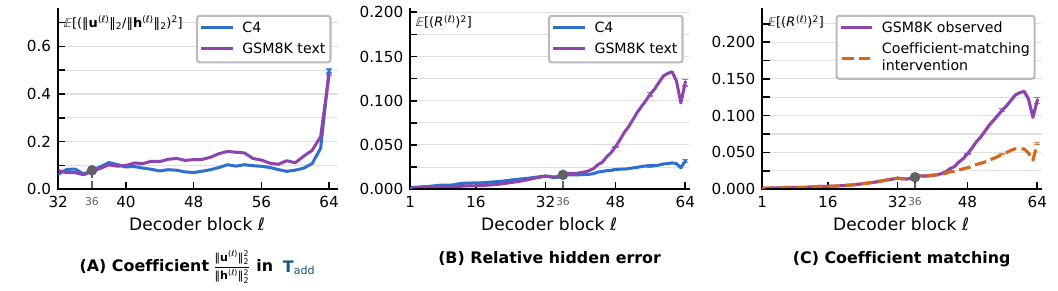}
  \caption{\textbf{The update-to-hidden norm ratio explains much of the
  data-dependent late hidden-error growth in Qwen3-32B.}
  \textbf{(A)} The coefficient
  $\norm{\uv^{(\ell)}}_2^2/\norm{\hv^{(\ell)}}_2^2$ in
  $\termlabel{termcontrast}{T_{\rm add}}$ is larger on GSM8K text over most
  late blocks. \textbf{(B)} The relative hidden error then grows faster.
  \textbf{(C)} Matching the coefficient to the C4 values from block $36$
  reduces this growth. Error bars show the standard error across $64$
  sequences at selected blocks.}
  \label{fig:app-sec5-data-dependence}
\end{figure}

\newpage
\subsection{Counteraction Emergence during Pretraining}
\label{app:pythia-checkpoint-trajectories}

Sec.~\ref{sec:residual-counteraction} identifies counteraction as a major
difference between randomly initialized and pretrained models. At random
initialization, the interaction between the block-input hidden error
$\Delta\hv^{(\ell-1)}$ and block-update error $\Delta\uv^{(\ell)}$ is nearly
zero. In pretrained Qwen3-32B, it is negative over much of the network and
offsets part of the block-update error, slowing hidden-error growth. The
OLMo3-7B Stage-1 checkpoints in Fig.~\ref{fig:hidden-error-growth} show this
change along pretraining, from weak interaction at step 0 to the negative
interaction observed at the final checkpoint. We use Pythia checkpoints to
test whether the same change appears outside the OLMo family.

Fig.~\ref{fig:app-pythia-checkpoint-trajectories} repeats the checkpoint
comparison for Pythia-1.4B and Pythia-2.8B. In both models, the cosine between
$\Delta\hv^{(\ell-1)}$ and $\Delta\uv^{(\ell)}$ changes from a near-zero
value at step 0 to broadly negative values at the final checkpoint. The OLMo3
and Pythia interaction curves show the same change during training:
counteraction is weak near initialization and becomes more pronounced at later
checkpoints.

\begin{figure}[H]
  \centering
  \includegraphics[width=\textwidth]{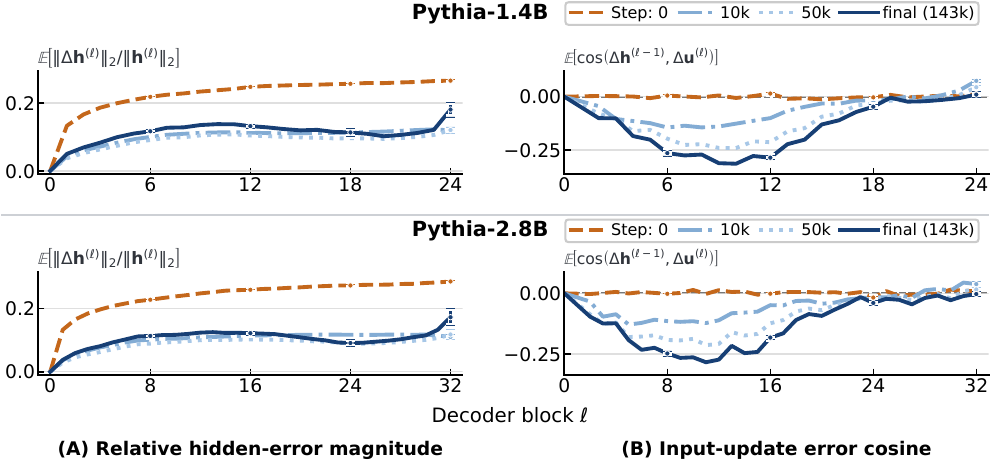}
  \caption{\textbf{Counteraction strengthens during Pythia pretraining.}
  Rows show four checkpoints of Pythia-1.4B and Pythia-2.8B under the same
  NVFP4 weight conversion. Columns report relative hidden error and the cosine
  between the block-input hidden error and block-update error. Error bars show
  standard deviation.}
  \label{fig:app-pythia-checkpoint-trajectories}
\end{figure}

\newpage
\subsection{Counteraction Source: Block-Input-Error Response}
\label{app:block-response-attribution}

Sec.~\ref{sec:counteraction-source} attributes counteraction mainly to the
block's response to the hidden error at its input. Here, we provide the
detailed evidence in two steps. We first show that \emph{counteraction} is specific
to the \emph{error} dynamics rather than a general relation between a hidden state
and its residual update.
We then decompose the block-update error to determine
whether the negative interaction comes from the direct weight perturbation or
from the block's response to the inherited hidden error.

\textbf{Counteraction is specific to the error dynamics.}
Counteraction refers to the tendency of the block-update error
$\Delta\uv^{(\ell)}$ to point against the hidden error
$\Delta\hv^{(\ell-1)}$ already present at the block input
(Sec.~\ref{sec:residual-absolute-counteraction}).
A related phenomenon has been observed for native residual contributions:
later Transformer layers can partially oppose earlier contributions to correct each other
\citep{patrawala2025llmlayerscorrecteachother}.
We therefore ask whether the negative interaction between quantization errors simply reflects a general tendency of a block update to oppose its input hidden state.

Fig.~\ref{fig:app-block-response-native-cosines} compares these relations.
The native hidden-state/update cosines
$\cos\angle(\hv^{(\ell-1)},\uv^{(\ell)})$ and
$\cos\angle(\compacthat{\hv}{(\ell-1)},\compacthat{\uv}{(\ell)})$
are positive in $65.6\%$ and $64.1\%$ of the blocks, respectively.
By contrast,
$\cos\angle(\Delta\hv^{(\ell-1)},\Delta\uv^{(\ell)})$
is negative in $82.5\%$ of the blocks.
Thus, the strong negative relation is not a general property of the residual
update; it appears much more consistently between the
inherited input hidden error
and the new block-update error.

\begin{figure}[H]
  \centering
  \includegraphics[width=0.9\textwidth]{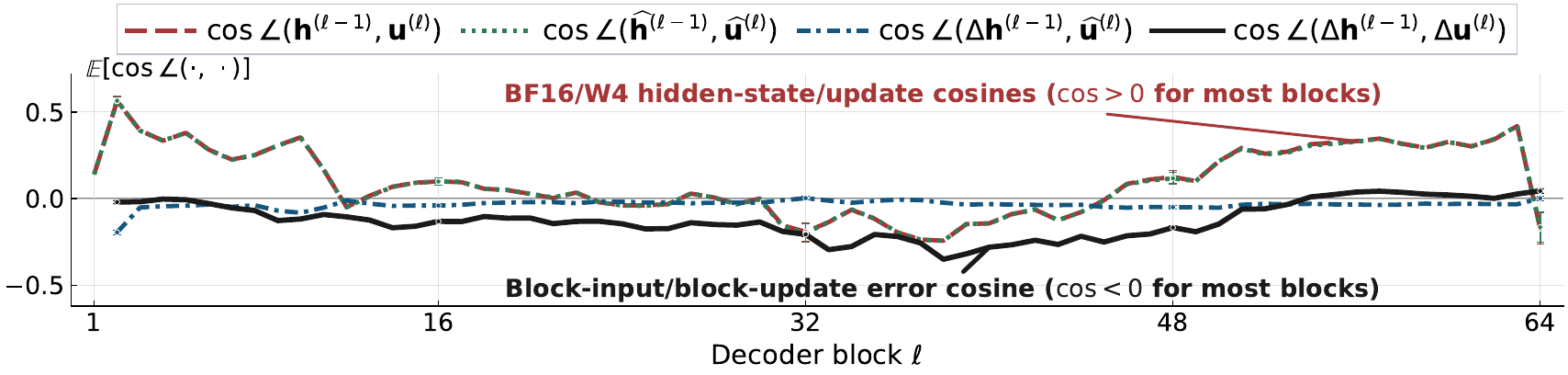}
  \caption{\textbf{Counteraction is specific to the error dynamics in
  Qwen3-32B.} BF16 and W4 hidden-state/update cosines are compared with the
  cosine between the block-input and block-update errors on C4 inputs. Error
  bars show one sample standard deviation.}
  \label{fig:app-block-response-native-cosines}
\end{figure}

\textbf{Decomposing the block-update error.}
We next ask which part of $\Delta\uv^{(\ell)}$ produces this negative
interaction. At block $\ell$, the original and quantized updates are
\begin{equation*}
\uv^{(\ell)}
=
{\color{cleanred}\Fv^{(\ell)}}(\hv^{(\ell-1)}),
\compacthat{\uv}{(\ell)}
=
{\color{quantgreen}\compacthat{\Fv}{(\ell)}}
(\compacthat{\hv}{(\ell-1)}).
\end{equation*}
We separate these two effects exactly:
\begin{equation}
\Delta\uv^{(\ell)}=
\underbrace{
{\color{quantgreen}\widehat{\Fv}^{(\ell)}}(\hv^{(\ell-1)})
-
{\color{cleanred}\Fv^{(\ell)}}(\hv^{(\ell-1)})
}_{
\Delta\uv_{\rm weight}^{(\ell)}
\ \text{(direct weight effect)}
}
+
\underbrace{
{\color{quantgreen}\widehat{\Fv}^{(\ell)}}
(\widehat{\hv}^{(\ell-1)})
-
{\color{quantgreen}\widehat{\Fv}^{(\ell)}}
(\hv^{(\ell-1)})
}_{
\Delta\uv_{\rm response}^{(\ell)}
\ \text{(response to block-input error)}
}.
\label{eq:main-block-response-decomposition}
\end{equation}
The first component changes the block weights while keeping its input fixed
at $\hv^{(\ell-1)}$. The second keeps the quantized block fixed and changes
only its input from $\hv^{(\ell-1)}$ to
$\compacthat{\hv}{(\ell-1)}$.
Because
$\termlabel{terminteraction}{T_{\rm inter}^{(\ell)}}=\frac{
    2\iprod{\Delta\hv^{(\ell-1)},\Delta\uv^{(\ell)}}
}{
    \norm{\hv^{(\ell)}}_2^2
}$
is linear in $\Delta\uv^{(\ell)}$, this decomposition also gives
\begin{equation*}
\termlabel{terminteraction}{T_{\rm inter}^{(\ell)}}
=
T_{\rm inter,weight}^{(\ell)}
+
T_{\rm inter,response}^{(\ell)}
=
\frac{
    2\iprod{\Delta\hv^{(\ell-1)},\Delta\uv^{(\ell)}_{\rm weight} }
}{
    \norm{\hv^{(\ell)}}_2^2
}+
\frac{
    2\iprod{\Delta\hv^{(\ell-1)},\Delta\uv^{(\ell)}_{\rm response} }
}{
    \norm{\hv^{(\ell)}}_2^2
}
.
\end{equation*}
Thus, the source of counteraction can be identified by asking which
component contributes the negative interaction with
$\Delta\hv^{(\ell-1)}$.
We measure both components by replaying every decoder block of Qwen3-32B on
the C4 inputs from Sec.~\ref{sec:residual-absolute-counteraction}.

\begin{figure}[H]
  \centering
  \includegraphics[width=0.9\textwidth]{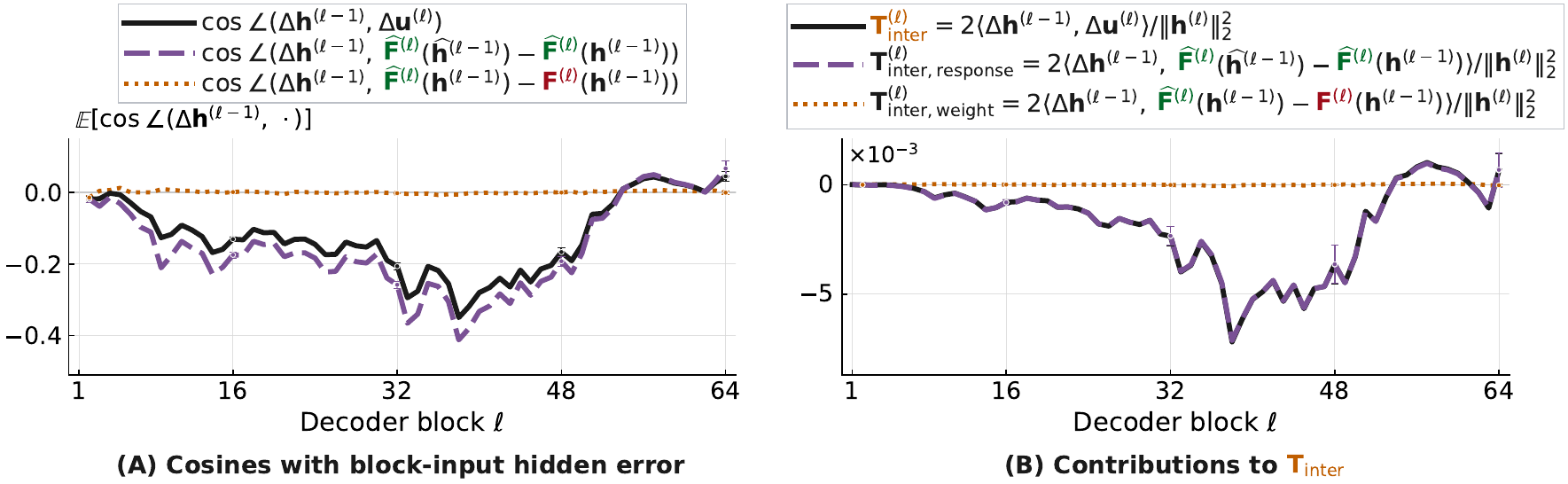}
  \caption{\textbf{The response to block-input error carries counteraction in
  Qwen3-32B.} C4 inputs.
  \textbf{(A)} Cosines between the block-input error and the
  block-update error $\Delta\uv^{(\ell)}$, split
  into the direct weight effect and input-error response from
  Eq.~\eqref{eq:main-block-response-decomposition}.
  \textbf{(B)} Their contributions
  $T_{\rm inter,weight}^{(\ell)}$ and
  $T_{\rm inter,response}^{(\ell)}$ to
  $\termlabel{terminteraction}{T_{\rm inter}^{(\ell)}}$.
  Error bars show one sample standard deviation.}
  \label{fig:app-block-response-attribution}
\end{figure}

\textbf{The input-error response carries the negative direction.}
Fig.~\ref{fig:app-block-response-attribution}(A) compares the directions of
the two components. The full block-update error
$\Delta\uv^{(\ell)}$ closely follows
$\Delta\uv_{\rm response}^{(\ell)}$, whose cosine with the block-input
hidden error is negative in $83\%$ of the blocks. By contrast, the cosine curve for $\Delta\uv_{\rm weight}^{(\ell)}$ stays within $0.02$ of zero at every block. 
The direct weight-quantization perturbation therefore produces a nonzero update error, but it has no consistent direction relative to the
hidden error already present at the block input.

Fig.~\ref{fig:app-block-response-attribution}(B) measures the quantity
directly relevant to counteraction: the signed contribution to
$\termlabel{terminteraction}{T_{\rm inter}}$.
After summing over blocks,
$
\sum_{\ell=1}^{L}T_{\rm inter,response}^{(\ell)}
$
accounts for $99.9\%$ of the negative
$
\sum_{\ell=1}^{L}
\termlabel{terminteraction}{T_{\rm inter}^{(\ell)}},
$
whereas
$\sum_{\ell=1}^{L}T_{\rm inter,weight}^{(\ell)}$
accounts for less than $0.1\%$.
This attribution is directional rather than a consequence of the response component being much larger in norm. In fact,
$
\Eb \bigl[
\frac{\norm{\Delta\uv_{\rm weight}^{(\ell)}}_2}
     {\norm{\Delta\uv^{(\ell)}}_2}
\bigr]
$
is $84\%$ of
$
\Eb \bigl[
\frac{\norm{\Delta\uv_{\rm response}^{(\ell)}}_2}
     {\norm{\Delta\uv^{(\ell)}}_2}
\bigr].
$
The two components therefore have \emph{comparable magnitudes}, but only the \textbf{\emph{input-error response}} points \emph{consistently against} the inherited hidden error.

\textbf{Robustness to the decomposition path.}
The exact decomposition of $\Delta\uv^{(\ell)}$ is not unique because one
may choose either block map for the intermediate evaluation. To check that
the attribution above does not depend on this choice, we instead add and
subtract
${\color{cleanred}\Fv^{(\ell)}}
(\widehat{\hv}^{(\ell-1)})$, giving
\begin{equation*}
\Delta\uv^{(\ell)}
=
\left[
{\color{quantgreen}\widehat{\Fv}^{(\ell)}}
(\widehat{\hv}^{(\ell-1)})
-
{\color{cleanred}\Fv^{(\ell)}}
(\widehat{\hv}^{(\ell-1)})
\right]
+
\left[
{\color{cleanred}\Fv^{(\ell)}}
(\widehat{\hv}^{(\ell-1)})
-
{\color{cleanred}\Fv^{(\ell)}}
(\hv^{(\ell-1)})
\right].
\end{equation*}
Here the first term measures the direct weight effect at the W4 input,
whereas the second measures the response of the original block to the same
input error. This alternative decomposition gives the same qualitative
attribution: across blocks, the first term above contributes $+0.0013$ to the cumulative $T_{\rm inter}$, whereas the second contributes $-0.1118$ to the 
$\sum_{\ell=1}^{L}
\termlabel{terminteraction}{T_{\rm inter}^{(\ell)}}=-0.1105$.

Therefore, under either exact decomposition, \emph{counteraction} is carried mainly
by \textbf{\emph{how the block responds to the hidden error at its input}}. The resulting
change in the block update tends to oppose the accumulated hidden error and slow its propagation. 

\newpage
\subsection{Counteraction Intervention: Removal and Reversal}
\label{app:counteraction-intervention}

Sec.~\ref{sec:residual-relative-decomposition} shows that removing or reversing
counteraction increases hidden error and output change in Qwen3-32B. Here, we
give the online recursions and term-wise accounting, then repeat the intervention
in Qwen3-8B to test whether the effect transfers across model scales.

Removal and reversal follow separate intervened trajectories initialized from
$\hv^{(0)}$. In each recursion below,
$\compacthat{\hv}{(\ell-1)}$ is the current modified block input. We recompute
the block-input hidden error and W4 block-update error at every block:
$
\Delta\hv^{(\ell-1)}:=\widehat{\hv}^{(\ell-1)}-\hv^{(\ell-1)},
\Delta\uv^{(\ell)}:={\color{quantgreen}\widehat{\Fv}^{(\ell)}}
(\widehat{\hv}^{(\ell-1)})
-{\color{cleanred}\Fv^{(\ell)}}(\hv^{(\ell-1)}).
$
Here, $\Delta\uv^{(\ell)}$ follows the definition in
Eq.~\eqref{eq:setup-paired-block}, but is computed on the \emph{modified trajectory}.

For removal, the equal-norm error orthogonal to the block-input hidden error is
\begin{equation*}
\Delta\uv_{\perp}^{(\ell)}
:=\norm{\Delta\uv^{(\ell)}}_2
\frac{
\displaystyle \Delta\uv^{(\ell)}
-\frac{\iprod{\Delta\hv^{(\ell-1)},\Delta\uv^{(\ell)}}}
{\norm{\Delta\hv^{(\ell-1)}}_2^2}\Delta\hv^{(\ell-1)}
}{
\displaystyle \left\|
\Delta\uv^{(\ell)}
-\frac{\iprod{\Delta\hv^{(\ell-1)},\Delta\uv^{(\ell)}}}
{\norm{\Delta\hv^{(\ell-1)}}_2^2}\Delta\hv^{(\ell-1)}
\right\|_2
}.
\end{equation*}
The removal and reversal trajectories are then propagated directly as
\begin{align}
\widehat{\hv}_{\rm rem}^{(\ell)}
&=\widehat{\hv}_{\rm rem}^{(\ell-1)}+\uv^{(\ell)}+
\begin{cases}
\Delta\uv_{\perp}^{(\ell)}, & 17\leq\ell\leq48,\\
\Delta\uv^{(\ell)}, & \text{otherwise},
\end{cases}
\label{eq:appendix-counteraction-removal}\\
\widehat{\hv}_{\rm rev}^{(\ell)}
&=\widehat{\hv}_{\rm rev}^{(\ell-1)}+\uv^{(\ell)}+
\begin{cases}
-\Delta\uv^{(\ell)},
& {17\leq\ell\leq48 ~\text{and~}
\iprod{\Delta\hv^{(\ell-1)},\Delta\uv^{(\ell)}}<0},\\
\Delta\uv^{(\ell)}, & \text{otherwise}.
\end{cases}
\label{eq:appendix-counteraction-reversal}
\end{align}
Removal makes the applied block-update error orthogonal to the block-input hidden
error at every position in blocks $17\leq\ell\leq48$; reversal flips only
negative interactions in this interval. Both preserve
$\norm{\Delta\uv^{(\ell)}}_2$ at the intervened block,
although later block-update errors may change along the modified trajectory. A
deterministic orthogonal fallback handles degenerate projections but was never
triggered. Term-wise analysis follows the abnormal case exclusion selection in
App.~\ref{app:main-figure-protocols}. The same filtering rule is applied to each modified trajectory.

Fig.~\ref{fig:app-counteraction-intervention-terms} provides the term-wise
recurrence accounting for Qwen3-32B. Across the intervened blocks, removal
nearly eliminates cumulative $\termlabel{terminteraction}{T_{\rm inter}}$
while increasing cumulative $\termlabel{termcontrast}{T_{\rm add}}$ by
$4.3\times$. Together, these changes increase the final
$\Eb[(R^{(L)})^2]$ by $8.4\times$. Although the recurrence separates
$\termlabel{termcontrast}{T_{\rm add}}$ and
$\termlabel{terminteraction}{T_{\rm inter}}$ algebraically, changing the
trajectory also changes the subsequent values of
$\termlabel{termcontrast}{T_{\rm add}}$. Thus, counteraction limits error
growth both through direct negative interaction and by preventing larger
values of $\termlabel{termcontrast}{T_{\rm add}}$ later in the forward pass.

\begingroup
\setlength{\intextsep}{8pt}
\begin{figure}[H]
  \centering
  \includegraphics[width=\textwidth]{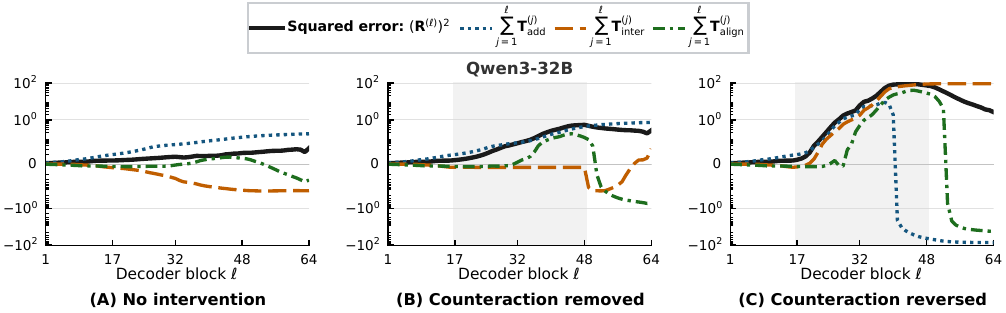}
  \caption{\textbf{Relative squared-error accounting after counteraction
  intervention in Qwen3-32B.} Columns compare unmodified W4 with
  counteraction removal and reversal over blocks $17$--$48$, where
  counteraction is strong. The corresponding hidden-error and output changes
  are reported in Sec.~\ref{sec:residual-relative-decomposition}. In contrast
  to the other figures, a symmetric
  logarithmic scale is used here to compare trajectories that span multiple
  orders of magnitude.}
  \label{fig:app-counteraction-intervention-terms}
\end{figure}
\endgroup

We repeat the intervention in Qwen3-8B over its strongest-counteraction
interval, blocks $6$--$23$. Fig.~\ref{fig:app-sec5-qwen8-intervention}
combines the hidden-error and output changes with the same term-wise
accounting. Removal again nearly eliminates cumulative
$\termlabel{terminteraction}{T_{\rm inter}}$, while cumulative
$\termlabel{termcontrast}{T_{\rm add}}$ increases by $2.7\times$ and the
final $\Eb[(R^{(L)})^2]$ increases by $4.8\times$. The Qwen3-8B result
therefore reproduces the same coupling between the negative interaction and
the subsequent hidden-state trajectory.

\label{app:sec5-qwen8-intervention}
\begin{figure}[H]
  \centering
  \includegraphics[width=0.95\textwidth]{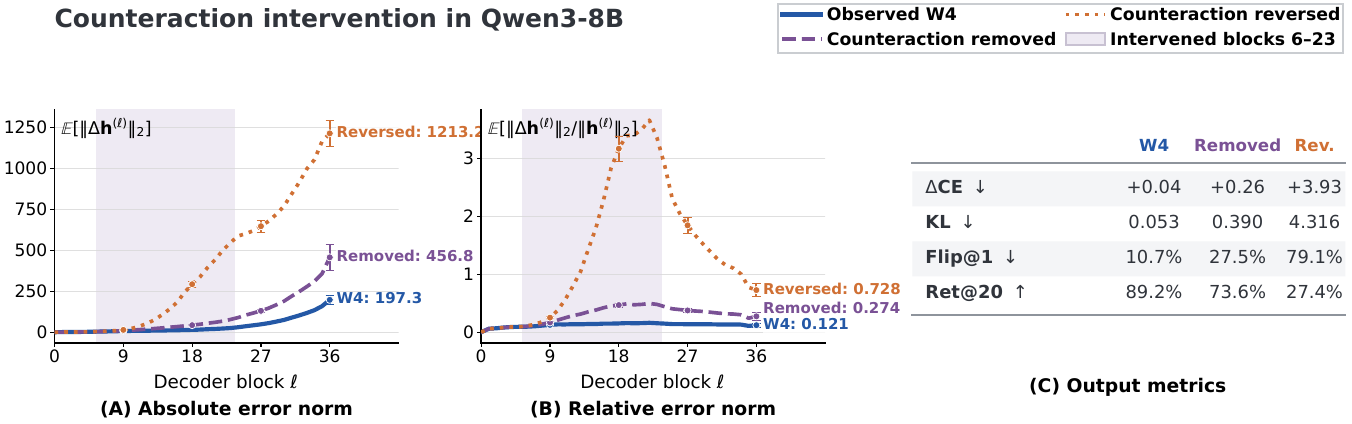}
  \par\vspace{6pt}
  \includegraphics[width=0.95\textwidth]{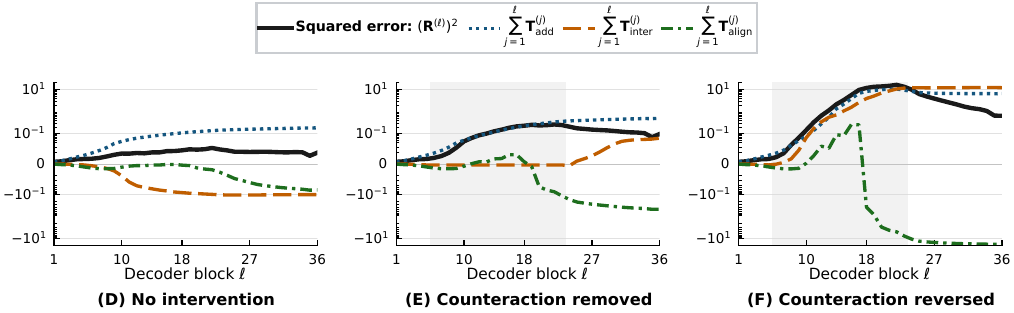}
  \caption{\textbf{Counteraction intervention in Qwen3-8B.}
  \textbf{(A,B)} Hidden-error trajectories and \textbf{(C)} output changes
  after removing or reversing counteraction over blocks $6$--$23$.
  \textbf{(D)--(F)} Relative squared-error accounting along the same three
  trajectories. Output metrics are defined in
  Sec.~\ref{sec:preferential-output-preservation}; error bars show one sample
  standard deviation across inputs. The lower panels use a symmetric
  logarithmic scale.}
  \label{fig:app-sec5-qwen8-intervention}
\end{figure}

\newpage
\subsection{Counteraction across Weight and Activation Quantization}
\label{app:sec5-quantization-components}

Sec.~\ref{sec:scope-model-quantization} asks whether the negative interaction
between block-input hidden error and block-update error is specific to weight
quantization. Fig.~\ref{fig:sec5-quantization-scope}(A) compares its blockwise
cosine under W4 (weight-only quantization), A4 (activation-only quantization),
and W4A4 (joint weight and activation quantization), and finds a similar
depth-wise pattern in all three settings. The cosine shows the direction of
the interaction, but not how much it contributes to the relative hidden-error
trajectory or how the remaining error changes the next-token distribution.

Fig.~\ref{fig:app-sec5-quantization-site-recurrence} applies
Thm.~\ref{thm:relative-hidden-error-recurrence} separately to each setting.
The relative-error trajectories and the amount of newly added error differ,
but the cumulative $\termlabel{terminteraction}{T_{\rm inter}}$ contribution
is negative in all three cases. Thus, the block-update error counteracts the
block-input hidden error whether the perturbation enters through weights,
activations, or both.

Fig.~\ref{fig:app-sec5-quantization-site-output} compares the corresponding
CE, KL, top-token retention, and rank-wise log-probability changes. The three
settings do not produce the same output error, even though their blockwise
counteraction curves are similar. Counteraction describes how hidden error is
limited through depth; output robustness also depends on how much error each
block introduces and how the remaining error affects the LM head.

\begin{figure}[H]
  \centering
  \includegraphics[width=0.94\textwidth]{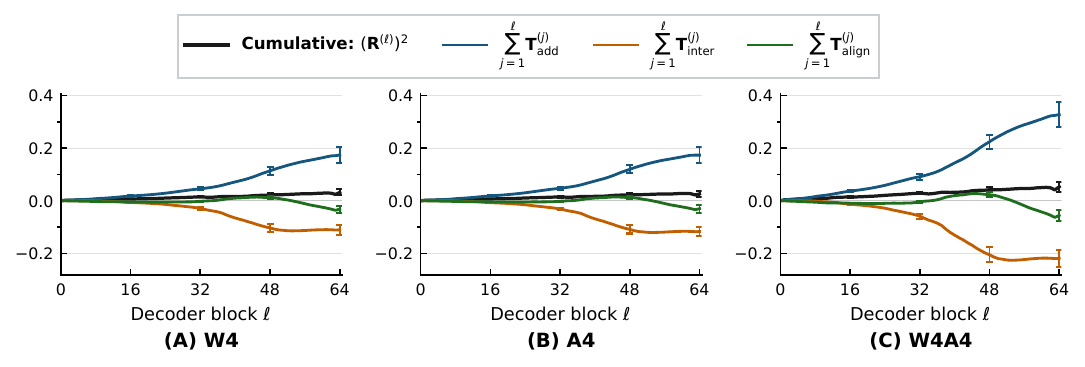}
  \caption{\textbf{Relative-error recurrence under weight and activation
  quantization.} Qwen3-32B on C4 under W4, A4, and W4A4. The black curve is
  the cumulative squared relative error; the colored curves are the cumulative
  contributions of $T_{\rm add}$, $T_{\rm inter}$, and $T_{\rm align}$.
  Error bars show one sample standard deviation across complete inputs.}
  \label{fig:app-sec5-quantization-site-recurrence}
\end{figure}

\begin{figure}[H]
  \centering
  \includegraphics[width=0.86\textwidth]{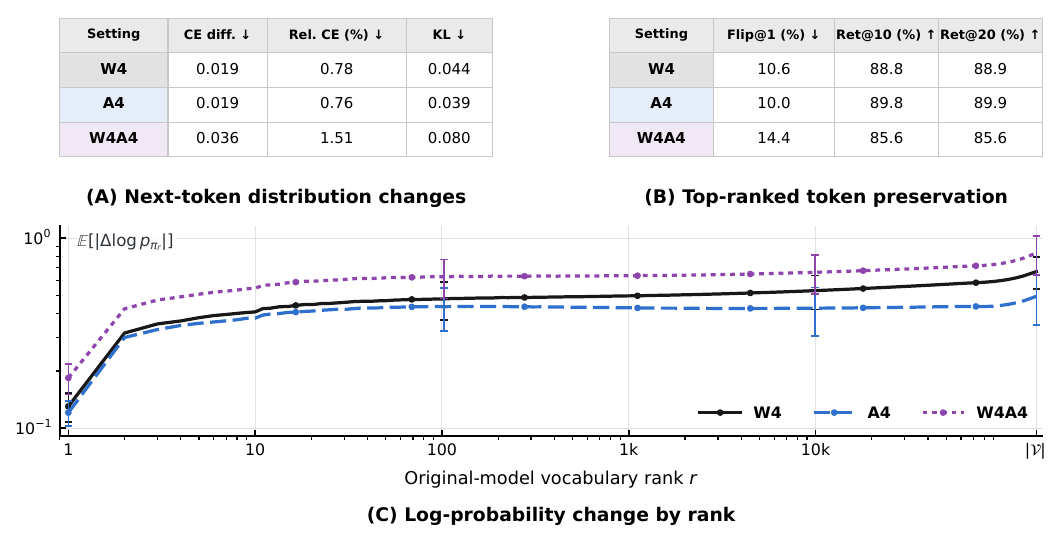}
  \caption{\textbf{Output changes under weight and activation quantization.}
  Qwen3-32B on the same C4 inputs and retained next-token positions used in
  Fig.~\ref{fig:app-sec5-quantization-site-recurrence}. Metrics follow
  Sec.~\ref{sec:preferential-output-preservation}. Error bars show one sample
  standard deviation across complete inputs.}
  \label{fig:app-sec5-quantization-site-output}
\end{figure}

\newpage
\subsection{Counteraction across PTQ Algorithms and Weight Formats}
\label{app:sec5-quantization-transfer}

Sec.~\ref{sec:scope-model-quantization} also asks whether counteraction depends
on how the 4-bit weights are produced. Fig.~\ref{fig:sec5-quantization-scope}(B)
shows that the blockwise counteraction cosine changes little across
calibration-free RTN and the calibration-based GPTQ and AWQ algorithms, all
using NVFP4 weights. This comparison does not show whether the same conclusion
holds for another weight format, or
whether similar counteraction curves imply similar relative-error and output
changes.

We compare RTN, GPTQ, and AWQ on Qwen3-4B under both NVFP4 and asymmetric INT4.
All conditions quantize the same Attention and
MLP weights, and we report the exact relative-error recurrence together with
the resulting next-token output changes. Evaluation follows
App.~\ref{app:main-figure-protocols} and retains at least $99.93\%$ of
positions. Every condition has
$\sum_{\ell=1}^{L}\termlabel{terminteraction}{T_{\rm inter}^{(\ell)}}<0$.
For NVFP4, the blockwise interaction-cosine curves under GPTQ and AWQ have
correlations above $0.998$ with the RTN curve.

\paragraph{Asymmetric INT4 quantization details.}
The INT4 experiments use the \texttt{W4A16\_ASYM} scheme from
\texttt{llmcompressor}.
All Attention and MLP linear weights are quantized in groups of 128
consecutive values along the input dimension, while activations remain in
BF16 and the LM head remains unquantized. Each group includes
zero and uses signed 4-bit integer values
$q_{\min}=-8$ and $q_{\max}=7$. Given group $w_{\min}$ and
$w_{\max}$, the static affine parameters are
\begin{equation*}
s=\frac{w_{\max}-w_{\min}}{15},
\qquad
z_{zp}=\operatorname{clip}
\left(
\operatorname{round}\left(q_{\min}-\frac{w_{\min}}{s}\right),
q_{\min},q_{\max}
\right),
\end{equation*}
where $z_{zp}$ is stored as INT8.

Quantization and reconstruction use
\begin{equation*}
q=\operatorname{clip}
\left(
\operatorname{round}\left(\frac{w}{s}+z_{zp}\right),
q_{\min},q_{\max}
\right),
\quad
\widehat{w}=s(q-z_{zp}).
\end{equation*}
\paragraph{Calibrated PTQ algorithm details.}
GPTQ and AWQ use 64 disjoint 512-token C4 calibration sequences; RTN uses no
calibration data. For NVFP4, NVIDIA Model Optimizer applies GPTQ with block
size 128 and Hessian dampening $0.01$, or AWQ-lite with alpha step $0.1$,
directly to the E2M1 weights and E4M3 block scales defined above. For
asymmetric INT4, GPTQ uses the same block size and dampening with static
activation ordering, while AWQ uses activation-and-weight duo scaling with
20 grid-search points.

Figs.~\ref{fig:app-sec5-calibrated-ptq},
\ref{fig:app-sec5-calibrated-ptq-three-term},
and~\ref{fig:app-sec5-calibrated-ptq-output} give the full results.
Across the tested algorithms and formats,
$\termlabel{terminteraction}{T_{\rm inter}}$ remains negative and its blockwise
geometry changes little, but $\termlabel{termcontrast}{T_{\rm add}}$, final
relative error, and output metrics still vary. A stronger cumulative
$\termlabel{terminteraction}{T_{\rm inter}}$ does not imply better
PTQ quality. Output changes also depend on newly introduced error and on how
the remaining hidden error affects the output layer.

\begin{figure}[H]
  \centering
  \includegraphics[width=0.85\textwidth]{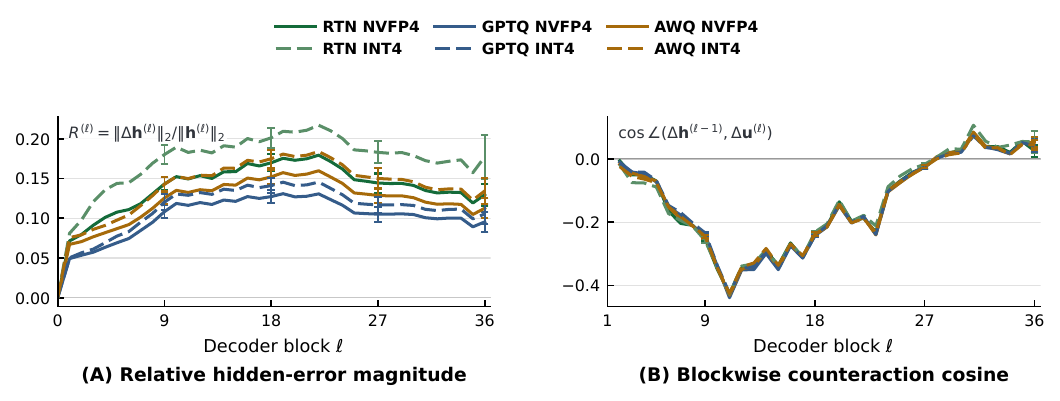}
  \caption{\textbf{Counteraction is consistent across PTQ algorithms and
  weight formats.} Qwen3-4B results. Solid and dashed curves denote NVFP4 and
  asymmetric INT4. Error bars show one sample standard deviation across
  inputs.}
  \label{fig:app-sec5-calibrated-ptq}
\end{figure}

\begin{figure}[H]
  \centering
  \includegraphics[width=0.9\textwidth]{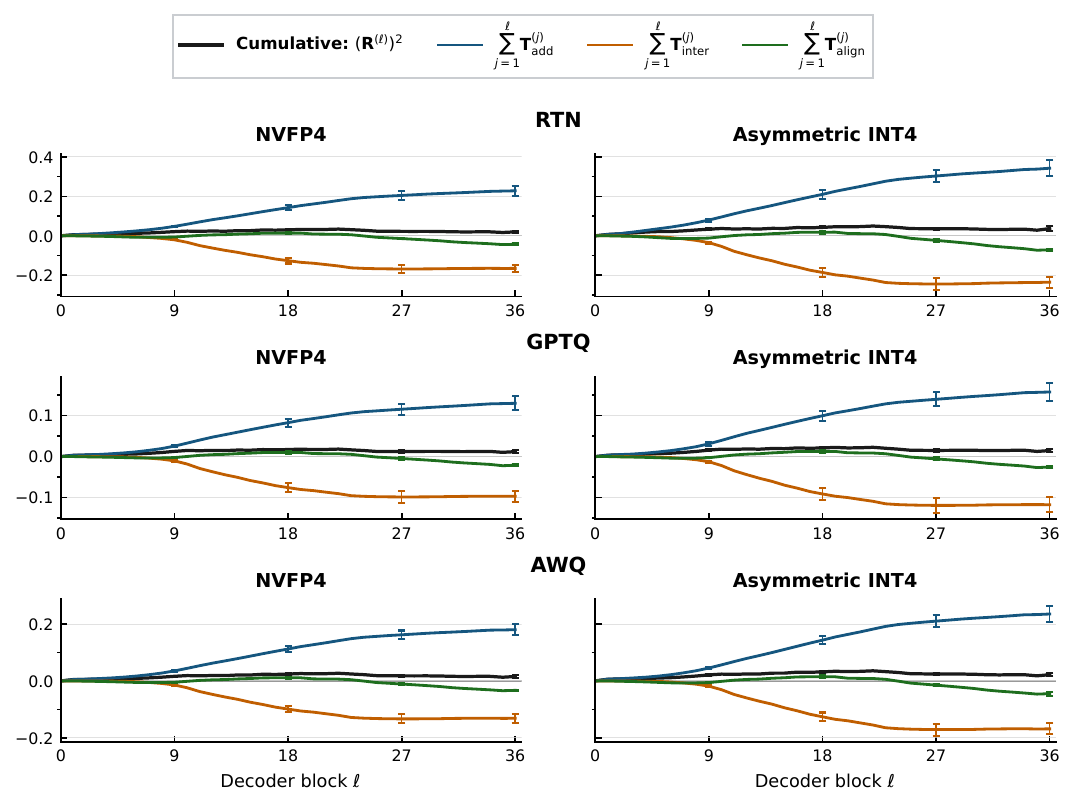}
  \caption{\textbf{Relative-error recurrence across PTQ algorithms and weight
  formats.} Qwen3-4B under RTN, GPTQ, and AWQ with NVFP4 and asymmetric INT4.
  Each row shows one PTQ algorithm, and the two columns share a vertical scale
  within that row. The black curve is the cumulative squared relative error;
  the colored curves are the cumulative contributions of
  $T_{\rm add}$, $T_{\rm inter}$, and $T_{\rm align}$.
  Error bars show one
  sample standard deviation across inputs.}
  \label{fig:app-sec5-calibrated-ptq-three-term}
\end{figure}

\begin{figure}[H]
  \centering
  \includegraphics[width=0.85\textwidth]{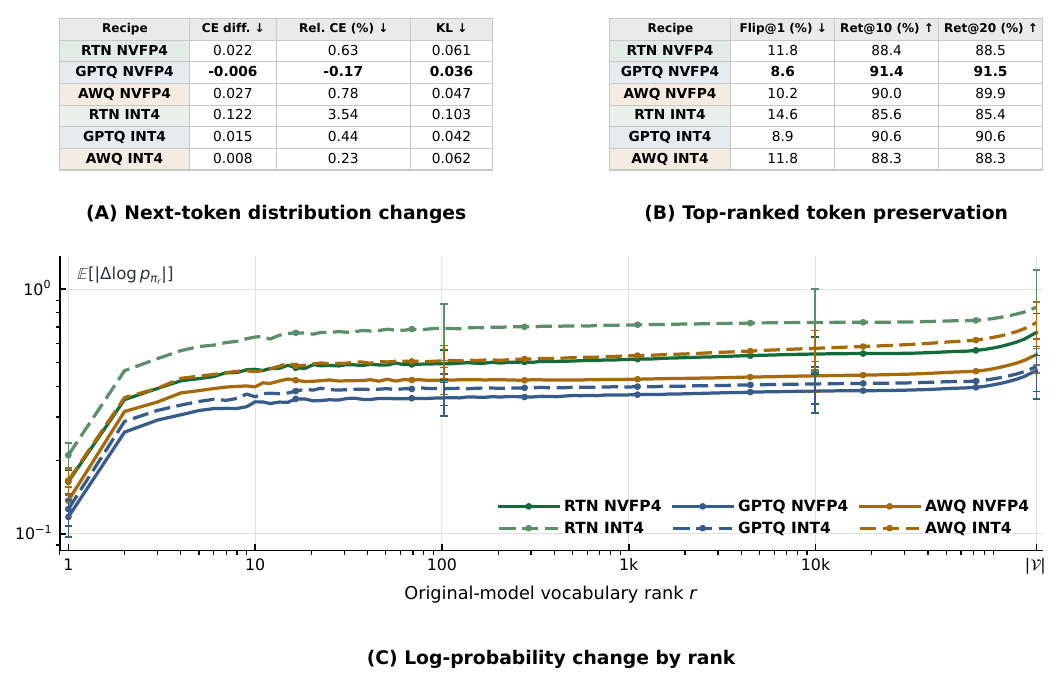}
  \caption{\textbf{Output changes across PTQ algorithms and weight formats.}
  Qwen3-4B results; metrics follow
  Sec.~\ref{sec:preferential-output-preservation}. \textbf{(C)} compares the
  rank-wise absolute log-probability change under the same inputs and position
  filter. Solid and dashed curves denote NVFP4 and asymmetric INT4. Error bars
  show one sample standard deviation across inputs.}
  \label{fig:app-sec5-calibrated-ptq-output}
\end{figure}

\newpage
\subsection{Output Robustness to Quantization across Models}
\label{app:output-analysis}
\label{app:sec5-cross-model-output}

Sec.~\ref{sec:lmhead-geometry} reports two connected observations for
Qwen3-32B. Weight quantization causes only modest changes in CE and KL, and
the scores and probabilities of top-ranked tokens change less than those of
lower-ranked tokens. Sec.~\ref{sec:lmhead-top-rank-geometry} explains the
rank dependence through the geometry of the shared LM head. Higher-ranked
token rows form smaller angles with the LM-head input, and the rotation of
that input produces much smaller changes in its angle to each fixed row. Here
we test whether the output stability and rank-dependent LM-head geometry extend
beyond the main model.

We evaluate pretrained models from the Qwen, OLMo, and Gemma families,
including dense and mixture-of-experts architectures, on C4, WikiText-103,
and GSM8K text. Each model is evaluated on 64 complete inputs from each data
view, with the original and quantized models compared at the same next-token
positions. The evaluation and aggregation follow
App.~\ref{app:main-figure-protocols}.

We first make the rank-dependent geometry explicit. At each position, let
$\pi_r$ be the token with BF16 score rank $r$, and define
$\theta_{\pi_r}:=\angle(\wv_{\pi_r},\hv_{\rm LM})$. For Qwen3-32B,
Fig.~\ref{fig:app-sec4-rank-cosine-monotonicity} shows that
$\Eb[\cos\theta_{\pi_r}]$ decreases from the top of the vocabulary toward lower ranks
after averaging over all three datasets. Thus, the rows of top-ranked tokens have
smaller projection angles to the LM-head input. Under the directional model
in Thm.~\ref{thm:lmhead-tangent-attenuation}, this is the geometry that makes
their relative scores less sensitive to the same LM-head input rotation.

\begin{figure}[H]
  \centering
  \includegraphics[width=0.7\textwidth]{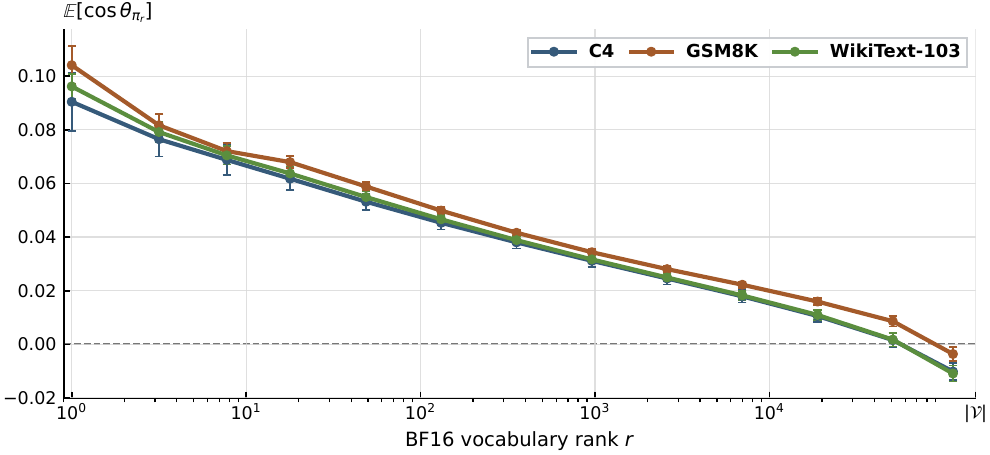}
  \caption{\textbf{Higher-ranked tokens have larger LM-head cosine values.}
  Error bars show one sample standard deviation across complete inputs. The grouped means
  decrease toward lower ranks on all three datasets,
  describing the overall rank trend rather than pointwise monotonicity.}
  \label{fig:app-sec4-rank-cosine-monotonicity}
\end{figure}

\noindent\begin{minipage}{\linewidth}
  \refstepcounter{table}\label{tab:sec4-cross-model-output}
  \small
  \textbf{Table~\thetable: Output changes across pretrained models.} Each entry is
  the macro-mean over C4, WikiText-103, and GSM8K text. $\Delta\mathrm{CE}$ and forward KL compare W4
  with BF16 on the same next-token positions.
  \par\smallskip
  \centering
  \setlength{\tabcolsep}{8pt}
  \begin{tabular}{lccc}
    \toprule
    Model & $\Delta\mathrm{CE}\downarrow$ & Rel. $\Delta\mathrm{CE}$ (\%)$\downarrow$
      & $\KL{\pv}{\widehat{\pv}}$ $\downarrow$ \\
    \midrule
    Qwen3-30B-A3B & 0.015 & 0.57 & 0.048 \\
    Qwen3-8B      & 0.017 & 0.54 & 0.060 \\
    OLMo3-32B     & 0.006 & 0.24 & 0.019 \\
    \bottomrule
\end{tabular}
\end{minipage}

Tab.~\ref{tab:sec4-cross-model-output} checks output stability on three
representative models. Their macro-mean $\Delta\mathrm{CE}$ ranges from
$0.006$ to $0.017$, relative $\Delta\mathrm{CE}$ from $0.24\%$ to $0.57\%$,
and forward KL from $0.019$ to $0.060$. The modest changes measured for
Qwen3-32B in Sec.~\ref{sec:preferential-output-preservation} therefore also
occur at other model scales and in another model family.

Fig.~\ref{fig:app-sec5-cross-model-output} then repeats the geometric and
rank-wise analysis for six dense and mixture-of-experts models.
Fig.~\ref{fig:app-sec5-cross-model-output}(A) compares the rotation
$\angle(\hv_{\rm LM},\widehat{\hv}_{\rm LM})$ with the vocabulary-mean
projection-angle change. Across the three data views, the LM-head input
rotates from $5.80^\circ$ to $21.71^\circ$, while the vocabulary-mean
projection-angle change ranges from only $0.093^\circ$ to $0.391^\circ$.
The attenuation observed for Qwen3-32B is therefore present in every tested
model.

Fig.~\ref{fig:app-sec5-cross-model-output}(B) measures the relative score
change of the BF16 rank-$r$ token,
$\Eb[|\widehat z_{\pi_r}-z_{\pi_r}|/|z_{\pi_r}|]$, over positions with
$z_{\pi_r}>0$. The observed W4 curves generally rise toward lower ranks. The
first- and second-order references set
$\rho=\norm{\widehat{\hv}_{\rm LM}}_2/\norm{\hv_{\rm LM}}_2=1$, so they
isolate the effect of the measured input rotation rather than its norm
change. They reproduce the overall rank trend, although the relative score
changes and local fluctuations vary across models.
Fig.~\ref{fig:app-sec5-cross-model-output}(C) carries the same comparison into
probability space. The top-ranked tokens have the smallest
absolute log-probability changes, and the approximations from
Thm.~\ref{thm:log-probability-approximations} follow the observed transition
from the top ranks to the rest of the vocabulary.

\begin{figure}[H]
  \centering
  \includegraphics[width=\textwidth]{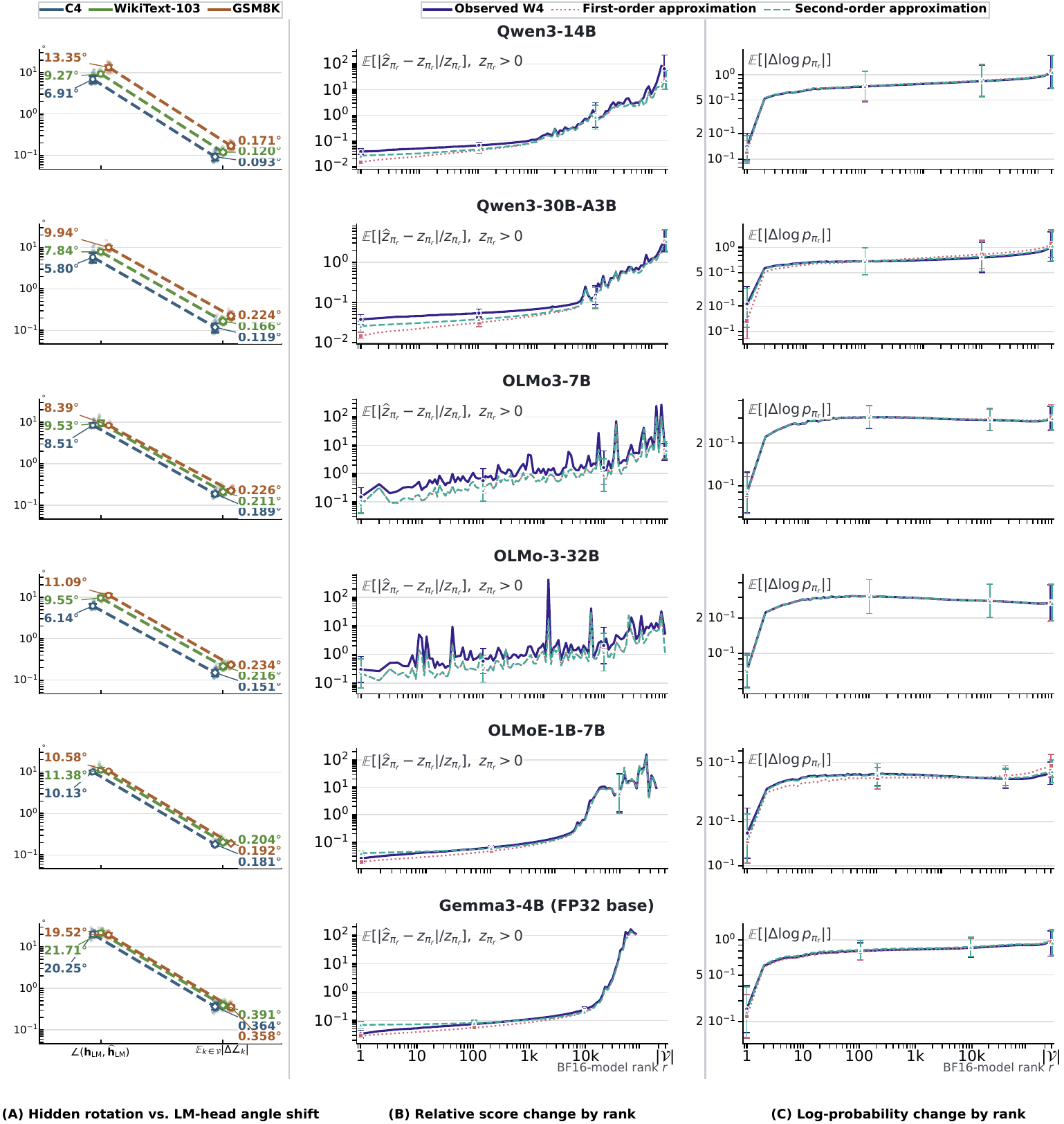}
  \caption{\textbf{Top-ranked token scores and probabilities are more stable
  after quantization across models.} Six models are evaluated on C4,
  WikiText-103, and GSM8K text. Projection-angle changes remain below one
  degree despite larger LM-head input rotations, and relative score and
  log-probability errors increase toward lower ranks. Score and log-probability results are averaged over the three datasets. Gemma3 uses FP32 for
  non-quantized computation. \textbf{(B)} uses the $\rho=1$ approximation from
  Thm.~\ref{thm:lmhead-tangent-attenuation}; \textbf{(C)} uses the approximations
  from Thm.~\ref{thm:log-probability-approximations}.
  }
  \label{fig:app-sec5-cross-model-output}
\end{figure}

\newpage
\subsection{Output Probability Sensitivity to Quantization across Temperatures}
\label{app:sec5-temperature}

Sec.~\ref{sec:scope-model-quantization} reports that preferential preservation
of top-ranked tokens persists when the softmax temperature changes. To separate
score geometry from the probability mapping, we keep each paired BF16/W4 score
vector fixed and vary only the temperature.

For fixed original and quantized score vectors, define
\begin{equation*}
\pv_T=\softmax(\zv/T),
\qquad
\widehat{\pv}_T=\softmax(\widehat{\zv}/T),
\qquad
T\in\{0.5,0.75,1,1.5,2\}.
\end{equation*}
No hidden state or LM-head projection is recomputed. Every $T>0$ preserves
the BF16 and W4 score rankings. Their top-$K$ token sets and Flip@1 therefore
remain unchanged, as do the LM-head angles. Probability metrics do change.
Fig.~\ref{fig:app-sec5-temperature-replay} shows the mean forward KL. Across the tested range,
it spans $0.0247$--$0.0661$ on C4 and $0.1514$--$0.3253$ on GSM8K text.
The same figure also reports the actual absolute log-probability change at each
rank and temperature.
Positive temperature rescaling preserves both score rankings, so it does not
change the score-side rank protection. It changes the magnitude of probability
errors, while the top-to-tail trend remains inherited from the score changes.
Thus, temperature affects probability sensitivity without changing the
hidden-state propagation or LM-head projections studied in
Secs.~\ref{sec:residual-counteraction} and~\ref{sec:lmhead-geometry}.

\begin{figure}[H]
  \centering
  \includegraphics[width=0.95\textwidth]{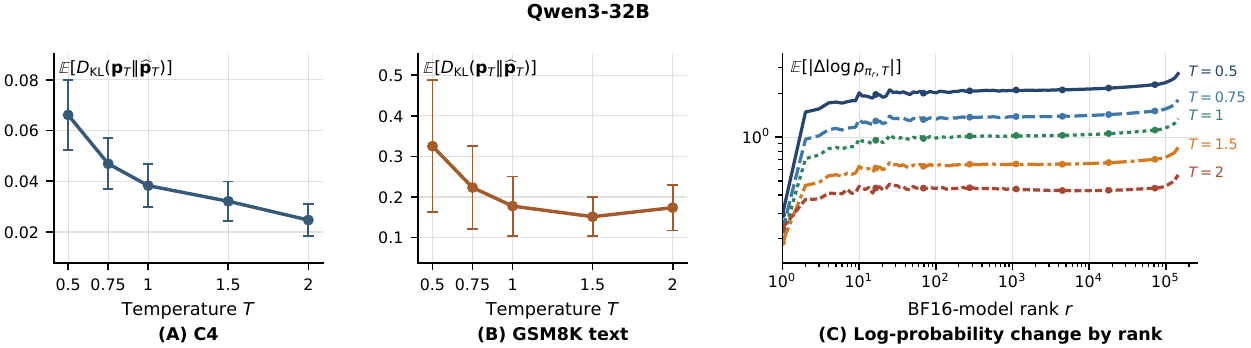}
  \caption{\textbf{Next-token probability changes under temperature
  rescaling.} \textbf{(A)} and \textbf{(B)} show the mean forward KL on C4 and
  GSM8K text. Error bars show standard deviation across inputs.
  \textbf{(C)} shows the actual mean absolute
  log-probability change over BF16-model rank for five temperatures averaged
  over both datasets.}
  \label{fig:app-sec5-temperature-replay}
\end{figure}

\end{document}